\documentclass[twoside,11pt]{article}
\usepackage{sidecap}
\sidecaptionvpos{figure}{c}

\usepackage{jmlr2e}

\usepackage{algorithm}

\usepackage[table]{xcolor}        
\usepackage[utf8]{inputenc} 
\usepackage[T1]{fontenc}    
\usepackage{hyperref}       
\usepackage{url}            
\usepackage{booktabs}       
\usepackage{tablefootnote}
\usepackage{amsfonts}       
\usepackage{nicefrac}       
\usepackage{microtype}      

\usepackage{hyperref}       
\usepackage{url}            
\usepackage{booktabs}       
\usepackage{amsfonts}       
\usepackage{nicefrac}       
\usepackage{microtype}      
\usepackage{lipsum}

\usepackage{mathtools} 
\usepackage{amsmath}
\usepackage{algorithm}
\usepackage{algpseudocode}
\usepackage{amssymb}
\usepackage{graphicx}
\usepackage{amsmath}
\usepackage{enumitem}

\usepackage{multirow, booktabs}
\usepackage{caption}

\usepackage{soul}

\usepackage{multirow, booktabs, xcolor, array}
\usepackage{caption}

\usepackage{bbold}

\usepackage{blindtext}

\usepackage{multirow}
\usepackage{dsfont}

\usepackage{subcaption}

\usepackage{caption}

\usepackage{booktabs,ragged2e}

\usepackage[flushleft]{threeparttable}

\usepackage[export]{adjustbox}

\makeatletter
\newcounter{phase}[algorithm]
\newlength{\phaserulewidth}
\newcommand{\setphaserulewidth}{\setlength{\phaserulewidth}}

\makeatother
\setphaserulewidth{.7pt}

\algdef{SE}[SUBALG]{Indent}{EndIndent}{}{\algorithmicend\ }%
\algtext*{Indent}
\algtext*{EndIndent}

\usepackage{threeparttable, booktabs, amsmath, multirow}

\newlength{\tempheight}
\newlength{\tempwidth}
\definecolor{midgray}{gray}{0.45} 

\newcommand{\rowname}[1]
{\rotatebox{90}{\makebox[\tempheight][c]{\textbf{#1}}}}

\newcommand{\columnname}[1]
{\makebox[\tempwidth][c]{\textbf{#1}}}

\usepackage{array}
\makeatletter
\g@addto@macro{\endtabular}{\gdef\rowfonttype{}}
\makeatother
\newcommand{\rowfonttype}{}
\newcolumntype{L}{>{\rowfonttype\strut}l}
\newcolumntype{C}{>{\rowfonttype\strut}c}
\newcolumntype{R}{>{\rowfonttype\strut}r}

\usepackage{newfloat}
\usepackage{listings}
\floatstyle{ruled}
\newfloat{listing}{tb}{lst}{}
\floatname{listing}{Listing}

\usepackage{xcolor}

\usepackage{mathrsfs}
\DeclareMathAlphabet{\mathcal}{OMS}{cmsy}{m}{n}

\firstpageno{1}

\begin{document}

\title{The Normalized Cross Density Functional: A Framework to Quantify Statistical Dependence for Random Processes}


\author{\name Bo Hu \email hubo@ufl.edu \\
      \addr Department of Electrical and Computer Engineering\\
      University of Florida\\
      Gainesville, FL 32611, USA
      \AND
      \name Jos{\'e}~C.~Pr{\'\i}ncipe \email principe@cnel.ufl.edu \\
      \addr Department of Electrical and Computer Engineering\\
      University of Florida\\
      Gainesville, FL 32611, USA}

\editor{}

\maketitle

\begin{abstract}

This paper presents a novel approach to measuring statistical dependence between two random processes (r.p.) using a positive-definite function called the Normalized Cross Density (NCD). NCD is derived directly from the probability density functions of two r.p. and constructs a data-dependent Hilbert space, the Normalized Cross-Density Hilbert Space (NCD-HS). By Mercer's Theorem, the NCD norm can be decomposed into its eigenspectrum, which we name the Multivariate Statistical Dependence (MSD) measure, and their sum, the Total Dependence Measure (TSD). Hence, the NCD-HS eigenfunctions serve as a novel embedded feature space, suitable for quantifying r.p. statistical dependence. In order to apply NCD directly to r.p. realizations, we introduce an architecture with two multiple-output neural networks, a cost function, and an algorithm named the Functional Maximal Correlation Algorithm (FMCA). With FMCA, the two networks learn concurrently by approximating each other's outputs, extending the Alternating Conditional Expectation (ACE) for multivariate functions. We mathematically prove that FMCA learns the dominant eigenvalues and eigenfunctions of NCD directly from realizations. Preliminary results with synthetic data and medium-sized image datasets corroborate the theory. Different strategies for applying NCD are proposed and discussed, demonstrating the method's versatility and stability beyond supervised learning. Specifically, when the two r.p. are high-dimensional real-world images and a white uniform noise process, FMCA learns factorial codes, i.e., the occurrence of a code guarantees that a specific training set image was present, which is important for feature learning.

\end{abstract}


\section{Introduction}

A central question in computer vision, information-theoretic learning, and computational neuroscience is how learning systems can uncover latent structures from real-world sensory data. This paper introduces a new, principled learning framework built on a bidirectional procedure for random processes (r.p.). By definition, a r.p. is a sequence of Borel measurable random variables (r.v.) indexed by a discrete or continuous variable $t$, such as a signal's time samples or an image's spatial coordinates, which are also known as the r.p. functional space. Our focus is the discrete-time r.p. $\mathbf{x} = \{\mathbf{x}(t, \omega)\}_{t=1}^{T}$, where $\omega$ is a subset of the common sample space $\Omega$. For simplicity, $\omega$ will be left out of the notation. Assume an ensemble of realizations $x_1, x_2, \cdots, x_N$ are sampled from this r.p. Each realization $x_n$ is a sequence $x_n(1), x_n(2), \cdots, x_n(T)$ which contains a rich latent functional structure.

Our contribution involves constructing a Hilbert space directly from the data's probability density functions ($pdfs$) to quantify the statistical dependence within and across two r.p. We start by reviewing the related literature.
\vspace{5pt}

\noindent \textbf{Explicit functional space information.} Initial attempts to quantify r.p. statistical dependence focused on the mean value of the $pdf$ between two time instances, i.e., $\mathbb{E}[p(\mathbf{x}(t), \mathbf{x}(t-\tau))]$. These efforts originate from Pearson's $\rho$ correlation coefficient and Principal Component Analysis (PCA), leading to Wiener's and Parzen's solutions for minimum mean-squared error methods (\citealt{wiener1949extrapolation, hagenblad1999aspects, hagenblad2000maximum, wills2013identification, parzen1961approach, parzen1999stochastic}). While suitable for stationary r.p., this quantifier only captures second-order statistics and can be limiting, as shown in Independent Component Analysis (ICA) (\citealt{hyvarinen2000independent}). Moreover, Wiener filter theory is limited to linear mappers, which are not universal approximators. 

Gaussian Processes (GP) assume that the joint $pdf$ between two time instances of the r.p. arises from a multivariate Gaussian distribution. This can also be seen as adding a Gaussian prior over the functions, useful for Bayesian inference. However, the embedded Gaussian prior still enforces a stationary assumption in the model. The GP regressor is inherently a nonlinear function of the samples~(\citealt{zhu1997gaussian, williams1998prediction, williams2006gaussian}), substantially improving upon Wiener theory and closely relating to optimal online kernel learning~(\citealt{principe2011kernel}). GP can be entirely described by its second-order statistics, which means that the sample covariance, assumed to have a zero mean, quantifies the functional structure of the random process. Since the covariance is positive semidefinite, it can be decomposed using the Karhunen Lo\`eve transform~(\citealt{karhunen1947lineare}), favoring a multivariate decomposition. More recent results~(\citealt{heinonen2016non, heinonen2018learning}) have extended GP for time-varying regression using a nonstationary extension of the squared exponential kernel. \vspace{5pt}

\noindent \textbf{Estimation of statistical descriptors directly from data.} Since the $pdf$ of the data is rarely available in machine learning, density estimation is necessary~(\citealt{parzen1962estimation, Liamestimate, silverman2018density}), but many methods struggle in high dimensions. Nyström's method (\citealt{baker1977numerical}), used in GP, first estimates the $pdf$ by counting the occurrences of paired samples, then applies standard eigendecomposition to the estimated joint $pdf$.
        
More recently, the kernel mean embedding theory (\citealt{muandet2017kernel, hayati2023kernel}) in Reproducing Kernel Hilbert Spaces (RKHS) has become dominant because of its excellent estimation properties and scalability to high dimensions. It can be employed in information-theoretic learning as estimators for R\'enyi's quadratic entropy, mutual information, and divergences (\citealt{Principe2010information}). It is applied to construct other statistical dependence criteria such as Kernel Independent Component Analysis (KICA)~(\citealt{bach2002kernel}) and the Hilbert-Schmidt Independence Criterion (HSIC)~(\citealt{fukumizu2004dimensionality, gretton2007kernel, sriperumbudur2010hilbert, muandet2017kernel}).

The mean embedding theory uses the functional mean operator to construct linear operators in the RKHS that involve $pdfs$, such as covariance or conditional mean, directly from data. It is particularly suitable for measuring statistical dependence, by approximating the cross-covariance operator in the RKHS. For a pair of r.v. $\mathbf{x}$ and $\mathbf{u}$ with a chosen positive-definite characteristic kernel $\mathcal{K}(\cdot, \cdot)$ (e.g., Gaussian), each r.v. is equipped with its respective RKHS, denoted as $\mathscr{H}$ and $\mathscr{G}$. Each RKHS has its own covariance operator $\mathcal{C}_{X\hspace{-1pt}X}$ or $\mathcal{C}_{U U}$. A cross-covariance operator $\mathcal{C}_{X\hspace{-1pt}U}$ is constructed, mapping from RKHS $\mathscr{H}$ to $\mathscr{G}$, such that for functions $f \in \mathscr{H}$ and $g \in \mathscr{G}$, their inner product $\langle f, \mathcal{C}_{X\hspace{-1pt}U}g \rangle$ is the covariance between them over their joint distribution. The cross-covariance operator is Hilbert-Schmidt and forms a new Hilbert space $\text{HS}(\mathscr{H}, \mathscr{G})$. Define the normalized cross-covariance operator $\mathcal{V}_{X\hspace{-1pt}U}:= \mathcal{C}_{X\hspace{-1pt}X}^{-\frac{1}{2}}\mathcal{C}_{X\hspace{-1pt}U} \mathcal{C}_{U\hspace{-1pt}U}^{-\frac{1}{2}}$. Its spectrum and norm measure the statistical dependence~(\citealt{bach2002kernel, gretton2007kernel}). In practice, these operators are approximated using Gram matrices by applying kernels to all pairwise samples in the training set. The construction of the operator and its new Hilbert space is elegant but faces significant limitations: (1) operations with Gram matrices are computationally very intensive; and (2) the space construction depends on the kernel and its hyperparameter, causing biases in estimation. 

Alternative approaches maximize a variational lower bound on mutual information using kernels or neural networkss~(\citealt{nguyen2009surrogate}). However, they fundamentally differ from GP and kernel mean embedding, as they learn a scalar-valued density ratio function without a Hilbert space or decomposition.

\vspace{5pt}

\noindent \textbf{Bidirectional mappings.} Alternating Conditional Expectation (ACE) substitutes the covariance with Alfred R\'enyi's maximal correlation, also a measure of statistical dependence~(\citealt{breiman1985estimating, andrew2013deep, hu2021mimo, huang2018gaussian, huang2019universal}). For r.v. $\mathbf{x}$ and $\mathbf{u}$, R\'enyi's maximal correlation finds optimal functions $f$ and $g$, among all feasible measurable functions, that maximize the correlation coefficient between $f(\mathbf{x})$ and $g(\mathbf{u})$. ACE uses a bidirectional recursive algorithm to find the solution, recursively updating $f$ and $g$ to match each other's conditional mean. ACE's limitation is that it projects samples onto a real line, fails to represent the full joint $pdf$, and does not form a multidimensional space.

This paper uses ACE's bidirectional recursions to directly model statistical dependence between two r.p. Central to this approach is the Normalized Cross Density (NCD), related to the density ratio found in Mutual Information (MI), with each r.p. of the pair defining its own NCD. NCD is positive definite, thus defines a new Hilbert space named the NCD-Hilbert Space (NCD-HS). Note that the norm of NCD-HS is defined directly from the $pdf$ of the data, so it is a data dependent Hilbert Space, unlike the conventional RKHS theory. The NCD norm allows a spectrum decomposition by Mercer's Theorem, producing a countably infinite sequence of eigenvalues and an orthonormal set of eigenfunctions. We name its eigenspectrum the Multivariate Statistical Dependence (MSD), which is ideal for r.p. quantification. 

Section~\ref{section_fmca} presents a neural network architecture, a cost function and an algorithm called the Functional Maximal Correlation Algorithm (FMCA), which extends ACE for multivariate functions. We show that FMCA's optimization implements the spectrum decomposition of the NCD through bidirectional recursive updates to find the dominant eigenvalues and eigenfunctions. Section~\ref{reference_process} discusses the examples that will demonstrate the flexibility and accuracy of the FMCA. Sections~\ref{experiment_section},~\ref{aggregation_section}, and~\ref{image_codeds_section} present numerical evidence of the quality of the decompositions and compare results with alternative methods. The paper concludes with a brief discussion.
Table 1 compares NCD-HS and FMCA with the alternatives compared in this introduction.

\begin{table}[h]
\caption{\footnotesize Comparing NCD-HS and FMCA with related methods. The advantages of FMCA include: (1) MSD is multivariate, while MI is a scalar; (2) Only the NCD Hilbert space is constructed from the densities, while other kernels, like Gaussian kernel are fixed and data-independent (only the operator are data dependent); (3) NCD-HS eliminates the need for hyperparameter kernel tuning and thus ensures exact, unbiased approximation (assuming convex optimization); (4) Only FMCA can project test data to features in the embedding space of the eigenfunctions, whereas KICA and HSIC produce only a statistical dependence criterion, not a feature projector.}
\centering
\footnotesize
\begin{threeparttable}
\resizebox{\linewidth}{!}{
    \begin{tabular}{llllll}
    \toprule
    Algorithm & Stat. Dep.  Measures\tnote{1}  & Space Type & Data-Dep. Kernel\tnote{5} & Solver & Outputs \\\midrule
    GP     & CORR\tnote{2} & $L_2$\tnote{3}   & No  & Nystr{\"o}m & Predictor \\
    KICA  & KGV & $\text{HS}(\mathscr{H}, \mathscr{G})$ &  No  & Eigenproblem & Scalar Measure\\
    HSIC & NOCCO & $\text{HS}(\mathscr{H}, \mathscr{G})$ &  No  & Generalized Eigenproblem & Scalar Measure\\
    MINE  & MI & $L_2$\tnote{4} &  No  & Optimization & Scalar Measure\\ 
    \rowcolor{gray!25} FMCA  & MSD of NCD& NCD-HS & Yes & Optimization & Embed. Space\tnote{6} \\
    \bottomrule
    \end{tabular}}
\begin{tablenotes}
\footnotesize
\item[1] Stat. Dep. Measure: Algorithm-defined statistical dependence measures.
\item[2] CORR: GP calculates r.v. correlations over time (second-order statistics).
\item[3] GP minimizes the $L_2$ distance for time series regression using kernels.
\item[4] MINE finds only the density ratio, a scalar-valued measurable function, without an explicit space.
\item[5] Data-Dep. Kernel: Kernel dependency on data.
\item[6] Embed. Space: Only FMCA projects unseen data to features in an eigenfunction embedding space.
\end{tablenotes}
\end{threeparttable}
\end{table}

\section{Measuring Statistical Dependence with a Novel Variational Approach}
\label{bidirectional_system}

The framework is presented in two parts. In this section, the focus is on the theoretical foundations of the variational system and the Normalized Cross Density (NCD). Assuming that all \textit{pdfs} of the r.p. are known, how can we properly quantify the statistical dependence within and between two r.p.? Unlike conventional methods, our proposed approach constructs a Bayesian system and creates latent internal variables that best reflect the statistical dependence. 


\subsection{Bidirectional recursion, normalized cross density, and dependence measure}

Suppose we have two r.p., $\mathbf{x} = \{\mathbf{x}(t)\}_{t=1}^{T_1}$ and $\mathbf{u} = \{\mathbf{u}(t)\}_{t=1}^{T_2}$, each with a fixed and finite length ($T_1$ and $T_2$, respectively). For the proofs, their joint and conditional $pdfs$ are assumed to exist and be Lebesgue measurable. No stationarity condition is required. The variational system is defined as follows.

\begin{figure}[t]
\captionsetup{type=figure}
\centering
{\includegraphics[width=1\textwidth]{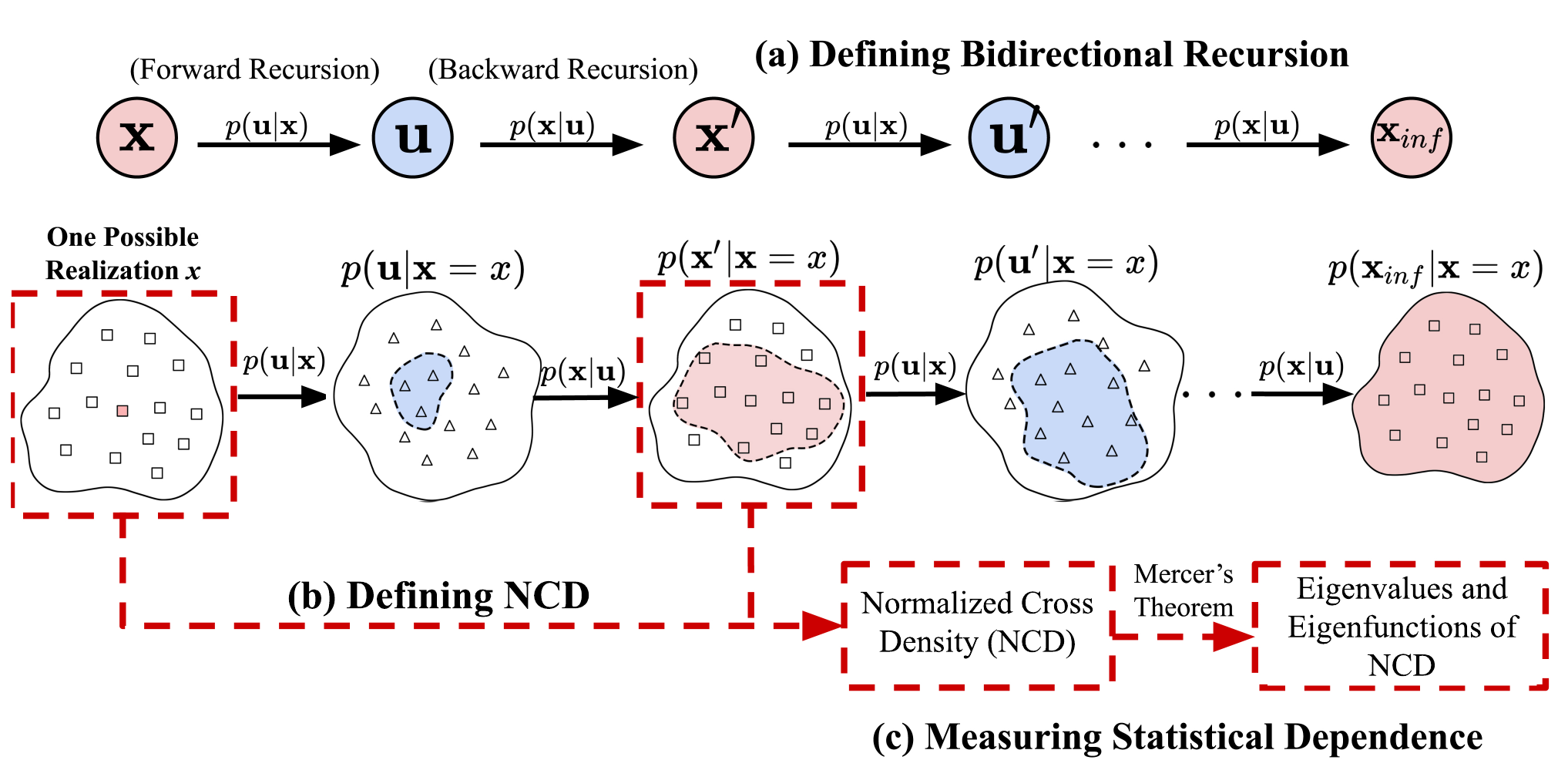}}\\
\captionof{figure}{\small Diagram of the bidirectional recursion process. (a) Given two r.p. $\mathbf{x}$ and $\mathbf{u}$ with a joint density, we treat the two conditional densities as bidirectional iterative functions. By recursively applying the two conditional mappings, two internal variables $\mathbf{x}'$ and $\mathbf{u}'$ are created. (b) The joint densities $p(x, x')$ and $p(u, u')$ obtained by applying marginalization evaluate the level of independence in the space of the internal variables, represented by red and blue regions. (c) Each joint density induces its NCD and Hilbert space (NCD-HS). By applying Mercer's theorem, the eigenspectrum and eigenfunctions are obtained. The spectrum is the statistical dependence measure, MSD.}\label{figureprocess}
\vspace{-5pt}
\end{figure}

\begin{definition}
(Bidirectional recursion) Suppose two r.p. $\mathbf{x} = \{\mathbf{x}(t)\}_{t=1}^{T_1}$, $\mathbf{u} = \{\mathbf{u}(t)\}_{t=1}^{T_2}$ and their joint \textit{pdf} $p(\mathbf{x} = x, \mathbf{u} = u)$ are given. The two conditional \textit{pdfs} have the forms
\begin{equation}
\begin{aligned}
p(\mathbf{x} = x|\mathbf{u} = u) = p\Big(\mathbf{x}(1) = x(1), \cdots, \mathbf{x}(T_1) = x(T_1)\big|\mathbf{u}(1) = u(1), \cdots, \mathbf{u}(T_2) = u(T_2)\Big), \\
p(\mathbf{u} = u|\mathbf{x} = x) = p\Big(\mathbf{u}(1) = u(1), \cdots, \mathbf{u}(T_2) = u(T_2)\big|\mathbf{x}(1) = x(1), \cdots, \mathbf{x}(T_1) = x(T_1)\Big).
\end{aligned}
\label{forward_backward}
\end{equation}
To simplify the explanation, we assign different directions to each $\textit{pdf}$: $p(\mathbf{u} = u|\mathbf{x} = x)$ as the forward direction and $p(\mathbf{x} = x|\mathbf{u} = u)$ as the backward direction. Interpreting each \textit{pdf} as an iterative function, we construct a variational system, termed bidirectional recursion, by recursively applying both directions of the conditional \textit{pdfs}. Starting with either variable $\mathbf{x}$ or $\mathbf{u}$, we construct two internal variables, $\mathbf{x}'$ and $\mathbf{u}'$, defined by two joint \textit{pdfs}
\begin{equation}
\begin{aligned}
p(\mathbf{x}=x, \mathbf{x}'=x')  &= \int_{\mathcal{U}} p(\mathbf{x}' = x'|\mathbf{u} = u) p(\mathbf{u} = u| \mathbf{x} = x) p(\mathbf{x} = x) pu,\\
p(\mathbf{u}=u, \mathbf{u}'=u') &= \int_{\mathcal{X}} p(\mathbf{u}' = u'|\mathbf{x} = x) p(\mathbf{x} = x| \mathbf{u} = u) p(\mathbf{u} = u) px.
\end{aligned}
\label{cross_densities}
\end{equation}
We name these two joint \textit{pdfs} the cross densities. The four variables $\mathbf{x}, \mathbf{u}, \mathbf{x}', \mathbf{u}'$ form the bidirectional recursion. For convenience, we use the following notations:
\vspace{5pt}

(i) Marginal \textit{pdfs}: $p(x):=p(\mathbf{x}=x),\; p(u):=p(\mathbf{u}=u);$ \vspace{5pt}

(ii) Conditional \textit{pdfs}: $p(u|x):=p(\mathbf{u} = u|\mathbf{x} = x),\; p(x'|u):=p(\mathbf{x} = x'|\mathbf{u} = u) $; \vspace{5pt}

(iii) Joint and conditional cross densities: 
\begin{equation}
\begin{gathered}
p(x, x'):= p(\mathbf{x}=x, \mathbf{x}'=x'), p(u, u') := p(\mathbf{u}=u, \mathbf{u}'=u'),\\
p(x'|x) := p(\mathbf{x}'=x'|\mathbf{x}=x), \,\,p(u'| u) := p(\mathbf{u}'=u'|\mathbf{u}=u).\vspace{5pt}
\end{gathered}
\end{equation}
Each cross density $p(x, x')$ and $p(u, u')$ will induce a positive-definite function.
\label{definition_density}
\end{definition}

\noindent Figure~\ref{definition_density} further illustrates the variational procedure, demonstrating how this construction characterizes statistical dependence. The process of generating internal variables consists of two steps. First, recursion is established by applying two conditional \textit{pdfs} bidirectionally, resulting in three variables. Depending on the initial variable, which can be either $\mathbf{x}$ or $\mathbf{u}$, the system takes the form of $\mathbf{x}\rightarrow\mathbf{u}\rightarrow\mathbf{x}'$ or $\mathbf{u}\rightarrow\mathbf{x}'\rightarrow\mathbf{u}'$. The second step derives the joint \textit{pdf} $p(x, x')$ or $p(u, u')$ by performing marginalization over the support of the intermediate variable, as demonstrated in Equation~\eqref{cross_densities}. The blue and red regions in Figure~\ref{definition_density} represent possible reachable supports that are dependent upon the amount of statistical dependence between the external variables.
\vspace{5pt}

\noindent \textbf{Extreme conditions of statistical dependence.} To further clarify this property, it is best to consider the extreme conditions. Referring to Figure~\ref{definition_density}, consider the case where $\mathbf{x}$ and $\mathbf{u}$ are in a one-to-one correspondence. In this scenario, a given specific realization $x$ of $\mathbf{x}$, there exists one corresponding realization $u$ of $\mathbf{u}$ that has a one-to-one mapping with this particular realization $x$. This implies that there will also be a one-to-one mapping between the realization $x$ and a realization $x'$ of the variable $\mathbf{x}'$, such that the cross density $p(x, x')$ is nonzero if and only if $x' = x$. Conversely, if $\mathbf{x}$ and $\mathbf{u}$ are not in a one-to-one correspondence, there will be multiple $x'$ that define a region of realizations with a positive cross density. In the extreme case where $\mathbf{x}$ and $\mathbf{u}$ are statistically independent, this region expands to the entire support of $\mathbf{x'}$ or $\mathbf{u}'$, as the cross density matches the product of marginal densities. Hence, by analyzing the support of the cross density, we can determine the degree of statistical dependence between two processes. \vspace{5pt}

\noindent \textbf{Equilibrium of the recursion.} A key characteristic of this variational system is its equilibrium. Recall the ACE algorithm, which updates two functions iteratively to maximize the correlation coefficient; our variational system exhibits similar bidirectionality. We find a direct link between ACE and our system through the spectrum decomposition of the NCD, induced by the cross densities. The equilibrium state of the variational system, matches the optimal functions approximated by ACE. However, ACE identifies only the top eigenfunction, ignoring other orthonormal functions in the decomposition. We provide definitions for the two functions, $K(x, x')$ and $K(u, u')$, in Lemma~\ref{lemma_2_kernel}.

\begin{lemma}
(NCD definition) The functions $K(x, x') = \frac{p(x, x')}{p^{\frac{1}{2}}(x) p^{\frac{1}{2}}(x')}$ and $K(u, u') = \frac{p(u, u')}{p^{\frac{1}{2}}(u) p^{\frac{1}{2}}(u')}$ are positive-definite kernel functions, referred to as the Normalized Cross Density (NCD). Each function defines a linear operator associated with a Hilbert space, referred to as the Normalized Cross-Density Hilbert Space (NCD-HS).
\label{lemma_2_kernel}
\end{lemma}

\noindent The proof of the lemma is straightforward and is located in Appendix~\ref{appendix_ncd_def}. An essential characteristic of a kernel function is its eigenfunction expansion. Applying Mercer's theorem to the linear operator associated with the NCD kernel, we show that its eigenvalues numerically measures the statistical dependence.
\begin{corollary}
\textcolor{black}{Applying} Mercer's theorem, the linear operator 
\begin{equation}
\begin{aligned}
\mathcal{T} \varphi (x) = \int_{\mathcal{X}}K(x, x') \varphi (x') dx'
\end{aligned}
\label{linear_operator}
\end{equation}
induces an orthonormal eigenfunction set $\{\phi_i \}_{i=1}^\infty$ and eigenvalues $\{ \lambda_i \}_{i=1}^\infty$ such that \begin{equation}
\begin{gathered}
\int_{\mathcal{X}}\phi^2_i(x) dx = 1,\; \mathcal{T}\phi_i(x) = \lambda_i \phi_i,\; i=1,2,\cdots, \infty,\\
K(x, x') = \sum_{i=1}^\infty \lambda_i \phi_i(x) \phi_i(x').
\end{gathered}
\end{equation}
NCD's eigenspectrum $\{ \lambda_i \}_{i=1}^\infty$ is a novel statistical dependence measure. The other kernel $K(u, u')$ also induces a different orthonormal eigenfunction set $\{\psi_i \}_{i=1}^\infty$. Moreover, the two NCDs $K(x, x')$ and $K(u, u')$ share the same eigenspectrum $\{ \lambda_i \}_{i=1}^\infty$.
\label{eiproperty}
\end{corollary}

Throughout the paper, we assume that the eigenvalues are ordered in a descending manner: $\lambda_1 \geq \lambda_2 \geq \cdots$. The eigenfunctions $\phi_1, \phi_2, \cdots$ and $\psi_1, \psi_2, \cdots$ are ordered in correspondence with their respective eigenvalues. The property of two NCDs having identical eigenvalues can be easily proven later in the second property in Lemma~\ref{lemma_solution}, Section~\ref{CORE-THEOREM}. Our next lemma shows that the eigenspectrum of NCD has a unique property in the case of independence.
\begin{lemma}
For any given r.p., the largest eigenvalue of NCD is always equal to 1, that is, $\lambda_1 = 1$. If and only if two r.p. are independent, all the other eigenvalues are zero $\lambda_2 = \lambda_3 = \cdots = 0$.



\label{dependence}
\end{lemma}
\begin{proof}
We show $\lambda_1 = 1$ by choosing the function $\phi_1(x) = p^{\frac{1}{2}}(x)$. This function satisfies the condition $\int \phi_1^2(x) dx = \int p(x) dx = 1$, ensuring that its norm is equal to $1$. Next, we show it is invariant to the linear operator:
\begin{equation}
\begin{gathered}
T\phi_1(x) = \int_{\mathcal{X}} \frac{p(x, x')}{p^{\frac{1}{2}}(x)p^{\frac{1}{2}}(x')} p^{\frac{1}{2}}(x') dx' = p^{\frac{1}{2}}(x) = \phi_1(x).
\end{gathered}
\end{equation}
Hence, $\phi_1(x) = p^{\frac{1}{2}}(x)$ is an eigenfunction with an eigenvalue $\lambda_1 = 1$. Consider now the case where the two r.p. are independent. In this scenario, the NCD takes the form $K(x, x') = p^{\frac{1}{2}}(x)p^{\frac{1}{2}}(x') = \phi_1(x) \phi_1(x')$. This implies that $K(x, x')$ has only one positive eigenvalue and one eigenfunction. Conversely, if $K(x, x')$ can be expressed as $K(x, x') = \phi_1(x) \phi_1(x') = p^{\frac{1}{2}}(x)p^{\frac{1}{2}}(x')$, it can be shown that
\begin{equation}
\begin{gathered}
K(x, x') = \frac{p(x, x')}{p^{\frac{1}{2}}(x)p^{\frac{1}{2}}(x')} = p^{\frac{1}{2}}(x)p^{\frac{1}{2}}(x') \Rightarrow p(x, x') = p(x)p(x'),
\end{gathered}
\end{equation}
which is satisfied only when the two r.p. $\mathbf{x}$ and $\mathbf{u}$ are independent.\end{proof}
Lemma~\ref{dependence} provides the proof that NCD's eigenvalues satisfy the important postulates for any statistical dependence measure between two r.v.~(\citealt{renyi1959measures}): the measure must be zero for independence. Lemma~\ref{dependence} extends this postulate to characterize multidimensional statistical dependence using a set of eigenvalues. Building upon Lemma~\ref{dependence}, we formally propose our statistical dependence measure in Definition~\ref{definition_eigenspectrum}.
\begin{definition}(Multivariate and total statistical dependence) We define the Multivariate Statistical Dependence (MSD) to be NCD's eigenspectrum $\{ \lambda_i \}_{i=1}^\infty$. In many scenarios, the total power of the spectrum is needed. In this paper, we construct the Total Statistical Dependence (TSD) and the corresponding truncated TSD as
\begin{equation}
\begin{gathered}
T = -\frac{1}{2} \sum_{i=1}^\infty \log({1-\lambda_i}), \;\; T_K = -\frac{1}{2} \sum_{i=1}^K \log({1-\lambda_i}).
\end{gathered} 
\end{equation} 
First, in this definition, their values can be unbounded when $\lambda = 1$, thus we set each eigenvalue with $\lambda_i \leftarrow \min(\lambda_i, 1-\epsilon)$, where $\epsilon$ is a small positive constant such that for each $i$, the term within the log is strictly positive. Second, since the first eigenvalue is a constant $1$, it can be discarded in the computation in practice, such that the measure becomes zero when independence occurs. In Appendix~\ref{TSD_DEFINITION}, we present a more general definition of TSD with all types of convex functions.
\label{definition_eigenspectrum}
\end{definition}



\noindent In the neural network community, measures or cost functions used for statistical dependence estimation have been limited to scalar quantities (e.g. mutual information, correlation coefficient), that can be derived from single or multidimensional sources. For a chosen order $K$, MSD, computed using NCD, defines a new Hilbert space of order $K$ that enables the quantification of all eigenvalues across the spatial or temporal structure of the random process up to order $K$. This allows for an accurate characterization of the error signal's pdf. Additionally, unlike ACE, solving the NCD eigenproblem yields an orthonormal family of functions. This allows for the full characterization of cross densities, surpassing the limitations of ACE that only capture two scalar-valued functions.

\subsection{Learning framework with r.p. realizations}
 \label{design_principle_rp}
A remaining issue is that the integration of the linear operator is w.r.t. the Lebesgue measure, while we usually have realizations of r.p. We can apply the standard variational technique of changing the measure in integration, a common practice in the GP literature (\citealt{williams2006gaussian}). We can apply this technique in our context as follows:
\begin{lemma}
Suppose NCD's decomposition yields a set of eigenvalues $\{ \lambda_i \}_{i=1}^\infty$ and eigenfunctions $\{{\phi_i}(x)\}_{i=1}^\infty$. In this case, there exists a set of functions $\{\widehat{\phi_i}(x)\}_{i=1}^\infty$ such that
\begin{equation}
\begin{gathered}
\int_{\mathcal{X}} \widehat{\phi_i}(x) \widehat{\phi_j}(x) p(x) dx = \begin{cases} 1, \;i=j& \\ 0, \;i\neq j \end{cases}\hspace{-8pt},\;\; i, j=1,2,\cdots, \infty,
\end{gathered}
\label{Orthonormal}
\end{equation}
\begin{equation}
\begin{gathered}
\mathbb{E}_{\mathbf{x}'}[\widehat{\phi_i} (\mathbf{x}')|x] = \int_{\mathcal{X}} p(x'|x) \widehat{\phi_i} (x') dx' = \lambda_i \widehat{\phi_i} (x)\;\; a.e.p(x), \;\; i=1,2,\cdots, \infty.\\
\end{gathered}
\label{conditional_mean}
\end{equation}
We refer to Equation~\eqref{Orthonormal} as the orthonormal condition and~\eqref{conditional_mean} as the equilibrium condition. Moreover, the set of functions $\{\widehat{\phi_i}(x)\}_{i=1}^\infty$ approximates a special probabilistic quantity, which we call the cross density ratio (CDR)
\begin{equation}
\begin{gathered}
 \frac{p(x, x')}{p(x) p(x')} = \sum_{i=1}^\infty \lambda_i \widehat{\phi_i}(x) \widehat{\phi_i}(x).
\end{gathered}
\label{final_equation}
\end{equation}
We denote the CDR as $\rho (x, x'):=\frac{p(x, x')}{p(x) p(x')}$. Similarly, there exists a set of functions $\{\widehat{\psi_i}(u)\}_{i=1}^\infty$ that satisfy the two conditions for the r.p. $\mathbf{u}$, and approximate the CDR \textcolor{black}{$\rho(u, u') := \frac{p(u, u')}{p(u) p(u')} = \sum_{i=1}^\infty \lambda_i \widehat{\psi_i}(u) \widehat{\psi_i}(u)$.}
\label{theorem1}
\end{lemma}
As this variational technique is a standard approach, a detailed proof can be found in Appendix~\ref{CDR_appendix}. This lemma reveals that the CDR's variational eigenproblem shares the same eigenvalues as MSD, with eigenfunctions differing only by the term $p^{-\frac{1}{2}}(x)$. For a function to be considered a CDR eigenfunction, the lemma poses two conditions: orthonormal and equilibrium conditions. Both of these conditions rely on empirical quantities, which enable direct optimization of learning systems using neural networks or other functional approximators.

\section{Solving the variational eigenproblem with functional maximal correlation algorithm}
\label{section_fmca}

In practical scenarios, solving the variational eigenproblem with a learning system requires implementing the bidirectional recursion procedure from realizations. Two algorithms closely related to our approach are ACE~(\citealt{breiman1985estimating}), already discussed, and the Maximal Correlation Algorithm (MCA)~(\citealt{hu2021mimo}). Our previous work, MCA, optimizes a neural network to determine a multivariate nonlinear function that maximizes the correlation ratio between input data and desired data, intended for regression and system identification. All these objective functions share a close connection to R\'enyi's maximal correlation~(\citealt{renyi1959measures}).

Building upon these works, we introduce a two neural network architecture as function approximators and the Functional Maximal Correlation Algorithm (FMCA) that follows an optimization scheme resembling recursive updates. FMCA's goal is to achieve equilibrium in the bidirectional recursion. It differs from ACE by approximating two sets of functions using one neural networks with multidimensional outputs for each, rather than just two scalar-valued functions. Concepts from the autocorrelation function (ACF), the cross-correlation function (CCF), and Wiener filters can be adopted to explain the learning and equilibrium condition.

Given two r.p. $\mathbf{x}$ and $\mathbf{u}$, as well as two networks $\mathbf{f}_\theta$ and $\mathbf{g}_\omega$, each network individually processes its corresponding r.p. to generate multivariate features from the respective r.p. realizations. Throughout each iteration, each network predicts the other's network output until they reach equilibrium. We propose a cost function based on the log determinants of ACFs and CCFs, designed to minimize the bidirectional prediction error while preserving the orthonormality of each set of functions.

Lemma~\ref{lemma_solution} and Theorem~\ref{main_theorem_parameterized_model} are the key theoretical contributions of this section, in which we prove that the objective function of FMCA will converge to the CDR's eigenvalues, and the two sets of functions represented by the networks will converge to the CDR's eigenfunctions. Our strategy for the proof is to show that the trained neural networks, as multivariate functions, satisfy both the orthonormality and equilibrium conditions specified in Lemma~\ref{theorem1}.

\begin{figure}[t]
\captionsetup{type=figure}
\centering
{\includegraphics[width=1\textwidth]{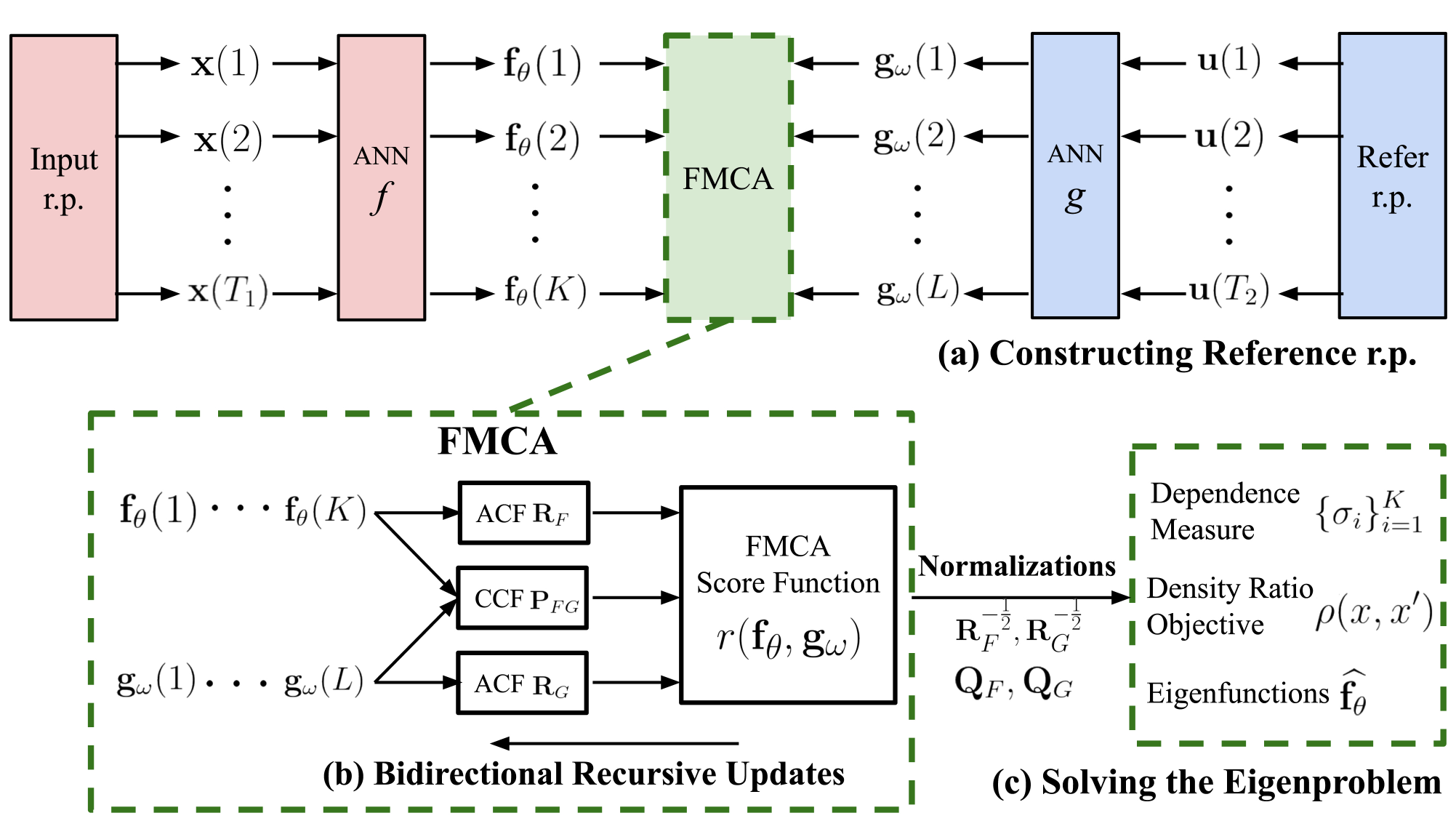}}\\
\captionof{figure}{\small Diagram describing the architecture and the FMCA cost function to find the leading eigenvalues of the NCD. The pipeline has three parts:  constructs the input and reference data latent spaces, optimizes the two neural networks, and obtains the functional decomposition by applying two normalization schemes. (a) Suppose two r.p. are given, FMCA is applied to quantify the statistical dependence between them. For the reference r.p., we will follow design rules introduced in Section~\ref{reference_process} for both supervised and unsupervised learning. (b) The optimization of the two neural networks follows a special bidirectional optimization scheme, where the output of each network is used as the target for the other in alternation. (c) After the networks are trained, we apply two normalization schemes to obtain the NCD eigenvalues and the CDR decomposition.}
\label{figure_fmca_pipeline}
\end{figure}

\subsection{R\'enyi's maximal correlation is the second largest eigenvalue of NCD}
\label{single_variate_representation}

We begin by discussing the connection between R\'enyi's maximal correlation (\citealt{renyi1959measures}) and MSD. As will be shown in Lemma~\ref{lemma_10}, the maximal correlation is in fact the second largest NCD eigenvalue. The solutions for the two scalar-valued functions used in the maximal correlation are the two CDR eigenfunctions.

Given two arbitrary scalar-valued functions $f:\mathcal{X}\rightarrow \mathbb{R}$ and $g:\mathcal{U}\rightarrow \mathbb{R}$, the maximal correlation is defined as the supremum of the correlation coefficient between $f(\mathbf{x})$ and $g(\mathbf{u})$ that can be obtained by searching over all feasible $f$ and $g$, and is expressed as
\begin{equation}
\begin{aligned}
\textbf{cc}(f(\mathbf{x}), g(\mathbf{u})) = \frac{\mathbb{E}[f(\mathbf{x}) g(\mathbf{u})] - \mathbb{E}[f(\mathbf{x})] \mathbb{E}[g(\mathbf{u})]}{\textbf{std}(f(\mathbf{x}))\, \textbf{std}(g(\mathbf{u}))},\;\;\textbf{cc}^*(\mathbf{x}, \mathbf{u}) = \sup_{f, g} \; \textbf{cc}(f(\mathbf{x}), g(\mathbf{u})),
\end{aligned}
\label{form1}
\end{equation}
where $\textbf{std}$ represents the standard deviation. 
\vspace{5pt}

\noindent \textbf{Recursive updates for maximal correlation:} In the above formulation, the two functions can be implemented by two neural networks $f_\theta$ and $g_\omega$. The supremum $\textbf{cc}^*$ in~\eqref{form1} can then be approximated by training these neural networks with gradient ascent using alternating optimization: with a fixed network $f_\theta$, the network $g_\omega$ is updated to match $f_\theta$ such that $g_\omega(u) \leftarrow \frac{\mathbb{E}_\mathbf{x}[f_\theta(\mathbf{x})|u]}{\text{std}(\mathbb{E}_\mathbf{x}[f_\theta(\mathbf{x})|\mathbf{u}])}$, thus increasing the correlation coefficient. Conversely, when $g_\omega$ is fixed, the network $f_\theta$ is updated to match $g_\omega$ such that $f_\theta(x) \leftarrow \frac{\mathbb{E}_\mathbf{u}[g_\omega(\mathbf{u})|x]}{\text{std}(\mathbb{E}_\mathbf{u}[g_\omega(\mathbf{u})|\mathbf{x}])}$. This procedure can be related to the ACE recursive update (as referred in Section~2). By introducing the variational system into the picture, we can theoretically analyze ACE's optimality in the following remark and lemma.

\begin{remark} The single-variate maximal correlation satisfies the following property:
\vspace{5pt}

(i) (Upper and lower bound) The maximal correlation is bounded: $0 \leq \mathbf{cc}^* \leq 1$;\vspace{5pt}

(ii) (Defining a metric) $r^* = \log(1-{\mathbf{cc}^*}^2)$ is a metric between $f$ and $g$;\vspace{5pt}

(iii) (Scale and shift-invariance)  Suppose $f$ and $g$ are the optimal solutions to the maximal correlation, then for any given scalars $a, b, c, d\in \mathbb{R}\backslash\{0\}$, the functions $a f+b$ and $c g+d$ are also the optimal solutions to the maximal correlation; \vspace{5pt}

(iv) (Vanishing condition) $\mathbf{cc}^*= 0$ if and only if two r.v. $\mathbf{x}$ and $\mathbf{u}$ are independent;\vspace{5pt}
 
(v) (Invariance to bijective transformations) Applying any bijective functions to $\mathbf{x}$ or $\mathbf{u}$ does not change the value of the maximal correlation.
\label{remark}
\end{remark}



\begin{lemma}

If $f$ and $g$ are optimal solutions for maximal correlation, we construct normalized functions $\overline{f}(x) = \frac{f(x) - \mathbb{E}(f(\mathbf{x}))}{\mathbf{std}(f(\mathbf{x}))}$ and $\overline{g}(u) = \frac{g(u) - \mathbb{E}(g(\mathbf{u}))}{\mathbf{std}(g(\mathbf{u}))}$. Their maximal correlation is the second-largest eigenvalue of CDR, while the functions themselves are the respective eigenfunctions. This can be expressed as $(\mathbf{cc}^*)^2 = \lambda_2$, $\overline{f}(x) = \widehat{\phi}_2(x)$ $a.e.p(x)$, and $\overline{g}(u) = \widehat{\psi}_2(u)$ $a.e.p(u)$.


\label{lemma_10}
\end{lemma}
The proof of Lemma~\ref{lemma_10} follows the Cauchy-Schwarz inequality and can be found in~\citealt{renyi1959measures}, with details included in the Appendix~\ref{appendix_mca}. While the algorithm and theory for maximal correlation are well-studied, its limitation as a scalar-valued quantity for r.v. makes it insufficient for many real-world applications. The following example illustrates this insufficiency.


\begin{example}
Measuring the statistical dependence of an r.p. requires more than the second largest eigenvalue. Given $\mathbf{x}= \{\mathbf{x}_t\}_{t=1}^{T_1}$ and $\mathbf{u}=\{\mathbf{u}_t\}_{t=1}^{T_2}$, assume a one-to-one mapping exists between the first element of each process, i.e., $\mathbf{x}(1)$ and $\mathbf{u}(1)$. The maps $\pi_x(\mathbf{x}) = \mathbf{x}(1)$ and $\pi_u(\mathbf{u}) = \mathbf{u}(1)$ satisfy $\mathbf{cc}(\pi_x, \pi_u) = 1$. Thus regardless of the length $T_1$, $T_2$ or the functional structure in $\{\mathbf{x}_t\}_{t=2}^{T_1}$ and $\{\mathbf{u}_t\}_{t=2}^{T_2}$, the maximal correlation remains $1$, implying an erroneous estimation of the r.p. functional structure.
\end{example}
A further limitation is that a single scalar-valued function representation compromises the accuracy of the CDR approximation, since only one eigenfunction of CDR is learned, regardless of the model's complexity or representational power. This shortcoming commonly occurs in neural networks, where the cost is only a scalar variable in regression or the class dimension in classification. These two issues motivated the introducion of the Functional Maximal Correlation Algorithm (FMCA).


%




\subsection{Functional maximal correlation algorithm}
\label{FMCA}
FMCA employs two sets of functions $\{f_i\}_{i=1}^K$ and $\{g_i\}_{i=1}^L$, also denoted as $\mathbf{f}=[f_1, \cdots, f_K]^\intercal:\mathcal{X} \rightarrow \mathbb{R}^{K}$ and $\mathbf{g}=[g_1, \cdots, g_L]^\intercal:\mathcal{U} \rightarrow \mathbb{R}^{L}$. Here we consider $K=L$ for simplicity. Due to the multidimensional nature of the formulation, the correlation coefficient used in MCA no longer applies; thus, a new cost function is needed to quantify the metric distance between these two sets of functions. Recall that in Definition~\ref{definition_eigenspectrum}, we defined the Total Statistical Dependence (TSD) as the total spectrum power. Hence, TSD will be used as an optimization objective here. Preserving R\'enyi's intuition of maximal correlation, we define the cost with the autocorrelation function (ACF) and the cross-correlation function (CCF) between network outputs. The following abbreviations will be used.

\begin{definition} Given $\mathbf{f}$ and $\mathbf{g}$, define:
\vspace{5pt}

(i) The ACF of $\mathbf{f}(\mathbf{x}):$ $\mathbf{R}_F := \mathbf{R}(\mathbf{f}(\mathbf{x})) = \mathbb{E}[\mathbf{f}(\mathbf{x})  \mathbf{f}^\intercal(\mathbf{x}) ] = \int_{\mathcal{X}} \mathbf{f}(x) \mathbf{f}^\intercal(x) p(x) dx;$ \vspace{5pt}

(ii) The ACF of $\mathbf{g}(\mathbf{u}):$ $\mathbf{R}_G := \mathbf{R}(\mathbf{g}(\mathbf{u})) = \mathbb{E}[\mathbf{g}(\mathbf{u})  \mathbf{g}^\intercal(\mathbf{u}) ]= \int_{\mathcal{U}} \mathbf{g}(u) \mathbf{g}^\intercal(u) p(u) du ;$ \vspace{5pt}

(iii) The CCF between $\mathbf{f}(\mathbf{x})$ and $\mathbf{g}(\mathbf{u}):$
\begin{equation}
\begin{gathered}
\mathbf{P}_{FG} := \mathbf{P}(\mathbf{f}(\mathbf{x}), \mathbf{g}(\mathbf{u})) = \mathbb{E}[\mathbf{f}(\mathbf{x})  \mathbf{g}^\intercal(\mathbf{u}) ] = \int_{\mathcal{X}}\int_\mathcal{U} \mathbf{f}(x) \mathbf{g}^\intercal(u) p(x, u) dx du;
\end{gathered}
\end{equation}

(iv) The joint ACF combining $\mathbf{f}(\mathbf{x})$ and $\mathbf{g}(\mathbf{u}):$
$\mathbf{R}_{FG} := \mathbf{R}(\mathbf{f}(\mathbf{x}), \mathbf{g}(\mathbf{u})) = \begin{bmatrix}\mathbf{R}_F & \mathbf{P}_{FG} \\
\mathbf{P}^\intercal_{FG} & \mathbf{R}_G
\end{bmatrix}.
$\vspace{5pt}

\label{definition_11}
\end{definition}
Based on Definition~\ref{definition_11}, we define the cost function of FMCA as follows:
\begin{equation}
\begin{gathered}
r(\mathbf{f}(\mathbf{x}), g(\mathbf{u})) = \log\det \mathbf{R}(\mathbf{f}(\mathbf{x}), \mathbf{g}(\mathbf{u}))-\log \det \mathbf{R}(\mathbf{f}(\mathbf{x})) - \log\det \mathbf{R}(\mathbf{g}(\mathbf{u})).
\end{gathered}
\label{fmca_cost}
\end{equation}
FMCA finds the infimum of the cost function~\eqref{fmca_cost}. This problem can be formulated as:
\begin{equation}
\begin{gathered}
r^*(\mathbf{x}, \mathbf{u}) = \inf_{\mathbf{f}, \mathbf{g}} r(\mathbf{f}(\mathbf{x}), \mathbf{g}(\mathbf{u})).
\end{gathered}
\label{fmca_maximal}
\end{equation}
For regularization, a small diagonal matrix $\epsilon \mathbf{I}$ is added to the ACF when computing the log determinant. As stated, during optimization, the two sets of functions, $\mathbf{f}$ and $\mathbf{g}$, are replaced by two multi-output neural networks, $\mathbf{f}_\theta$ and $\mathbf{g}_\omega$.





\vspace{5pt}
\noindent \textbf{Recursive updates of FMCA:} The bidirectional recursive updates of FMCA can be shown by taking the Schur complement of $\mathbf{R}_{FG}$, as:

\vspace{-7pt}


\begin{equation}
\begin{aligned}
\log\det \begin{bmatrix}{\mathbf{R}}_F & {\mathbf{P}}_{FG} \\
{\mathbf{P}}^\intercal_{FG} & {\mathbf{R}}_G
\end{bmatrix} &= \log\det {\mathbf{R}}_F + \log\det ({\mathbf{R}}_G  - {\mathbf{P}}^\intercal_{FG} {\mathbf{R}}_F^{-1} {\mathbf{P}}_{FG})\\
&= \log\det {\mathbf{R}}_G + \log\det ({\mathbf{R}}_F  - {\mathbf{P}}_{FG} {\mathbf{R}}_G^{-1} {\mathbf{P}}^\intercal_{FG}).
\end{aligned}
\label{equation_RFG}
\end{equation} 
Furthermore, using this form of $\mathbf{R}_{FG}$ in the cost function~\eqref{fmca_cost} yields
\begin{equation}
\begin{aligned}
r(\mathbf{f}(\mathbf{x}), g(\mathbf{u})) &= \log\det ({\mathbf{R}}_G  - {\mathbf{P}}^\intercal_{FG} {\mathbf{R}}_F^{-1} {\mathbf{P}}_{FG}) - \log\det {\mathbf{R}}_G \\
&= \log\det ({\mathbf{R}}_F  - {\mathbf{P}}_{FG} {\mathbf{R}}_G^{-1} {\mathbf{P}}^\intercal_{FG}) - \log\det {\mathbf{R}}_F.
\end{aligned}
\label{wiener_solution_schuml}
\end{equation}
For those familiar with Wiener filter theory, the argument of the first log determinant form in~\eqref{wiener_solution_schuml} is equivalent to the prediction error obtained from applying two optimal Wiener filters on each network outputs. In other words, the two networks become optimal linear projectors for each other, complemented by an additional term, either $\log\det {\mathbf{R}}_F$ or $\log\det {\mathbf{R}}_G$, which ensures an invertible ACF. Equation~\eqref{wiener_solution_schuml} shows that $r$ can be expressed in two different forms, depending on the direction in which the Schur complement is applied, corresponding to two optimal Wiener filters applied in different directions. The following subsection discusses our key contribution, which demonstrates that, through recursive updates, this cost function will approach the truncated TSD, and the optimization will solve the decomposition problem of CDR as the network learns the dominant eigenvalues and eigenfunctions.

\subsection{FMCA learns CDR's orthonormal decomposition}
\label{CORE-THEOREM}
The strategy for connecting FMCA and NCD is to prove that the functions optimized by recursive updates satisfy both the orthonormal condition~\eqref{Orthonormal} and the equilibrium condition~\eqref{conditional_mean}, as stated in Lemma~\ref{theorem1}. This section is divided into the following segments. \vspace{-5pt}

\begin{enumerate}[leftmargin=*]
\item \textbf{Lemma~\ref{lemma_less_than_0} and Lemma~\ref{LEMMA_ORTHONORMAL}:} The essential properties of FMCA are proved, including that the orthonormal condition can be easily satisfied through normalization, with $\overline{\mathbf{f}} = {\mathbf{R}_F}^{-\frac{1}{2}} \mathbf{f}$ and $\overline{\mathbf{g}} = {\mathbf{R}_G}^{-\frac{1}{2}} \mathbf{g}$. Our primary focus will then shift to proving the equilibrium condition, which requires the network outputs to be invariant w.r.t. the linear operator of the CDR.\vspace{-5pt}
\item \textbf{Lemma~\ref{matrix_inequality_lemma} and Corollary~\ref{corollary_matrix_inequality}:} A matrix inequality~\eqref{matrix_inequality} is introduced, analogous to the Cauchy-Schwarz inequality for single-variate maximal correlation. The matrix inequality becomes equality when Wiener solutions, $\mathbf{m}_F = \overline{\mathbf{P}} \overline{\mathbf{g}}$ and $\mathbf{m}_G = \overline{\mathbf{P}}^\intercal \overline{\mathbf{f}}$, are satisfied. \vspace{-5pt}

\item \textbf{Lemma~\ref{lemma_solution}:} Wiener solutions imply that the equlibrium condition~\eqref{conditional_mean} is satisfied, upon normalization $\widehat{\mathbf{f}} = {\mathbf{Q}_F}^\intercal \overline{\mathbf{f}}$ and $\widehat{\mathbf{g}} = {\mathbf{Q}_G}^\intercal \overline{\mathbf{g}}$. \vspace{-5pt}

\item\textbf{Theorem~\ref{main_theorem_parameterized_model}:} Wiener solutions are equivalent to the cost optimality, proven by examining the cost's first-order partial derivatives. Due to its length, the proof of Theorem~\ref{main_theorem_parameterized_model} can be found in Appendix~\ref{main_proof}. 
\end{enumerate}\vspace{-5pt}
\noindent We now begin with the essential properties of FMCA and the orthonormal condition. \vspace{-5pt}

\begin{lemma}
(Cost function’s upper bound) For any $\mathbf{f}$ and $\mathbf{g}$ with invertible $\mathbf{R}_F$ and $\mathbf{R}_G$, the FMCA cost function is upper bounded by $0$, i.e.
\begin{equation}
\begin{aligned}
r(\mathbf{f}(\mathbf{x}), \mathbf{g}(\mathbf{u})) \leq 0,
\end{aligned}
\end{equation}
and subsequently the optimal cost function $r^*(\mathbf{x}, \mathbf{u}) \leq 0$.
\label{lemma_less_than_0}
\end{lemma}

\begin{lemma}
(Enforcing orthonormality) For any given $\mathbf{f}$ and $\mathbf{g}$ with invertible $\mathbf{R}_F$ and $\mathbf{R}_G$, construct normalizations
\begin{equation}
\begin{aligned}
\overline{\mathbf{f}} = {\mathbf{R}_F}^{-\frac{1}{2}} \mathbf{f}, \;\; \overline{\mathbf{g}} = {\mathbf{R}_G}^{-\frac{1}{2}} \mathbf{g}.
\end{aligned}
\label{normalize_inverse}
\end{equation}
Each of the two sets of functions, represented by $\overline{\mathbf{f}}$ or $\overline{\mathbf{g}}$, satisfies the orthonormal condition~\eqref{Orthonormal}, meaning they are orthonormal w.r.t. the marginal distributions. Moreover, these sets of functions satisfy the following properties:
\vspace{3pt}

\noindent (i) Normalization does not change the cost: $r(\overline{\mathbf{f}}, \overline{\mathbf{g}}) = r(\mathbf{f}, \mathbf{g});$ \vspace{3pt}

\noindent (ii) The cost is invariant under affine transformations. Any pair of elements from sets
\begin{equation}
\begin{aligned}
\mathcal{F} &= \{\mathbf{A} \overline{\mathbf{f}} : \mathbf{A} \in \mathbb{R}^{K \times K} \hspace{-4pt}\text{ and full rank}\},\;\;\mathcal{G} = \{\mathbf{B} \overline{\mathbf{g}} : \mathbf{B} \in \mathbb{R}^{L \times L} \text{ and full rank}\}.
\end{aligned}
\label{solution_set}
\end{equation}
share the same cost function as $\overline{\mathbf{f}}$ and $\overline{\mathbf{g}}$. Assuming $\mathbf{A}\mathbf{A}^\intercal = \mathbf{I}_K$ and $\mathbf{B}\mathbf{B}^\intercal = \mathbf{I}_L$, the transformation will also preserve orthonormality.
\label{LEMMA_ORTHONORMAL}
\end{lemma}

\begin{proof}
The proofs of Lemmas~\ref{lemma_less_than_0} and~\ref{LEMMA_ORTHONORMAL} follow conventional log determinant inequalities (\citealt{marcus1992survey}) by normalizing $\mathbf{f}$ and $\mathbf{g}$ to diagonalize $\mathbf{R}_{FG}$. Given $\mathbf{R}_F$ and $\mathbf{R}_G$ are full rank, their negative half powers, $\mathbf{R}_F^{-\frac{1}{2}}$ and $\mathbf{R}_G^{-\frac{1}{2}}$, are used for this diagonalization:
\begin{equation}
\begin{aligned}
\begin{bmatrix}\mathbf{R}^{-\frac{1}{2}}_F & \mathbf{0} \\
\mathbf{0} & \mathbf{R}^{-\frac{1}{2}}_G
\end{bmatrix} \begin{bmatrix}\mathbf{R}_F & \mathbf{P}_{FG} \\
\mathbf{P}^\intercal_{FG} & \mathbf{R}_G.
\end{bmatrix} \begin{bmatrix}\mathbf{R}^{-\frac{1}{2}}_F & \mathbf{0} \\
\mathbf{0} & \mathbf{R}^{-\frac{1}{2}}_G
\end{bmatrix} = \begin{bmatrix} \mathbf{I}_K & \mathbf{R}^{-\frac{1}{2}}_F \mathbf{P}_{FG} \mathbf{R}^{-\frac{1}{2}}_G \\
\mathbf{R}^{-\frac{1}{2}}_G \mathbf{P}^\intercal_{FG} \mathbf{R}^{-\frac{1}{2}}_F & \mathbf{I}_L
\end{bmatrix}:= \overline{\mathbf{R}_{FG}}.
\end{aligned}
\label{normalized}
\end{equation}
Since the diagonal elements of $\overline{\mathbf{R}_{FG}}$ are $1$, its trace, $\text{Tr}(\overline{\mathbf{R}_{FG}}) = K+L$, is constant for any feasible $\mathbf{f}$ and $\mathbf{g}$. Denoting the eigenvalues of $\overline{\mathbf{R}_{FG}}$ as $\{\lambda_i(\overline{\mathbf{R}_{FG}})\}_{i=1}^{K+L}$, we obtain:
\begin{equation}
\begin{aligned}
\text{Tr}(\overline{\mathbf{R}_{FG}}) = \sum_{i=1}^{K+L} \lambda_i(\overline{\mathbf{R}_{FG}}) = K+L.
\end{aligned}
\label{determinant_inequality}
\end{equation}
The next step is to apply the inequality of arithmetic and geometric means to the eigenvalues:
\begin{equation}
\begin{aligned}
\det \overline{\mathbf{R}_{FG}} = \prod_{i=1}^{K+L} \lambda_i(\overline{\mathbf{R}_{FG}}) 
\leq \left(\frac{1}{K+L}\sum_{i=1}^{K+L} \lambda_i(\overline{\mathbf{R}_{FG}}) \right)^{K+L}= 1,
\end{aligned}
\label{inequality}
\end{equation}
The cost function, $r(\mathbf{f}(\mathbf{x}), \mathbf{g}(\mathbf{u}))$, can be written as the log determinant of $\overline{\mathbf{R}_{FG}}$: \vspace{2pt}
\begin{equation}
\begin{aligned}
\log\det \overline{\mathbf{R}_{FG}} = \log\det \mathbf{R}_{FG} - \log\det \mathbf{R}_F -  \log\det \mathbf{R}_G := r(\mathbf{f}(\mathbf{x}), \mathbf{g}(\mathbf{u})).
\end{aligned}
\label{equation_31}
\end{equation}
\noindent Applying the inequality~\eqref{inequality} to the cost~\eqref{equation_31} yields $r(\mathbf{f}(\mathbf{x}), \mathbf{g}(\mathbf{u}))\leq 0$. Since the optimal cost, $r^*(\mathbf{x}, \mathbf{u})$, is the infimum, it follows that $r^*(\mathbf{x}, \mathbf{u})\leq 0$.

The orthonormal condition follows immediately. By constructing normalizations~\eqref{normalize_inverse}, it can be shown that $\int \overline{\mathbf{f}}(x) \overline{\mathbf{f}}(x)^\intercal p(x) dx = \mathbf{I}_K$ and $\int \overline{\mathbf{g}}(u) \overline{\mathbf{g}}(u)^\intercal p(u) du = \mathbf{I}_L$. Any affine transformations of two functions with a pair of full-rank matrices, $\mathbf{A}$ and $\mathbf{B}$, do not change the value of the cost function:
\begin{equation}
\resizebox{1\linewidth}{!}{
$\begin{gathered}
\mathbf{R}(\mathbf{A} \overline{\mathbf{f}}) = \mathbf{A} \mathbf{A}^\intercal,\;\;
\mathbf{R}(\mathbf{B} \overline{\mathbf{g}}) = \mathbf{B} \mathbf{B}^\intercal, \;\; 
\mathbf{P}(\mathbf{A} \overline{\mathbf{f}}, \mathbf{B} \overline{\mathbf{g}}) = \mathbb{E}[\mathbf{A} \overline{\mathbf{f}}\;\overline{\mathbf{g}}^\intercal \mathbf{B}^\intercal]= \mathbf{A} \overline{\mathbf{P}_{FG}} \mathbf{B}^\intercal,\;\;\;\;\;\;\;\;\;\;\;\;\;\;\;\;\;\;\;\;\;\;\;\;\;\;\;\;\;\;\; \vspace{3pt}\\
\mathbf{R}(\mathbf{A} \overline{\mathbf{f}}, \mathbf{B} \overline{\mathbf{g}}) = \begin{bmatrix}\mathbf{R}(\mathbf{A} \overline{\mathbf{f}}) & \mathbf{P}(\mathbf{A} \overline{\mathbf{f}}, \mathbf{B} \overline{\mathbf{g}}) \vspace{3pt}\\
\mathbf{P}^\intercal(\mathbf{A} \overline{\mathbf{f}}, \mathbf{B} \overline{\mathbf{g}}) & \mathbf{R}(\mathbf{B} \overline{\mathbf{g}})
\end{bmatrix} = \begin{bmatrix}\mathbf{A} & \mathbf{0} \\
\mathbf{0} & \mathbf{B}
\end{bmatrix} \overline{\mathbf{R}_{FG}} \begin{bmatrix}\mathbf{A}^\intercal & \mathbf{0} \\
\mathbf{0} & \mathbf{B}^\intercal 
\end{bmatrix} \Rightarrow r(\mathbf{A} \overline{\mathbf{f}}, \mathbf{B} \overline{\mathbf{g}}) = r(\overline{\mathbf{f}}, \overline{\mathbf{g}}).
\end{gathered}$}
\label{quantities_affine}
\end{equation}
Therefore, the normalized functions $\overline{\mathbf{f}}$ and $\overline{\mathbf{g}}$ form two sets~\ref{solution_set}, making the cost function invariant to any affine transformation.
\end{proof}
The next step is to further normalize functions $\overline{\mathbf{f}}$ and $\overline{\mathbf{g}}$ such that they satisfy the equilibrium condition~\eqref{conditional_mean}. This involves the eigenfunctions and eigenvalues of the matrices $\overline{\mathbf{P}}\;\overline{\mathbf{P}}^\intercal$ and $\overline{\mathbf{P}}^\intercal\;\overline{\mathbf{P}}$. Lemma~\ref{matrix_inequality_lemma} and Corollary~\ref{corollary_matrix_inequality} will first show a special matrix inequality that coincides with the Cauchy-Schwarz inequality in single-variate maximal correlation. The condition for this matrix inequality to become equality is the well-known Wiener solutions, meaning that each function's projection is the conditional mean of the other.







\begin{lemma}
(Matrix inequalities for bidirectional recursion) Let the orthonormal sets of functions $\overline{\mathbf{f}}$ and $\overline{\mathbf{g}}$ be given. Denote their CCF as $\overline{\mathbf{P}} := \mathbb{E}[\overline{\mathbf{f}} \; \overline{\mathbf{g}}^\intercal]$. We investigate the conditional expectations $\mathbf{m}_G(x) := \mathbb{E}_{\mathbf{u}}[\overline{\mathbf{g}}(\mathbf{u})|x]$ and $\mathbf{m}_F(u) := \mathbb{E}_{\mathbf{x}}[\overline{\mathbf{f}}(\mathbf{x})|u]$. The two conditional expectations are two mapping functions $\mathbf{m}_G:\mathcal{X} \rightarrow \mathbb{R}^L$ and $\mathbf{m}_F:\mathcal{U} \rightarrow \mathbb{R}^K$. Since they are functions of $x$ and $u$, they also have ACFs, which we denote as \vspace{-5pt}
\begin{equation}
\begin{aligned}
\mathbf{R}_{{G|x}} =  \int_{\mathcal{X}}\mathbf{m}_G (x) \mathbf{m}_G^\intercal(x) p(x) dx, \;\; \mathbf{R}_{{F|u}} =  \int_{\mathcal{U}}\mathbf{m}_F (u) \mathbf{m}_F^\intercal(u) p(u) du.
\end{aligned}
\label{COVARIANCE_CONDITIONAL_MEAN}
\end{equation}
The following matrix inequality holds: \vspace{-1pt}
\begin{equation}
\begin{aligned}
\overline{\mathbf{P}}^\intercal\overline{\mathbf{P}} \preceq \mathbf{R}_{{G|x}}, \;\;   
 \overline{\mathbf{P}}\;\overline{\mathbf{P}}^\intercal  \preceq \mathbf{R}_{{F|u}}. 
\end{aligned}
\label{matrix_inequality}
\end{equation}
Suppose $\mathbf{m}_F = \overline{\mathbf{P}} \overline{\mathbf{g}}$, the equality $\mathbf{R}_{{F|u}} = \overline{\mathbf{P}}\;\overline{\mathbf{P}}^\intercal$ will hold. Similarly, if $\mathbf{m}_G = \overline{\mathbf{P}}^\intercal \overline{\mathbf{f}}$, one finds $\mathbf{R}_{{G|x}} = \overline{\mathbf{P}}^\intercal\overline{\mathbf{P}}$. These conditions coincide with the the Wiener solutions of two function sets.
\label{matrix_inequality_lemma}
\end{lemma}

\begin{corollary} (The form of bidirectional Wiener solutions) 
Applying the Schur complement to the cost function $r(\overline{\mathbf{f}}, \overline{\mathbf{g}})$ yields 
\begin{equation}
\begin{aligned}
r(\overline{\mathbf{f}}, \overline{\mathbf{g}}) = \log\det (\mathbf{I}_K - \overline{\mathbf{P}}\;\overline{\mathbf{P}}^\intercal) = \log\det (\mathbf{I}_L - \overline{\mathbf{P}}^\intercal\overline{\mathbf{P}}). 
\end{aligned}
\label{score_function_complement}
\end{equation}
Suppose the Wiener solutions hold, i.e., if their conditional expectations satisfy $\mathbf{m}_F = \overline{\mathbf{P}} \overline{\mathbf{g}}$ and $\mathbf{m}_G = \overline{\mathbf{P}}^\intercal \overline{\mathbf{f}}$, then the matrix inequalities~\eqref{matrix_inequality} will become equalities, and the cost function will take the form \vspace{-10pt}
\begin{equation}
\begin{aligned}
r(\overline{\mathbf{f}}, \overline{\mathbf{g}})=  \log\det(\mathbf{I}_K - \mathbf{R}_{{F|u}}) =  \log\det(\mathbf{I}_L - \mathbf{R}_{{G|x}}).
\end{aligned}
\label{inequality_log_det}
\end{equation}

\label{corollary_matrix_inequality}
\end{corollary}

\begin{proof} The matrix inequality~\eqref{matrix_inequality} is given by the property that the autocorrelation matrix is positive semidefinite. With the definitions in~\eqref{COVARIANCE_CONDITIONAL_MEAN}, if we compute the CCF between functions $\mathbf{m}_F (u)$ and $\mathbf{g}(u)$, by the tower property of conditional expectation, it yields
\begin{equation}
\begin{aligned}
\int \mathbf{m}_F (u) \overline{\mathbf{g}}^\intercal(u) p(u) du = \int  \mathbb{E}[\overline{\mathbf{f}}(\mathbf{x})|u] \overline{\mathbf{g}}^\intercal(u) p(u) du = \iint \overline{\mathbf{f}}(x) \overline{\mathbf{g}}^\intercal(u) p(x, u) dx du := \overline{\mathbf{P}}.
\end{aligned}
\label{tower_property}
\end{equation}
It can be deduced that $\int \mathbf{m}_G (x) \overline{\mathbf{f}}^\intercal(x) p(x) dx = \overline{\mathbf{P}}^\intercal$ from the same analysis. This implies that the joint ACF $\mathbf{R} (\mathbf{m}_F, \overline{\mathbf{g}})$ takes the following form:
\begin{equation}
\begin{aligned}
\mathbf{R} (\mathbf{m}_F, \overline{\mathbf{g}})= \begin{bmatrix}\mathbf{R}_{{F|u}} & \overline{\mathbf{P}} \\
\overline{\mathbf{P}}^\intercal & \mathbf{I}_L 
\end{bmatrix} \Rightarrow \log \det  \mathbf{R} (\mathbf{m}_F, \overline{\mathbf{g}}) = \log \det (\mathbf{R}_{{F|u}} - \overline{\mathbf{P}}\;
\overline{\mathbf{P}}^\intercal). 
\end{aligned}
\end{equation}
This ACF $\mathbf{R}(\mathbf{m}_F, \overline{\mathbf{g}})$, with the conditional expectation, deviates from $\mathbf{R}(\overline{\mathbf{f}}, \overline{\mathbf{g}}) := \begin{bmatrix}\mathbf{I}_K & \overline{\mathbf{P}} \\ \overline{\mathbf{P}}^\intercal & \mathbf{I}_L \end{bmatrix}$ only in terms of $\mathbf{R}_{{F|u}}$, attributed to the tower property in~\eqref{tower_property}. Since $\mathbf{R} (\mathbf{m}_F, \overline{\mathbf{g}})$ is an ACF, it is positive semidefinite by definition. Consequently, its Schur complement is also positive semidefinite, and therefore, $\overline{\mathbf{P}}\;\overline{\mathbf{P}}^\intercal \preceq \mathbf{R}_{{F|u}}$, as well as $\overline{\mathbf{P}}^\intercal\overline{\mathbf{P}} \preceq \mathbf{R}_{{G|x}}$. The two inequalities become equalities when $\mathbf{m}_F = \overline{\mathbf{P}} \overline{\mathbf{g}}$ and $\mathbf{m}_G = \overline{\mathbf{P}}^\intercal \overline{\mathbf{f}}$, respectively.

\end{proof}
After presenting the forms of Wiener solutions, we proceed to demonstrate in Lemma~\ref{lemma_solution} that these solutions imply the eigenfunctions of CDR, upon normalization. In Theorem~\ref{main_theorem_parameterized_model}, we further demonstrate the equivalence of Wiener solutions and cost optimality. Hence, it will be shown that achieving cost optimality is equivalent to solving the CDR eigenproblem.





\begin{lemma} 
(Wiener solutions imply eigenfunctions) Suppose both $\mathbf{m}_F = \overline{\mathbf{P}} \overline{\mathbf{g}}$ and $\mathbf{m}_G = \overline{\mathbf{P}}^\intercal \overline{\mathbf{f}}$ are satisfied. Solve the matrix eigenproblem
\begin{equation}
\begin{aligned}
\overline{\mathbf{P}}\, \overline{\mathbf{P}}^{\intercal} = \mathbf{Q}_F \mathbf{\Sigma} {\mathbf{Q}_F}^\intercal, \;\;\overline{\mathbf{P}}^{\intercal} \overline{\mathbf{P}}= \mathbf{Q}_G \mathbf{\Sigma} {\mathbf{Q}_G}^\intercal,
\end{aligned}
\end{equation}
and construct two new sets of functions $\widehat{\mathbf{f}} = {\mathbf{Q}_F}^\intercal \overline{\mathbf{f}}$ and $\widehat{\mathbf{g}} = {\mathbf{Q}_G}^\intercal \overline{\mathbf{g}}$, equivalently $\widehat{\mathbf{f}} = {\mathbf{Q}_F}^\intercal \mathbf{R}_F^{-\frac{1}{2}} \mathbf{f}$ and $\widehat{\mathbf{g}} = {\mathbf{Q}_G}^\intercal \mathbf{R}_G^{-\frac{1}{2}} \mathbf{g}$. The following properties can be proven for these quantities:
\vspace{5pt}

\noindent (i) Functions $\widehat{\mathbf{f}}$ and $\widehat{\mathbf{g}}$ satisfy both orthonormal~\eqref{Orthonormal} and equilibrium conditions~\eqref{conditional_mean} defined in Lemma~\ref{theorem1}. Therefore, they are eigenfunctions of CDR. \vspace{5pt}

\noindent (ii) Eigenvalues in $\mathbf{\Sigma}$, shared by $\overline{\mathbf{P}}\, \overline{\mathbf{P}}^{\intercal}$ and $\overline{\mathbf{P}}^{\intercal}\overline{\mathbf{P}}$, are the dominant eigenvalues of CDR (i.e., MSD).

\noindent (iii) The infimum of the cost is attained and the optimal cost corresponds to the total power of statistical dependence (i.e, TSD), defined in Definition~\ref{definition_eigenspectrum}.

\label{lemma_solution}
\end{lemma}

\begin{proof}
Given the cross density $p(x'|x) = \int p(x'|u) p(u|x) du$. Using both conditions  $\mathbf{m}_F = \overline{\mathbf{P}} \overline{\mathbf{g}}$ and $\mathbf{m}_G = \overline{\mathbf{P}}^\intercal \overline{\mathbf{f}}$ to the conditional expectation $\mathbb{E}[\mathbf{f}(\mathbf{x'})|x]$, we derive:
\begin{equation}
\resizebox{1\linewidth}{!}{
$\begin{aligned}
\int_\mathcal{X} \overline{\mathbf{f}}(x') p(x'|x) dx' = \int_\mathcal{X} \overline{\mathbf{f}}(x') \int_\mathcal{U} p(x'|u) p(u|x) du dx' = \int_\mathcal{U} \mathbf{m}_F(u) p(u|x) du = \overline{\mathbf{P}} \mathbf{m}_G =  \overline{\mathbf{P}}\;\overline{\mathbf{P}}^\intercal \overline{\mathbf{f}}. 
\end{aligned}$}
\label{equation_decomposition}
\end{equation}
Therefore, the conditional mean satisfies $\mathbb{E}[\mathbf{f}(\mathbf{x'})|x] = \overline{\mathbf{P}}\;\overline{\mathbf{P}}^\intercal \overline{\mathbf{f}}$. Utilizing the eigendecomposition $\overline{\mathbf{P}}\, \overline{\mathbf{P}}^{\intercal} = \mathbf{Q}_F \mathbf{\Sigma} {\mathbf{Q}_F}^\intercal$, we derive $\mathbf{Q}_F^\intercal \mathbb{E}[\mathbf{f}(\mathbf{x'})|x] = \mathbf{\Sigma} \mathbf{Q}_F^\intercal \overline{\mathbf{f}}(x)$. Further substituting $\widehat{\mathbf{f}} = {\mathbf{Q}_F}^\intercal \overline{\mathbf{f}}$, we obtain:\vspace{-5pt}
\begin{equation}
\begin{aligned}
\int \widehat{\mathbf{f}}(x') p(x'|x) dx' = \mathbb{E}[\widehat{\mathbf{f}}(\mathbf{x'})|x] = \mathbf{\Sigma} \widehat{\mathbf{f}}(x) 
\end{aligned}
\vspace{-5pt}
\end{equation}
Denote $\widehat{\mathbf{f}} = [\widehat{{f}}_1, \cdots, \widehat{{f}}_K]^\intercal$ and the diagonal elements of $\Sigma$ as $\{\sigma_1, \cdots, \sigma_K\}$. Each function $\widehat{{f}}_k$ is invariant to the conditional mean operator of the cross density, with an eigenvalue $\sigma_k$, satisfying the equilibrium condition defined in~\eqref{conditional_mean}. Therefore, every function $\widehat{{f}}_k$ of $\widehat{\mathbf{f}}$ is an eigenfunction of the CDR $\rho(x, x'):= \frac{p(x, x')}{p(x)p(x')}$, with $\{\sigma_1, \cdots, \sigma_K\}$ being their eigenvalues.



Assuming dimensionalities $K\leq L$, the cost function $r(\overline{\mathbf{f}}, \overline{\mathbf{g}}) = \log\det (\mathbf{I}_K - \overline{\mathbf{P}}\;\overline{\mathbf{P}}^\intercal) = \sum_{i=1}^K\log (1-\sigma_i)$. As the cost function reaches its minimum, the eigenvalues should converge to the dominant eigenvalues in MSD. Consequently, the negative value of the cost becomes the truncated TSD, as defined in Definition~\ref{definition_eigenspectrum}, where $r(\overline{\mathbf{f}}, \overline{\mathbf{g}}) = -2T_K$.


We have now demonstrated that upon achieving optimality, the normalized functions $\widehat{\mathbf{f}} = {\mathbf{Q}_F}^\intercal \overline{\mathbf{f}}$ and $\widehat{\mathbf{g}} = {\mathbf{Q}_G}^\intercal \overline{\mathbf{g}}$ reach the leading eigenfunctions of CDR. Additionally, the eigenvalues from $\overline{\mathbf{P}}\, \overline{\mathbf{P}}^{\intercal}$ and $\overline{\mathbf{P}}^{\intercal} \overline{\mathbf{P}}$ reach the leading eigenvalues of CDR (i.e., MSD). The negative value of the optimal cost reaches the truncated TSD.
\end{proof}

We have now established that Wiener solutions solve the CDR eigenproblem. The following analysis demonstrates that these solutions are equivalent to the optimality of the cost function, further implying that cost optimality implies solving the CDR eigenproblem. Theorem~\ref{main_theorem_parameterized_model} demonstrates this claim by analyzing the first-order derivative of the cost function. A detailed proof of this theorem, presented in Appendix~\ref{main_proof}, shows that satisfying the Wiener solutions is equivalent to maximizing the eigenvalues of $\overline{\mathbf{P}}\;\overline{\mathbf{P}}^\intercal$, thus achieving the optimal cost function.

\begin{theorem}
(Equivalence of Wiener solutions and optimality) Suppose the two functions are now parameterized models $\mathbf{f}_\theta$ and $\mathbf{g}_\omega$. Assume the models are universal. The notations of matrices and normalizations are kept unchanged. The partial derivative of $r(\mathbf{f}_\theta, \mathbf{g}_\omega)$ w.r.t. all parameters reaches zero if and only if the two sets of functions satisfy the Wiener solution bidirectionally, i.e., they are the optimal linear projectors of each other:
\begin{equation}
\begin{aligned}
\mathbf{P}_{FG}^\intercal \mathbf{R}_F^{-1}  \mathbf{f}_\theta(x) = \mathbb{E}[\mathbf{g}_\omega(\mathbf{u})|x], \;\; \mathbf{P}_{FG} \mathbf{R}_G^{-1} \mathbf{g}_\omega(u) = \mathbb{E}[\mathbf{f}_\theta(\mathbf{x})|u]. 
\end{aligned}
\label{parameterized_model_condition}
\end{equation}
Consequently, upon normalization, the two parameterized models will satisfy all properties outlined in Lemma~\ref{lemma_solution} as leading eigenfunctions of CDR.
\label{main_theorem_parameterized_model}
\end{theorem}
Nonconvexity in the optimization still presents an unsolved challenge. To execute the exact update, the system must achieve the infimum of the cost at each recursive step, which is clearly infeasible for neural networks due to the large number of local minima. Applying convex learning machines to address this issue is an ongoing area of research.

\section{FMCA applications}
\label{reference_process}

In this section, we describe how FMCA can be employed for classification and regression, as well as applications for Markov chain aggregation and learning image codes.


Suppose the input data from the real world consists of a high-dimensional, non-stationary r.p. with a rich and complex functional structure, such as an image dataset. This input r.p. is fed into the primary network, $\mathbf{f}_\theta$. A second r.p. with lower dimension and/or less complexity, referred to as the reference r.p., may also be available and is fed into the reference network, $\mathbf{g}_\omega$. They can be used in  specific learning tasks, such as image labels for classification, a target r.p. for regression, or simply uniform, unstructured random noise for variational applications.  Outputs of the two networks, which contain the functional structure of the r.p., construct an embedded projection space and can be seen as codes for images that unify supervised and unsupervised learning tasks. Depending on the design of the reference r.p., various forms of codes and CDR can be achieved, demonstrating the flexibility of the architecture.



\subsection{Conventional classification and regression}

\noindent \textbf{Classification (FMCA-C).} In classification, image data is fed to $\mathbf{f}_\theta$ and labels to $\mathbf{g}_\omega$. The reference r.p. here is a one-hot vector representing labels with the ideal CDR ${\rho}(x, x') = 1$ for samples in the same class and $0$ otherwise. Output dimensions for both networks are usually set to be the same in practice. Once the projection space is trained, testing is performed by training a small Multi-Layer Perceptron (MLP) that projects the learned latent codes of the primary r.p. (images), into the space of the reference r.p. (labels) to determine the classification accuracy. This architecture, FMCA-C, will be implemented in the experiments. \vspace{5pt}

\noindent \textbf{Regression.} FMCA can be easily adapted for regression or time series filtering. For a finite-horizon time series $\mathbf{x}(1), \mathbf{x}(2), \cdots$ paired with the target signal $\mathbf{d}(1), \mathbf{d}(2), \cdots$, the goal is to predict $\mathbf{d}(t)$ using past samples of the input signal. The primary and reference r.p. are chosen to be the input and target signals, respectively. For a user-specified order $\tau$, network $\mathbf{f}_\theta$ receives the input signal of length $\tau$, and network $\mathbf{g}_\omega$ receives one or multiple future samples of the target signal $\mathbf{d}(t)$. Both networks are trained until convergence, yielding time-invariant eigenfunctions in the embedding space that encode the functional structure of each signal. The parameters are then fixed. 

A regressor, typically a linear model, is trained to project the learned eigenfunctions of the primary r.p. (input signal) to the space of the reference r.p. (target signal) to obtain a prediction by minimizing the Mean-Squared Error (MSE) criterion, thereby producing the optimal Wiener projection. For testing, the input is fed to the primary network and the linear regressor output provides the prediction. More details about applying FMCA to nonlinear filtering can be found in our conference paper~(\citealt{hu2023cross}).
\vspace{-5pt}
\subsection{Markov chain aggregation}
\label{aggregation_problem}

Markov chain aggregation requires solving a decomposition of the transition probability $p(x_{t+1}|x_t) = \sum_{k=1}^r p(x_{t+1}|z_t=k)p(z_t=k|x_{t})$, where each $z_t\in \{1, 2, \cdots, r\}$ is an assignment of each state to its discrete aggregated state. The probability $p(x_{t+1}|z_t=k)$ represents the state assignment, independent of $t$~(\citealt{duan2019state}). Baselines related to ours include Spectral State Compression (SSC)~(\citealt{zhang2019spectral}), soft-state aggregations with anchor states~(\citealt{duan2019state}), and Nonnegative Matrix Factorization (NMF)~(\citealt{lee1999learning, lee2000algorithms, donoho2003does}). These methods use a similar spectrum decomposition approach as FMCA but decompose a different quantity and often require estimating the $pdf$ empirically.



\vspace{5pt}

\noindent \textbf{Deterministic state aggregation with CDR.} Here, the previous time step of the chain is treated as $\mathbf{x}$ (input for $\mathbf{f}_\theta$), while the current time step is treated as $\mathbf{u}$ (input for $\mathbf{g}_\omega$). FMCA approximates their two NCDs, along with their basis functions. Once the networks are trained, the top $r$ basis functions, spanning an $r$-dimensional space, are used to derive the $r$ aggregated states. In this subspace, we find $r$ centroids with the optimal $L_2$ distance to partition the state space. \vspace{-5pt}
\begin{equation}
\resizebox{0.6\textwidth}{!}{
$\begin{gathered}
\min_{\Omega_1, \cdots, \Omega_r} \text{min}_{v_1, \cdots, v_r\in \mathbb{R}^r} \sum_{s=1}^r \int_{x_t \in \Omega_s } ||\widehat{\mathbf{f}}_\theta(x_t) - v_s||_2^2 \;p(x_t)\; dx_t,
\end{gathered}$}
\end{equation}
where $v_1, \cdots, v_r$ are the centroids, and $\Omega_1, \cdots, \Omega_r$ are the $r$ partitions of the state space.
\vspace{5pt}


\noindent \textbf{Soft-state aggregation.} The goal of soft-state aggregation is to learn the distribution $p(z_t=k|x_{t})$, not just the states themselves. In FMCA, a direct approach is to use the softmax function as the activation function in the networks' final layer and set the output dimension to $r$. By doing so, training the network will inherently learn a distribution.
\vspace{5pt}


\noindent Further comparisons with the baselines and their implementations are provided in Section~\ref{aggregation_section} and Appendix~\ref{implementation_details}.




 \vspace{-5pt}
\subsection{Learning r.p. encoders of images} 
\label{imagerp}\vspace{-5pt}



FMCA measures the statistical dependence between pairs of two r.p., $\mathbf{x}$ and $\mathbf{u}$. In this application, the image dataset is assigned to the primary r.p., $\mathbf{x}$. We present various methods for constructing the reference r.p., $\mathbf{u}$, to define the joint density $p(x, u)$. We demonstrate that the same architecture can be used for unsupervised representation learning applications, emphasizing the significance of eigenfunctions in learning representations.


\vspace{5pt}

\noindent \textbf{Maximal functional space codes (FMCA-M).} The simplest case is to use the original image as both the primary and the reference r.p. This produces the upper bound of statistical dependence. It corresponds to the conventional autoencoder (AE) topology, but it is trained with the eigenfunctions of the CDR.


\vspace{3pt}

\noindent \textbf{Partial functional space codes (FMCA-PT).} 
The idea is to use image patches as the reference $\mathbf{u}$. Construct a mapping function $h(x, c)$ such that for a given image $x$ and a specified choice of coordinates $c$, an image patch $h(x, c)$ is created from this image. We consider all possible patches from the same image as the reference, thus producing a joint distribution $p(x, u) = p(\mathbf{x}=x)\frac{1}{|\mathcal{C}|}\int_{c\in {C}} \mathbb{1}(u = h(x, c))dc$, by marginalizing over all choices of patches. The aim of such a construction is to constrain the training to extract only the meaningful information in the functional space, by intentionally creating statistical independence between the original image r.p. $\mathbf{x}$ and the reference r.p. $\mathbf{u}$, but still retaining local statistical dependence at the patch level for training. \vspace{3pt}

\noindent \textbf{Factorial codes (FMCA-FC).} Factorial codes, by definition, minimize the sum of the bit entropies of the code components~(\citealt{barlow1989finding, barlow1989unsupervised}), i.e., the code components become statistically independent. In image processing, factorial codes can distinguish one image from all the other images in the dataset. With the proposed architecture, this means that the reference input $u$ should be as random as possible, i.e., a white noise r.p. Before training starts, each image in the dataset is paired with a noise realization of a certain length. Note that both factorial coding and maximal dependence coding have the same goal of capturing all statistical dependence in the realizations of the input random process by setting the CDR ${\rho}(x, x')$ to be positive if and only if $x = x'$, which captures the dissimilarities between all realizations due to normalization. \vspace{3pt}

Three sets of experiments are presented in the upcoming sections to illustrate these applications. Section~\ref{experiment_section} shows FMCA's ability to effectively approximate eigenfunctions and accurately measure statistical dependence across various distributions. Section~\ref{aggregation_section} presents results for Markov chain aggregation. Section~\ref{image_codeds_section} presents FMCA's application to representation learning in real-world image datasets.


\vspace{-5pt} \section{Experiments for measuring statistical dependence}
\label{experiment_section}

This section shows the robustness and capability MSD with FMCA in measuring statistical dependence across various distributions.
\vspace{5pt}

\noindent \textbf{Baselines.} Three baselines are compared: \vspace{-5pt}
\begin{itemize}[leftmargin=*]
\item KICA-KGV: Kernel Generalized Variance from Kernel Independent Component Analysis (\citealt{bach2002kernel}, \citealt{bach2005probabilistic}); \vspace{-5pt}
\item HSIC-NOCCO: Normalized Cross-Covariance Operator from Hilbert-Schmidt Independence Criterion (\citealt{gretton2007kernel, sriperumbudur2010hilbert});\vspace{-5pt}
\item MINE: Mutual Information Neural Estimator (\citealt{belghazi2018mutual}).
\end{itemize}\vspace{-5pt}

\noindent Both KICA-KGV and HSIC-NOCCO use Gaussian kernels to construct a Gram matrix and perform a spectral decomposition without optimization. The Gaussian kernel size is crucial and set to $0.1$. MINE optimizes a neural network using a variational cost, producing a scalar estimate of Shannon's mutual information. 

\vspace{5pt}

\noindent KICA-KGV, and HSIC-NOCCO theoretically approximate the same eigenvalues as our method but introduce bias due to kernels. TSD results for KICA-KGV are consistent with ours, whereas HSIC-NOCCO approximates a slightly different quantity using matrix trace, which sums the eigenvalues ($c(\lambda) = \lambda$ in Definition~\ref{definition_generalized_tsd} of Appendix~\ref{appendix_ncd_def}). Additional details can be found in Appendix~\ref{implementation_details}. \vspace{5pt}







\noindent \textbf{Reproducing theoretical values of MSD.} For simple distributions, the ground truth can be computed using the Nystr{\"o}m method~(\citealt{williams2006gaussian}) by estimating joint $pdf$ empirically, constructing the CDR, and performing eigendecomposition. For continuous distributions, this requires discretizing the sample space into intervals and estimating the $pdf$ of the resulting discretized distribution. The number of intervals to $1000$ for each dimension. This is inapplicable for complex, high-dimensional distributions.




\vspace{5pt}

\noindent\textbf{Learning eigenfunctions of Gaussian distributions (Figure~\ref{figure_3_gauss}).} When two r.p. follow a joint multivariate Gaussian distribution, their cross density will also be a multivariate Gaussian $pdf$. In this case, the MSD and eigenfunctions have a closed-form solution offered by Hermite polynomials (\citealt{williams2006gaussian}, \citealt{zhu1997gaussian}). The correlation coefficient between the pair is varied to compare the learned MSD, basis functions, and CDR, as shown in Figure~\ref{figure_3_gauss}.

The results reveal that: (1) neural network outputs accurately approximate Hermite polynomials, the well-known Gaussian bases; (2) the obtained MSD matches their true values, with the largest eigenvalue at constant $1$ and others bounded by $1$; (3) eigenvalues converge sequentially, resembling Gram-Schmidt sequential decomposition, where the $(i+1)$-th mode converges only after the $i$-th mode converges.\vspace{5pt}

\begin{figure}[t]
   \begin{subfigure}[b]{0.4\textwidth}  
    \centering
    \includegraphics[width=1\textwidth]{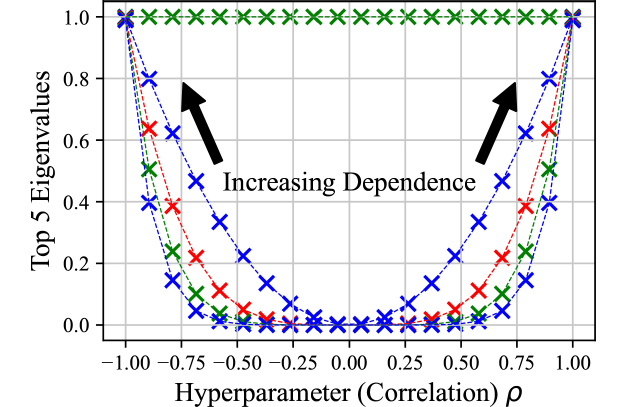}\vspace{-1pt}
    \caption{MSD w.r.t. correlations}
  \end{subfigure}
  \begin{subfigure}{0.5\textwidth}  
    \centering
    \includegraphics[width=\linewidth]{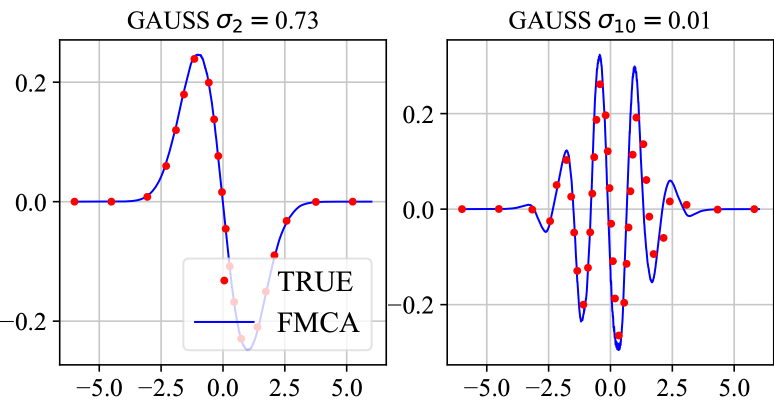}\vspace{-1pt}
    \caption{The second and fifth basis w/ ground truth}
  \end{subfigure}\hfill \vspace{3pt}\\ 
   \hfill \begin{subfigure}[b]{0.325\textwidth}  
    \centering
    \includegraphics[width=\linewidth]{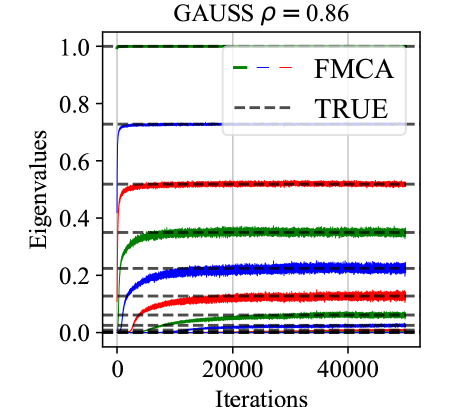}\vspace{-1pt}
    \caption{MSD learning curves}
  \end{subfigure}\hspace{-10pt}
    \begin{subfigure}[b]{0.615\textwidth}  
    \centering
    \includegraphics[width=\linewidth]{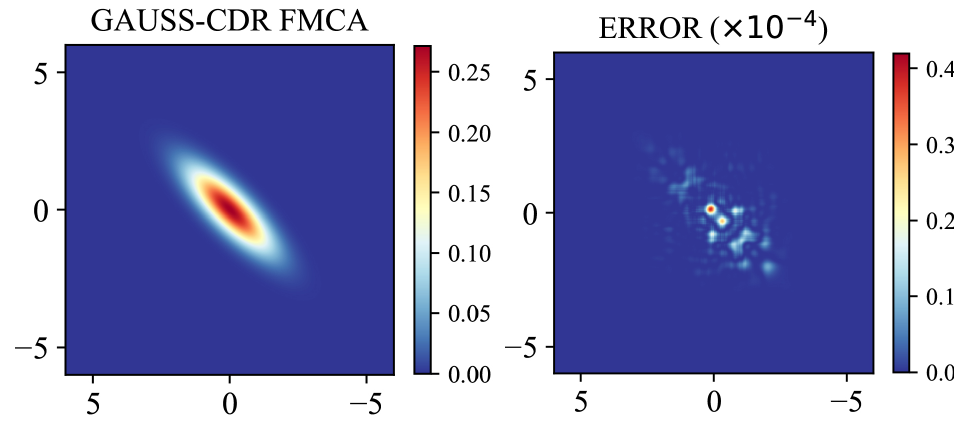}\vspace{-1pt}
        \caption{Learned CDR w/ approximation error}
  \end{subfigure}\hfill\vspace{-10pt} 
\caption{\small Results for joint Gaussian distributions with varying correlation coefficients. Visualization of bases, learning curves, and CDR for standardized Gaussians at correlation $\rho=0.86$. (a) shows the growth of learned MSD eigenvalues as correlations between two distributions increase. (b) compares approximated basis functions with theoretical ground truth, given by Hermite polynomials. (c) shows the smooth convergence of MSD eigenvalues, sequentially from smallest to largest. (d) shows the approximated CDR, which combines learned bases and spectrum, with a low error of $10^{-4}$. The visualization (d) is created by multiplying the CDR and marginals product for joint density.}\vspace{-5pt}
\label{figure_3_gauss}
\end{figure}

\noindent \textbf{Eigenfunctions of other distributions (Figure~\ref{spiral_plot}).} Further demonstrations include various 2D distributions: standardized Gaussians (GAUSS), checkerboard (CHECK), sinusoid (SW), and spiral (SPIRAL). The shapes and parameters for these distributions are detailed in Appendix~\ref{implementation_details}. Figure~\ref{spiral_plot} displays in sequence and in a circle arrangement their learned MSD eigenvalues, one per slice, starting from the positive horizontal axis. In red are the true values obtained with Nystr{\"o}m method. 

FMCA consistently outperforms baselines, with KICA and HSIC's generalized eigenproblems yielding similar results that tend to underestimate large eigenvalues and overestimate small ones. The bias is introduced by: (1) the single kernel nature of the RKHS and its hyperparameter, and (2) the limited number of samples due to the computational cost of constructing Gram matrices and computing their spectrum. Conversely, FMCA directly solves the eigenproblem via optimization, providing improved approximation characteristics.

As a multivariate statistical dependence measure, MSD captures structure in the CDR Hilbert space that scalar-valued measures like mutual information cannot achieve. This distinction is evident in Figure~\ref{spiral_plot}, showcasing varying spreads of eigenvalues with limiting cases like the checkerboard (just two eigenvalues) and the spiral, which exhibits significant dependence across all dimensions. Note that these datasets have similar global dependencies in terms of TSD. MSD's unique capability to quantify structure in the CDR space surpasses the limitations of scalar-valued measures such as mutual information. \vspace{5pt}

\begin{figure}[t]
  \centering
  \begin{subfigure}{.255\textwidth}
    \includegraphics[width=\linewidth]{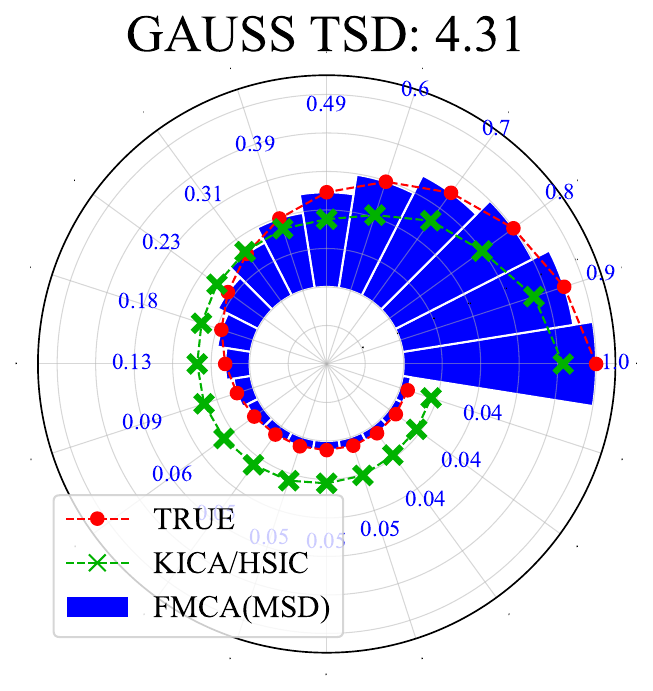}\vspace{-8pt}
    \phantomsubcaption
  \end{subfigure}\hspace{-10pt}
    \begin{subfigure}{.255\textwidth}
    \includegraphics[width=\linewidth]{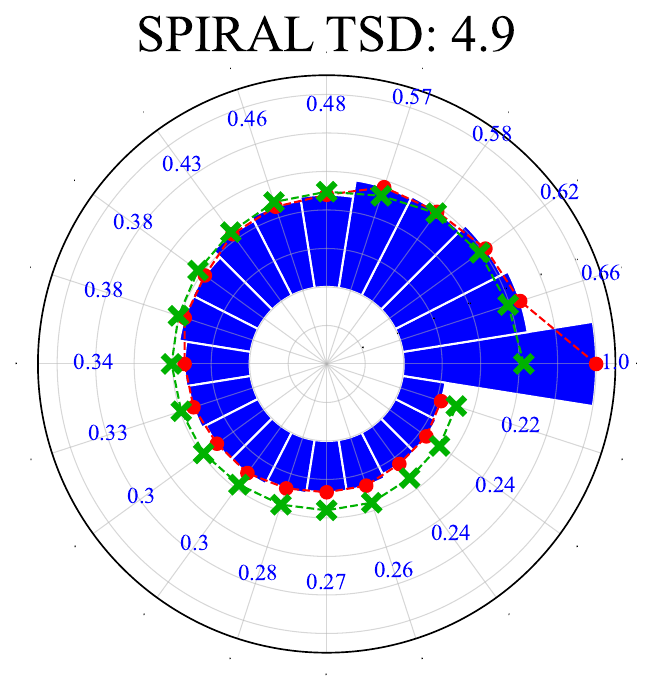}\vspace{-8pt}
    \phantomsubcaption
  \end{subfigure}\hspace{-10pt}
      \begin{subfigure}{.255\textwidth}
    \includegraphics[width=\linewidth]{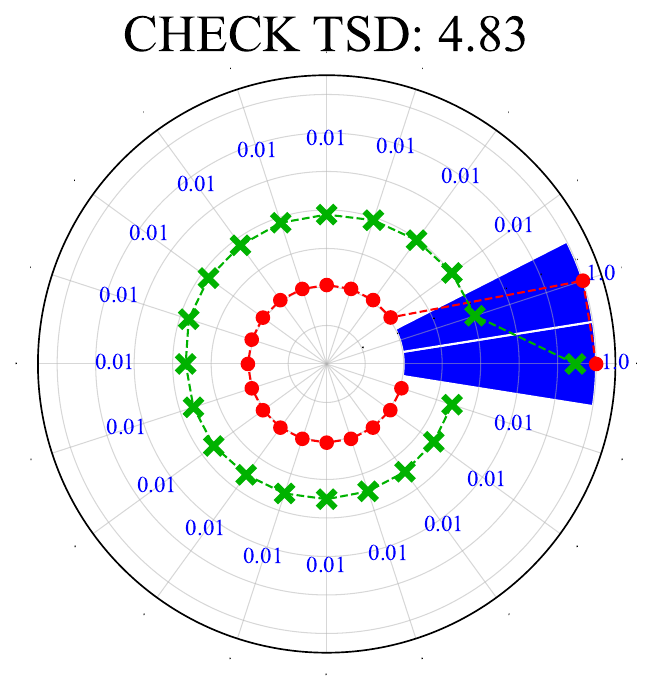}\vspace{-8pt}
    \phantomsubcaption
  \end{subfigure}\hspace{-10pt}
        \begin{subfigure}{.255\textwidth}
    \includegraphics[width=\linewidth]{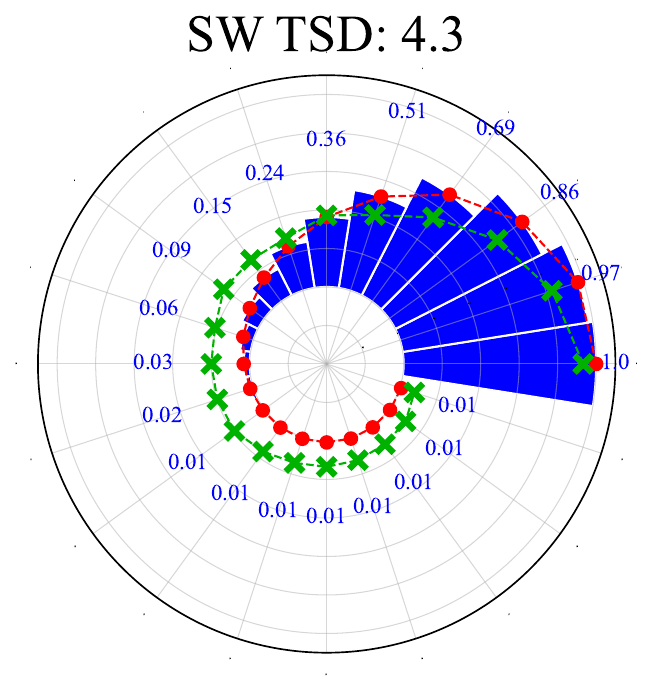}\vspace{-8pt}
    \phantomsubcaption
  \end{subfigure}
\caption{\small MSD of FMCA, compared to those produced by KICA/HSIC and the ground truth from the Nystr{\"o}m method. Despite varying TSD definitions, KICA and HSIC's decompositions are consistent and comparable to MSD. Due to Gaussian kernels in RKHS with kernel size as a hyperparameter, their approximations are always biased. FMCA, instead, produces the exact approximation to the ground truth. With similar TSD, MSD can greatly vary, which scalar measures like mutual information fail to capture. Eigenvalues of KICA/HSIC are generated by decomposing the the centered and normalized Gram matrix for the cross-covariance operator. } \vspace{-5pt} \label{spiral_plot}
\end{figure}

\noindent \textbf{Advantages over KICA, HSIC and MINE (Figure~\ref{figure_5_TSD}).} We compare the total statistical dependence learned by the different methods. FMCA and KICA estimate the same TSD quantity; meanwhile, HSIC estimates the sum of eigenvalues, and MINE estimates Shannon's mutual information. To plot on the same scale, we divide HSIC-NOCCO values by $2$.

In Figure~\ref{FA}, for Gaussian distributions with varying correlations, all baselines show an increase in TSD as the correlation increases. The curve variance is calculated by conducting $20$ trials of estimation, each time using $10^4$ randomly sampled samples. Among them, FMCA has the least variance. 

In Figure~\ref{FB}, we use a more challenging distribution: a 2D spiral with a white noise prior. As the noise variance increases, the statistical dependence decreases. FMCA accurately captures the impact of the noise, while KICA and HSIC indicate significantly poorer specificity and stability, particularly when the noise variance is small.

In Figure~\ref{FC}, we compare FMCA with MINE on Gaussian distributions when the correlation coefficient approaches $1$. MINE exhibits significant instability for high statistical dependence, which can be attributed to: (1) the logarithm in the cost function, introducing high variance during training; and (2) the approximated mutual information value of MINE being unbounded, causing the value to explode in one-to-one correspondence situations. FMCA avoids this issue as the eigenvalues are normalized, and the TSD is truncated by the dimension of the network outputs. \vspace{-5pt}

\begin{figure}[t]
  \centering
  \begin{subfigure}{.339\textwidth}
    \includegraphics[width=\linewidth]{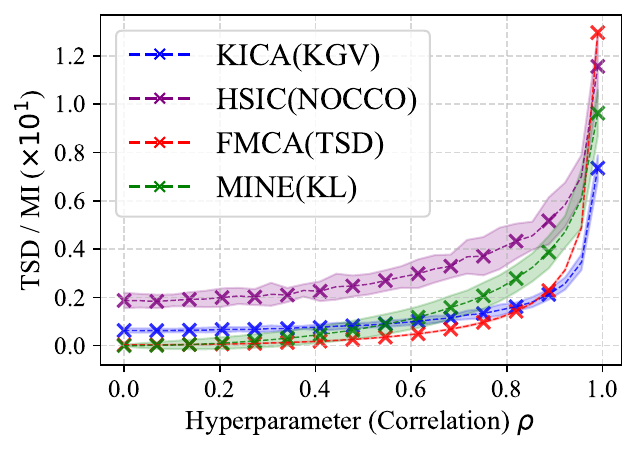}\vspace{-3pt}
    \caption{TSD/MI for GAUSS}
    \label{FA}
  \end{subfigure}\hspace{-7pt}
  \begin{subfigure}{.339\textwidth}
    \includegraphics[width=\linewidth]{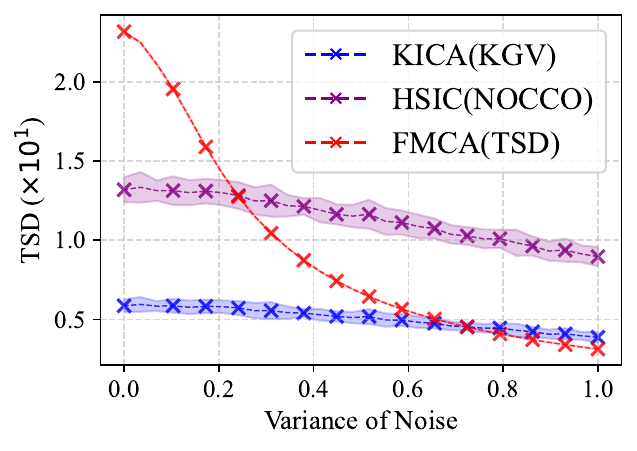}\vspace{-3pt}
    \caption{Inaccuracy of KICA/HSIC}
    \label{FB}
  \end{subfigure}\hspace{-7pt}
    \begin{subfigure}{.339\textwidth}
    \includegraphics[width=\linewidth]{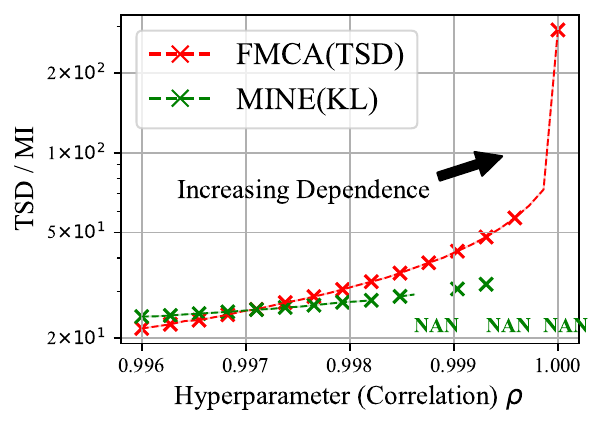}\vspace{-3pt}
    \caption{Instability of MINE}
    \label{FC}
  \end{subfigure}\vspace{-5pt}
\caption{\small Advantages of FMCA over KICA, HSIC, and MINE. (a) For a multivariate standardized Gaussian, all measures capture increasing statistical dependence as correlations increase, with less variance in FMCA. (b) In the SPIRAL dataset with varying noise levels, FMCA outperforms KICA and HSIC in terms of accuracy and stability. (c) MINE can diverge as the Gaussians become nearly a one-to-one correspondence; FMCA avoids this instability issue by truncating small eigenvalues. }
  \label{figure_5_TSD}\vspace{-10pt}
\end{figure}

\section{Experiments for Markov chain aggregation}
\label{aggregation_section}

In this section, we compare FMCA to the baselines for the Markov chain aggregation problem. 

\vspace{5pt}


\noindent\textbf{Dataset.} The New York City yellow cab trips dataset (~\citealt{nyc_taxi_data}) is utilized, a dataset widely used in existing aggregation literature~(\citealt{zhang2019spectral, duan2019state, roblitz2013fuzzy}). This dataset contains taxi trip records, including pick-up and drop-off locations in longitude and latitude on NYC's 2D map (the state space). The transition data from the entire year of $2016$, containing $10^7$ trip records, is analyzed. These transitions form a Markov chain with a transition probability, and the state aggregation task aims to partition the city map into subregions.


  \begin{figure}[t]
    \includegraphics[width=\linewidth]{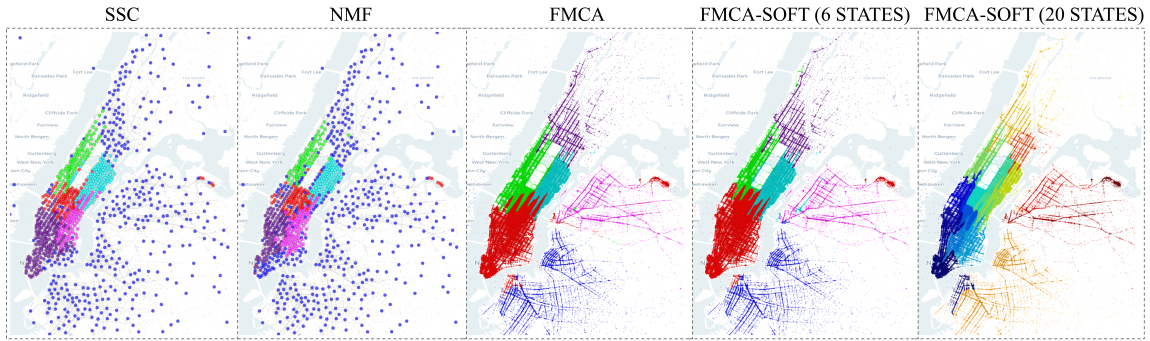}\vspace{-3pt}
    \caption{\small Taxi traffic data aggregation results. In comparison to baselines, FMCA neither requires the discretization of the state space nor the empirical estimation of transition probabilities. State assignment that appear as noise in SSC and NMF are identifiable streets in FMCA. Using softmax for soft-state distribution, FMCA-SOFT easily adapts to accommodate 20 states.}
    \label{FIGURE_AGGREGATION}
  \end{figure}%

\vspace{5pt}

\noindent\textbf{Baselines.} FMCA is compared with two baselines: Spectral State Compression (SSC) (\citealt{zhang2019spectral, duan2019state}) and Nonnegative Matrix Factorization (NMF) (\citealt{lee1999learning, lee2000algorithms}). SSC involves the following steps: (1) Discretizing the state space into grids. Since linear interpolation is inefficient, k-means is used to find $10^3$ centroids on the city map; (2) empirically estimating the transition probabilities of the $1000$-state chain; (3) performing SVD of the estimated $pdf$; (4) clustering using the top $r$ bases to obtain $r$ aggregated states. NMF follows a similar procedure as SSC but adds a penalty to SVD, ensuring that the bases are nonnegative. Further details of their implementations can be found in Appendix~\ref{implementation_details}.

\vspace{5pt}

\noindent\textbf{Aggregation procedures with FMCA.} As discussed in Section~\ref{aggregation_problem}, FMCA for deterministic aggregations uses top $r$ basis functions to obtain $r$ aggregated states with clustering, consistent with SSC. FMCA-SOFT uses a softmax function in the last layer of the neural network to learn a distribution, consistent with NMF, without requiring any clustering.


\vspace{5pt}
\noindent\textbf{Results of FMCA, compared to SSC and NMF (Figure~\ref{FIGURE_AGGREGATION}).} The first key simplification in FMCA is that it avoids partitioning the 2D continuous space, opting to use longitude and latitude as input instead of grid indices. The second difference is that FMCA eliminates the need for any empirical estimation of transition probability.

Figure~\ref{FIGURE_AGGREGATION} compares the aggregation results. FMCA effectively distinguishes the aggregated states without space discretization, identifying streets from what appears as noise in baseline methods. FMCA-SOFT can be effortlessly extended to accommodate more states by simply changing the output dimension. The aggregated outcomes from FMCA and FMCA-SOFT are highly consistent. In our experiments, we find FMCA-SOFT to be more stable and easy to train, while FMCA converges to the true eigenvalues more quickly. \vspace{5pt}

\noindent\textbf{Numerical comparisons (Figure~\ref{FIGURE_APPROX_ERROR}).} The approximation error of FMCA is compared with the error of KICA and HSIC, as all three methods are based on decomposition. A significant advantage of FMCA is that its approximation is not restricted by discretization resolution, as demonstrated in Figure~\ref{FIGURE_APPROX_ERROR}.\vspace{-5pt}
\begin{itemize}[leftmargin=*]
 \item{Low-res error (Figure~\ref{FIGURE7A}):} Initially, with a grid size of $10^3$, the CDR are estimated, and approximation errors ($L_2$ distance between approximation and the true CDR) are computed for all three methods. Due to the penalty constraint for non-negativity, NMF yields biased estimations with high errors. Theoretically, SSC attains minimal error in discretized space (rather than the original state space). FMCA's error stays close to the ground truth. \vspace{-5pt}
\item{High-res error (Figure~\ref{FIGURE7B}):} By changing the grid size to $10^4$ and recomputing the CDR as a $10^4\times10^4$ matrix while still using the SSC and NMF basis functions obtained from a grid size of $10^3$ for CDR approximation, FMCA demonstrates the lowest error among all methods.\vspace{-5pt}
\item{MSD (Figure~\ref{FIGURE7C}):} The MSD of FMCA is also computed and compared with eigenvalues acquired from the Nyström method using a grid size of $10^4$, showing high accuracy.\end{itemize}


  \begin{figure}[t]
  \vspace{-10pt}
  \begin{subfigure}{.36\textwidth}
    \includegraphics[width=\linewidth]{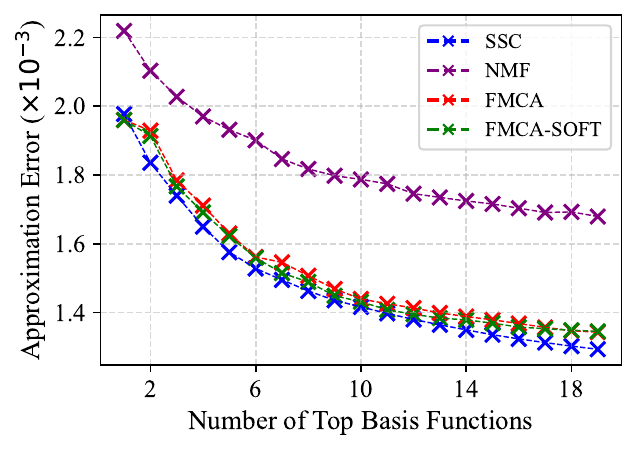}\vspace{-8pt}
    \caption{Low-res approximation error}
    \label{FIGURE7A}
  \end{subfigure}
  \begin{subfigure}{.36\textwidth}
  \vspace{-1pt}
    \includegraphics[width=\linewidth]{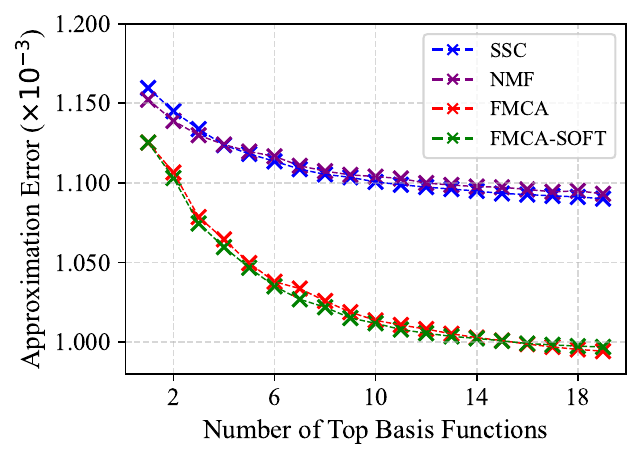}\vspace{-8pt}
    \caption{High-res approximation error}
    \label{FIGURE7B}
  \end{subfigure}\hfill
    \begin{subfigure}{.26\textwidth}
    \includegraphics[width=\linewidth]{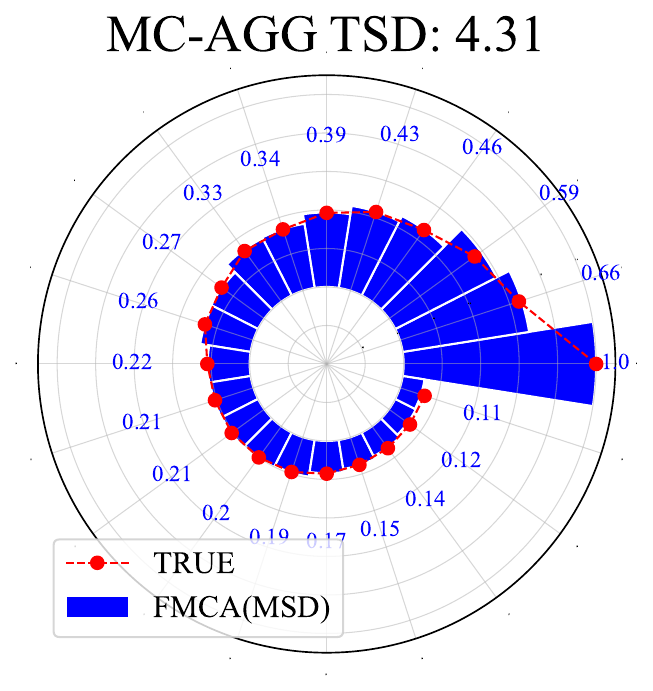}\vspace{-8pt}
    \caption{Learned MSD}
    \label{FIGURE7C}
  \end{subfigure}\vspace{-5pt}
  \caption{\small CDR approximation error comparison for FMCA, KICA, and HSIC at two map discretization resolutions: (a) a grid size of $10^3$, and (b) a grid size of $10^4$. The x-axis in the figures denotes the number of basis functions for approximating CDR. NMF and SSC are constrained to use bases obtained from low resolution. NMF always produces biased estimation. SSC has low error at low resolution, but struggle at high resolution. FMCA achieves the lowest error at high resolution, as it does not require map discretization. (c) displays FMCA's MSD in comparison to the Nystr{\"o}m method estimates.}
  \label{FIGURE_APPROX_ERROR}
  \end{figure}

\vspace{5pt}

\vspace{-20pt}
\section{Experiments for coding real-world images}
\label{image_codeds_section}

There are two traditional, contrasting approaches to learning image features: training a classifier, which minimizes classification error, and training an autoencoder (AE), which minimizes reconstruction error. These two tasks embody opposing learning principles—supervised and unsupervised. Classification emphasizes the generalization capability of codes, since they must reflect features in the functional space that represents classes. Reconstruction prioritizes the specificity of codes, as they must reconstruct the original images using a decoder.


The intrinsic trade-off between generalization ability and code specificity is well-known. For example, networks trained with AE do not achieve the highest classification accuracy, and cost penalties are often necessary to impose constraints that enhance class discriminability. FMCA's flexibility enables a broader range of possibilities to balance this trade-off between generalization capability and code specificity. As discussed in~\ref{imagerp}, we compared the following variations of FMCA:


\vspace{-5pt}
\begin{itemize}[leftmargin=*]
\item {FMCA-M:} Codes with maximal dependence, where $\mathbf{f}_\theta$ and $\mathbf{g}_\omega$ receive the same image to achieve the highest statistical dependence;
\vspace{-5pt}
\item {FMCA-PT:} Codes with intermediate dependence, where $\mathbf{f}_\theta$ receives a patch of an image, and $\mathbf{g}_\omega$ receives a different patch from the same image;
\vspace{-5pt}
\item {FMCA-C:} Class codes, where $\mathbf{f}_\theta$ receives an image and $\mathbf{g}_\omega$ receives a one-hot vector representing the image's label;\vspace{-5pt}
\item {FMCA-FC:} Factorial codes, where $\mathbf{f}_\theta$ takes an input image while random noise (sampled from a paired uniform distribution) is fed to $\mathbf{g}_\omega$. \end{itemize}\vspace{-5pt}
\noindent Let us elaborate on the significance of FMCA-FC. Codes are internal representations of the external world, so for autonomous learning, it is important to infer the state of the world from the activation of a given embedded code. Factorial codes have long been investigated for this purpose (\citealt{barlow1989finding}). For image recognition, the goal is to generate one single internal code activation for each image, which is much more specific than class assignment.





 \vspace{5pt}

\noindent \textbf{Network structure.} See Appendix~\ref{implementation_details} for information.

\vspace{5pt}

\noindent\textbf{Generalization criterion.} The test set prediction accuracy (\textbf{Acc}) is used for code generalization. FMCA-C directly produces the classification in the test set. For FMCA-M/PT/FC, the two FMCA networks that create the projection space are initially trained, and their parameters are fixed. Next, an MLP classifier is trained to project image features onto labels. The test images are then input to FMCA's primary network and the MLP classifier to obtain predictions. For the AE, the MLP classifier is trained to project latent features from a trained encoder to labels, and the encoder-classifier combination is applied to test images. \vspace{5pt}

\noindent\textbf{Code specificity criterion.} FMCA approximates the CDR $\rho(x, x')$ between image pairs in the training set. Different FMCA architectures produce varying CDRs. Code specificity requires the CDR between a sample $x_n$ and itself, $\rho(x_n, x')$ when $x' = x$, to be dominantly larger than the CDR between this sample and any other samples, $\rho(x_n, x')$ when $x' \neq x$. If the CDR value at $x' = x_n$ dominates all other samples, the code is specific. To evaluate specificity, we introduce two descriptors, \textbf{S-T} and \textbf{S-A}.

For a training set of size $N$, an $N\times N$ CDR matrix, $\rho(x, x')$, is obtained, with focus on the diagonal elements. Each row is normalized using the softmax function: $\overline{\rho}(x, x') = \frac{e^{\rho(x, x')}}{\sum_{n=1}^N e^{\rho(x, x_n)}}$. The average diagonal value, $\textbf{S-A}$, calculated as $\frac{1}{N} \sum_{n=1}^N\overline{\rho}(x_n, x_n)$, represents the dominance of diagonal elements. Due to possible large training set size and numerous diagonal values close to $1$, the diagonal elements $\overline{\rho}(x_n, x_n)$ are ranked from smallest to largest, considering only the $1000$ smallest values to compute the average, denoted as $\textbf{S-T}$. This descriptor $\textbf{S-T}$ enhances distinction. Both $\textbf{S-A}$ and $\textbf{S-T}$ values range from $0$ to $1$, with higher values indicating more specific codes. AE's CDR is generated by $\mathbf{f}^\intercal_\theta(x) \mathbf{R}_F^{-\frac{1}{2}} \mathbf{f}_\theta(x')$ with the encoder $\mathbf{f}_\theta$ and the ACF $\mathbf{R}_F$.

\vspace{5pt}

\noindent \textbf{Reconstruction criterion.} 
The reconstruction error, \textbf{R-Err}, is also evaluated to compare FMCA's performance with AE. A VGG-16 decoder is trained with the same architecture as AE's decoder to project the embedded space of a pre-trained FMCA network back to the image space, minimizing the reconstruction error. The error \textbf{R-Err} is reported on the training set. \vspace{5pt}

\noindent \textbf{Datasets.}
Three datasets with increasing difficulty are used to test the image codes obtained: MNIST is the simplest; CIFAR 10 has more complex image structures; CIFAR 100 has both a higher number of classes and more complex image structures. The results presented in Table~\ref{table_compare} provide a comparison of the different algorithms and will be discussed in the following sections.
\vspace{5pt}


\begin{table}[h]
\centering
\caption{Comparison of Generalization Capability and Code Specificity. FMCA-PT shows the highest classification accuracy for unsupervised learning. FMCA-FC shows the best code specificity and the second-best reconstruction error, only surpassed by AE, which minimizes errors as its objective.}
\label{tab:comparison}
\resizebox{0.9\textwidth}{!}{%
\begin{tabular}{|c|l|c|c|c|c|c|}
\hline
\rowcolor{gray!25}
\textbf{Datasets} & \textbf{Methods} & \textbf{Generalization} & \multicolumn{2}{c|}{\textbf{Code Specificity}} &\textbf{R-Err$\downarrow$}& \textbf{TSD}\\
\cline{4-5}
\rowcolor{gray!25}
 & & \textbf{Class Acc$\uparrow$} & \textbf{S-A$\uparrow$} & \textbf{S-T$\uparrow$} & & \\
\hline
\multirow{4}{*}{MNIST} & AE  & 0.983 &  \textbf{0.998} & \textbf{0.360} & \textbf{0.0006} & -\\
 & FMCA-PT & \textcolor{darkgray}{\textbf{0.987}}  & 0.851 & 0.004  & 0.073 & 23.91\\
 & FMCA-FC & 0.208 & \textbf{1.0} & \textbf{1.0} & \textbf{0.002} & 211.15\\
 & FMCA-M & 0.955  & 0.960 & 0.0039 & 0.006 & 297.56 \\
  & FMCA-C & \textbf{0.990}  & 0.0007 & 0.0001 & 0.04 & 12.67\\
\hline
\multirow{4}{*}{CIFAR10} & AE & 0.479 &  \textbf{0.993} & 0.004 & \textbf{0.005} & - \\
 & FMCA-PT & \textcolor{darkgray}{\textbf{0.876}} &  0.954 & 0.006 & 0.038 & 18.72\\
 & FMCA-FC & 0.139 & \textbf{1.0} & \textbf{1.0} & \textbf{0.009} & 213.56\\
 & FMCA-M & 0.351 & 0.990 & \textbf{0.368} & 0.019 & 299.38\\
  & FMCA-C & \textbf{0.902} & 0.007 & 0.0001 &0.058 & 25.68\\
\hline
\multirow{4}{*}{CIFAR100} & AE & 0.228 & \textbf{0.995} & \textbf{0.117} & \textbf{0.005} & - \\
 & FMCA-PT & \textcolor{darkgray}{\textbf{0.675}} & 0.934 &0.003 & 0.038 &18.11 \\
 & FMCA-FC & 0.093 & \textbf{1.0} &  \textbf{0.986} & \textbf{0.008} & 208.63\\
 & FMCA-M & 0.142 & 0.961 &0.006 & 0.018 &  299.27\\
  & FMCA-C & \textbf{0.752} &  0.034 & 0.0001 & 0.052 & 35.10\\
\hline
\end{tabular}}
\label{table_compare}
\end{table}

\noindent \textbf{FMCA-PT shows remarkable generalization capability.}
While FMCA-C, a supervised method, is expected to achieve the highest test set classification accuracy, it is surprising that the unsupervised method FMCA-PT achieves performance close to the supervised approach. The generalization capability of FMCA-M is lower, as it matches a complete image to itself, followed by FMCA-FC, which lacks functional space and class information in the codes, making it unable to generalize to different images of the same class. Remarkably, AE's performance falls short of FMCA-PT, particularly for more complex datasets with a higher number of classes. This indicates that FMCA-PT offers a better balance than AE between code specificity and generalization capability. \vspace{5pt}



\noindent \textbf{FMCA-FC has the best code specificity.} Regarding code specificity, factorial codes consistently yield the highest scores, achieving values near $1$ for both \textbf{S-A} and \textbf{S-T}. FMCA-C performs poorly in this aspect, ranking as the least effective. Surprisingly, AE consistently ranks second in \textbf{S-A} and almost always second in \textbf{S-T}. This finding suggests that specificity is related to how FMCA-C utilizes statistical dependence to create basis functions. Under \textbf{S-T}, the performance of FMCA-PT and FMCA-M significantly drops, revealing the impact of within-class similarity on the learned codes.  \vspace{5pt}


\noindent \textbf{FMCA does not compromise on reconstruction errors (Figure~\ref{BASIS_RECONSTRUCT_FIGURE}).} In terms of reconstruction error, AE outperforms other methods, followed by FMCA-FC. Lower errors in FMCA-FC can be contributed to the high sparsity of the codes, allowing each code to effectively reconstruct its corresponding image in the decoder. FMCA-M outperforms FMCA-PT, while FMCA-C performs the worst in most datasets. As illustrated in Figure~\ref{BASIS_RECONSTRUCT_FIGURE}, Despite FMCA-FC's poor generalization performance, it reconstructs images reasonably well using only the top $10$ eigenfunctions, indicating that FMCA-FC codes effectively preserve perceptual distance within the image space. \vspace{5pt}


\begin{figure}[h]
  \centering
    \begin{subfigure}{.8\textwidth}
    \includegraphics[width=\linewidth]{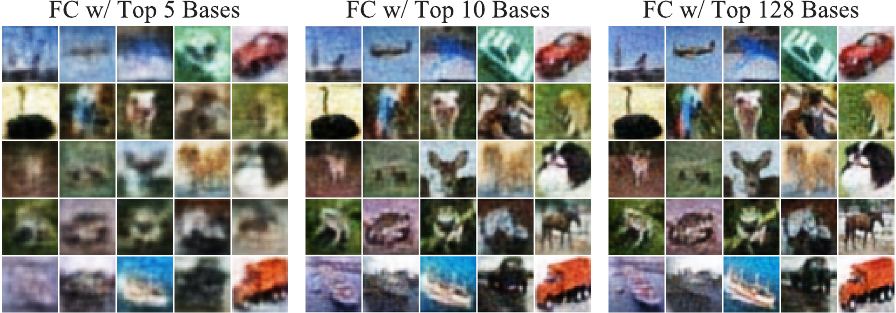}
    \phantomsubcaption
  \end{subfigure}
\caption{\small Reconstruction of images with FMCA-FC, utilizing varying numbers of basis functions to train the decoder. Although generalization performance is subpar, FMCA-FC successfully reconstructs training set images using just the top $10$ eigenfunctions, demonstrating effective preservation of perceptual distance within the image space.}
  \label{BASIS_RECONSTRUCT_FIGURE}
\end{figure}


\noindent \textbf{CDR Visualization (Figure~\ref{CDR_comparison_fmca}).} To better understand the trade-off between code specificity and generalization, the differences in CDR for each method are illustrated in Figure~\ref{CDR_comparison_fmca}. FMCA-C is a supervised coding method with class labels as reference, resulting in a positive CDR only when two samples belong to the same class. FMCA-PT's CDR incorporates functional space class structure beyond classification by learning from cross-patch information. FMCA-M's CDR contains class information but emphasizes matrix diagonal elements through self-matching, increasing code specificity. FMCA-FC's CDR is similar to FMCA-M's, as its CDR is maximal only when two images are identical (diagonal entries), but this is achieved using unsupervised learning. \vspace{5pt}

\begin{figure}[t]
\vspace{-9pt}
  \centering
  \begin{subfigure}{.201\textwidth}
    \includegraphics[width=\linewidth]{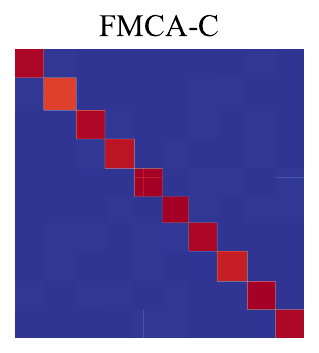}
    \phantomsubcaption
  \end{subfigure}\hspace{-7pt}
  \begin{subfigure}{.201\textwidth}
    \includegraphics[width=\linewidth]{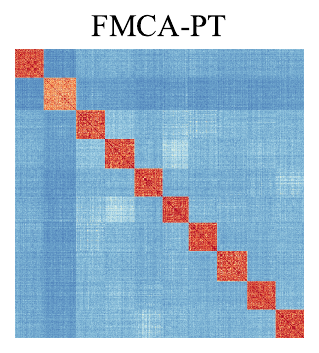}
    \phantomsubcaption
  \end{subfigure}\hspace{-7pt}
    \begin{subfigure}{.201\textwidth}
    \includegraphics[width=\linewidth]{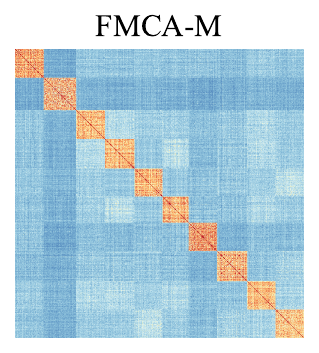}
    \phantomsubcaption
  \end{subfigure}\hspace{-7pt}
      \begin{subfigure}{.201\textwidth}
    \includegraphics[width=\linewidth]{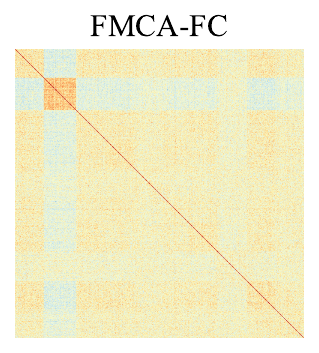}
    \phantomsubcaption
  \end{subfigure}\hspace{-7pt}
    \begin{subfigure}{.201\textwidth}
    \includegraphics[width=\linewidth]{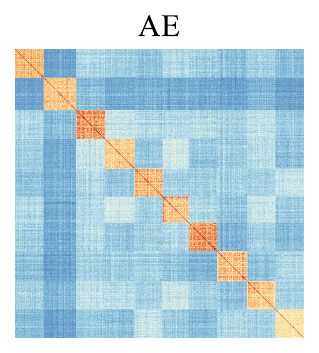}
    \phantomsubcaption
  \end{subfigure}\vspace{-5pt}
\caption{\small Comparison of CDR for coding methods, illustrating the trade-off between code specificity and generalization. For visualization purposes, we set the maximum value of CDR at $20$ and the minimum at $0$. From left to right, the methods are FMCA-C (supervised coding), FMCA-PT, FMCA-M, and FMCA-FC, respectively. Supervised coding is positive only when two samples belong to the same class. FMCA-PT learns cross-patch information. FMCA-M's diagonal matrix elements have higher values due to matching each image with itself. FMCA-FC has the maximal CDR when two images are identical, using unsupervised learning. FMCA-FC's CDR specificity is significantly better than AE's.}  \label{CDR_comparison_fmca} \vspace{-10pt}
\end{figure}

In conclusion, these image coding tests demonstrate the flexibility of the projection space in quantifying multivariate statistical dependence between two random processes. Beyond conventional machine learning applications, such as classification and dimensionality reduction, there are numerous other possibilities for learning internal representations that could prove critical for innovative applications, such as security. In an extreme case, a learning system can be trained to be far more specific than a classifier, capable of distinguishing one exemplar of the training set from all others.

\vspace{5pt}

\vspace{-10pt}
\section{Conclusion and Discussion}


This paper introduces a novel form of the Normalized Cross Density (NCD) function to quantify statistical dependence between two nonstationary random processes, $\mathbf{x}$ and $\mathbf{u}$. The method involves two NCDs, one for each process. The NCD numerator consists of the joint $pdf$, where $\mathbf{x'}'$ is an internally generated variable using a multidimensional extension of the Alternating Conditional Expectation (ACE), a bidirectional recursion algorithm. During the recursion for $\mathbf{x}$, the variable $\mathbf{u}$ is marginalized, producing an internal variable $\mathbf{x}'$ to estimate the statistical dependence between $\mathbf{x}$ and $\mathbf{u}$, and likewise for $\mathbf{u}$.


The NCD is a positive-definite function with two arguments, allowing decomposition into its eigenfunctions, which generates a feature space similar to principal component analysis (PCA) for correlation functions, except that NCD decomposes the $pdf$ structure. The NCD eigenspectrum serves as the first multivariate statistical dependence (MSD) measure between two r.p., extending both maximal correlation and conventional mutual information, which are scalar descriptors. The NCD offers significant improvements over Wiener theory, as conditioning and marginalization through linear operators in $pdf$ space, effectively capture nonlinear relationships in the data space. While this paper primarily focuses on the theory and multivariate characteristics of the proposed statistical dependence measure, the framework presents numerous new opportunities and holds significant potential for machine and information-theoretical learning applications.

\vspace{-5pt}
\section{Acknowledgment}
\vspace{-7pt}
The authors gratefully acknowledge partial support  by the Office of Naval Research (N00014-21-1-2324), (N00014-21-1-2295) and (N00014-21-1-2345). The authors also gratefully acknowledges Dr.Yuheng Bu, Assistant Professor with the Department of Electrical \& Computer Engineering (ECE) at the University of Florida, for constructive discussions and suggestions to improve the paper readability.
\vspace{-10pt}

\appendix

\section{Proof of the semi-definitenss of NCD}
\label{appendix_ncd_def}

Here, we prove Lemma~\ref{lemma_2_kernel} that two NCDs are positive-definite functions. 

\vspace{5pt}
\begin{proof}
We show the symmetry of $K(x, x'):== \frac{p(x, x')}{p^{\frac{1}{2}}(x) p^{\frac{1}{2}}(x')}$ starting with its denominator, by proving that the marginal distributions of $\mathbf{x}$ and $\mathbf{x}'$ are identical:
\begin{equation}
\begin{aligned}
p(\mathbf{x}' = x') = \int_{\mathcal{U}} p(\mathbf{x}' = x' |\mathbf{u}=u) p(\mathbf{u}=u)du = \int_{\mathcal{U}} p(\mathbf{x} = x' |\mathbf{u}=u) p(\mathbf{u}=u)du = p(\mathbf{x} = x').
\end{aligned}
\end{equation}
The symmetry of the numerator, i.e., the symmetry of their joint $pdf$, follows:
\begin{equation}
\begin{aligned}
p(x, x') = p(\mathbf{x} = x, \mathbf{x}' = x') &= \int_\mathcal{U}p(\mathbf{x}' =x'|\mathbf{u} = u)p(\mathbf{u} = u|\mathbf{x} = x)p(\mathbf{x} = x) du\\
& = \int_\mathcal{U}p(\mathbf{x} =x'|\mathbf{u} = u)p(\mathbf{u} = u|\mathbf{x}' = x)p(\mathbf{x}' = x) du \\
&= p(\mathbf{x} = x', \mathbf{x}' = x) = p(x', x).
\end{aligned}
\end{equation}

Therefore, $K(x, x')$ is symmetrical. We proceed to show that $K(x, x')$ is positive-definite. Given a sequence of possible realizations $\{x_i\}_{i=1}^N$ and a set of scalars $\{\alpha_i\}_{i=1}^N$, we have
\begin{equation}
\begin{aligned}
\sum_{i=1}^N \sum_{j=1}^N \alpha_i \alpha_j K(x_i, x_j) &= \sum_{i=1}^N \sum_{j=1}^N \alpha_i \alpha_j \frac{\int_{\mathcal{U}} p(x_j|u) p(u|x_i) p(x_i) du}{p^{\frac{1}{2}}(x_i) p^{\frac{1}{2}}(x_j)} \\
&= \int_{\mathcal{U}} \sum_{i=1}^N \sum_{j=1}^N \frac{  \alpha_i \alpha_j  p(x_i, u) p(x_j, u)}{p^{\frac{1}{2}}(x_i) p^{\frac{1}{2}}(x_j) p(u)} du \\
&= \int_{\mathcal{U}} \left( \sum_{i=1}^N \frac{\alpha_i p(x_i, u)}{p^{\frac{1}{2}}(x_i)}  \right)^2 \frac{1}{p(u)} du \geq 0.
\end{aligned}
\end{equation}
Therefore, $K(x, x')$ is positive-definite. Since it is both symmetrical and positive-definite, the function is a positive-definite kernel function. Similarly, $K(u, u') = \frac{p(u, u')}{p^{\frac{1}{2}}(u) p^{\frac{1}{2}}(u')}$ is also a positive-definite kernel. Furthermore, each kernel is associated with a linear operator and a corresponding Hilbert space.
\end{proof}

\section{Generalized definition of TSD}
\label{TSD_DEFINITION}
Here we present the generalized definition of the Total Statistical Dependence (TSD). 
\begin{definition}(Generalized definition of TSD) Given the NCD's eigenspectrum $\{ \lambda_i \}_{i=1}^\infty$ and any monotonically increasing convex function $c(\cdot):[0, 1]\rightarrow [0, \infty)$ that satisfies $c(0) = 0$, we define the generalized form of TSD along with its truncated TSD as:
\begin{equation}
\begin{gathered}
T(\lambda_1, \lambda_2, \cdots; c) = \sum_{i=1}^\infty c(\lambda_i), \;\; T_K(\lambda_1, \lambda_2, \cdots; c) = \sum_{i=1}^K c(\lambda_i).
\end{gathered} 
\end{equation} 
If an NCD has countably infinite eigenvalues that are strictly positive, in which case the TSD is unbounded, the truncated TSD uses the dominant $K$ eigenvalues and ensures that the measure remains bounded. Immediately, function $c$ can be chosen as $c(\lambda) = \lambda$ or polynomials of $\lambda$. In Definition~\ref{definition_eigenspectrum} for this paper, function $c$ is chosen as a special function $c(\lambda) = \frac{1}{2}\log\frac{1}{1-\lambda}$. 
\label{definition_generalized_tsd}
\end{definition}
\section{Proof of the variational eigenproblem}
\label{CDR_appendix}

Here we prove Lemma~\ref{theorem1}, which transforms the NCD eigenproblem into a variational eigenproblem of the CDR involving only empirical quantities. The proof follows a standard approach from GP. \vspace{5pt}
\begin{proof}
The set $\{\widehat{\phi_i}(x)\}_{i=1}^\infty$ can be constructed as follows. For each eigenfunction $\phi_i$, there exists a function $\widehat{\phi_i}(x)$ such that $\phi_i(x) = \widehat{\phi_i}(x) p^{\frac{1}{2}}(x)$ a.e.$p(x)$. Assuming $p(x)$ has bounded support, we can construct $\widehat{\phi_i}(x)$ such that $\widehat{\phi_i}(x) = \phi_i(x) p^{-\frac{1}{2}}(x)$ on the support of $p(x)$, and $\widehat{\phi_i}(x) = 0$ otherwise. To verify that the constructed $\{\widehat{\phi_i}(x)\}_{i=1}^\infty$ satisfies the orthonormal condition~\eqref{Orthonormal}, 
we can prove the following relationship
\begin{equation}
\begin{gathered}
\int \phi_i(x) \phi_j(x)  dx = \int \widehat{\phi_i}(x) p^{\frac{1}{2}}(x)  \widehat{\phi_j}(x) p^{\frac{1}{2}}(x)  dx = \int \widehat{\phi_i}(x) \widehat{\phi_j}(x) p(x) dx.
\end{gathered}
\label{proof_orthonormal}
\end{equation}
Therefore if the functions $\{{\phi_i}(x)\}_{i=1}^\infty$ are orthonormal w.r.t. the Lebesgue measure, then the functions $\{\widehat{\phi_i}(x)\}_{i=1}^\infty$ are orthonormal w.r.t. the probability measure. To show that $\{\widehat{\phi_i}(x)\}_{i=1}^\infty$ satisfy the equilibrium condition~\eqref{conditional_mean}, we substitute $\phi_i$ with $\widehat{\phi_i}(x)p^{\frac{1}{2}}(x)$ for $T\phi_i$, which yields
\begin{equation}
\resizebox{1\linewidth}{!}{
$\begin{gathered}
T\phi_i(x) = \int_{\mathcal{X}}  \frac{p(x, x')}{p^{\frac{1}{2}}(x) p^{\frac{1}{2}}(x')} \phi_i(x') dx' =  \int_{\mathcal{X}} \frac{p(x, x')}{p^{\frac{1}{2}}(x) p^{\frac{1}{2}}(x')} \widehat{\phi_i}(x') p^{\frac{1}{2}}(x') dx' =  \int_{\mathcal{X}} \frac{p(x, x')}{p^{\frac{1}{2}}(x)} \widehat{\phi_i}(x') dx'.
\end{gathered}$}
\end{equation}
Using the definition of the eigenfunction $T\phi_i(x) = {\phi_i}(x) = \widehat{\phi_i}(x) p^{\frac{1}{2}}(x)$, we have 
\begin{equation}
\begin{gathered}
\int_{\mathcal{X}} \frac{p(x, x')}{p^{\frac{1}{2}}(x)} \widehat{\phi_i}(x') = \lambda_i \widehat{\phi_i}(x) p^{\frac{1}{2}}(x) \Rightarrow \int_{\mathcal{X}} p(x'|x) \widehat{\phi_i} (x') = \lambda_i \widehat{\phi_i} (x)\;\; a.e.p(x),
\end{gathered}
\label{RHSOF3}
\end{equation}
which proves~\eqref{conditional_mean}. Finally, to show Equation~\eqref{final_equation}, we use the property of the kernel function $K(x, x') = \sum_{i=1}^\infty \lambda_i \phi_i(x) \phi_i(x')$. Substituting $\phi_i(x)$ by $\widehat{\phi_i}(x) p^{\frac{1}{2}}(x)$, we obtain Equation~\eqref{final_equation}. Thus given any two realizations $x, x' \in \mathcal{X}$, we can construct the CDR $\rho (x, x') = \frac{p(x, x')}{p(x) p(x')}$ through $\{ \lambda_i \}_{i=1}^\infty$ and $\{\widehat{\phi_i}(x)\}_{i=1}^\infty$.
\end{proof}

\vspace{-20pt}

\section{The maximal correlation is the second largest eigenvalue of NCD}
\label{appendix_mca}
Here, we prove Lemma~\ref{lemma_10}, which demonstrates that the maximal correlation is the second-largest eigenvalue of NCD, and its solutions are the second eigenfunctions of the two CDR. \vspace{5pt}
\begin{proof}
The proof follows by applying the Cauchy-Schwarz inequality. Suppose the functions $\overline{f}(x)$ and $\overline{g}(u)$ achieve maximal correlation and are normalized to have a mean of $0$ and a standard deviation of $1$. By the shift-invariance and the scale-invariance of the maximal correlation, the functions $\overline{f}(x)$ and $\overline{g}(u)$ defined in Lemma~\ref{lemma_10} are also the optimal solution of the maximal correlation. Observe that $\mathbf{cc}^*$ can be written as
\begin{equation}
\resizebox{1\linewidth}{!}{
$\begin{aligned}
\textbf{cc}^* = \mathbb{E}[\overline{f}(\mathbf{x})\overline{g}(\mathbf{u})] = \mathbb{E}_\mathbf{x} \big[ \overline{f}(\mathbf{x}) \, \mathbb{E}_\mathbf{u}[\overline{g}(\mathbf{u})|\mathbf{x}] \big] \leq \textbf{std}(\overline{f}(\mathbf{x})) \; \textbf{std}(\mathbb{E}_\mathbf{u}[\overline{g}(\mathbf{u})|\mathbf{x}]) = \textbf{std}(\mathbb{E}_\mathbf{u}[\overline{g}(\mathbf{u})|\mathbf{x}]). 
\label{CS}
\end{aligned}$}
\end{equation}
This implies that $\textbf{cc}^*$ is upper bounded by the standard deviation of the conditional mean $\mathbb{E}_\mathbf{u}[\overline{g}(\mathbf{u})|\mathbf{x}]$. 

First, if the supremum $\textbf{cc}^*$ is attained, the inequality~\eqref{CS} must become equality. For any fixed function $\overline{g}(\cdot)$, the LHS of~\eqref{CS} has the largest value only when $\overline{f}(x) = \frac{\mathbb{E}_\mathbf{u}[\overline{g}(\mathbf{u})|x]}{\textbf{std}(\mathbb{E}_\mathbf{u}[\overline{g}(\mathbf{u})|\mathbf{x}])}$. By applying this form of $\overline{f}(x)$ to~\eqref{CS}, we obtain $\textbf{cc}^* = \textbf{std}(\mathbb{E}_\mathbf{u}[\overline{g}(\mathbf{u})|\mathbf{x}])$. And it follows that $\mathbb{E}_\mathbf{u}[\overline{g}(\mathbf{u})|x] = \textbf{cc}^*  \overline{f}(x)~a.e.p(x)$. Symmetrically, given any fixed $\overline{f}(\mathbf{x})$, we have $\mathbb{E}_\mathbf{x}[\overline{f}(\mathbf{x})|u] = \textbf{cc}^* \overline{g}(u)~a.e.p(u)$. Combining the two equalities and applying the bidirectional recursion in Definition~\ref{definition_density}, we have
\begin{equation}
\begin{aligned}
\mathbb{E}_{\mathbf{x'}}[\overline{f}(\mathbf{x'})|x] = (\textbf{cc}^*)^2 \, \overline{f}(x)~a.e.p(x),\;\;\mathbb{E}_{\mathbf{u'}}[\overline{g}(\mathbf{u'})|u] = (\textbf{cc}^*)^2 \, \overline{g}(u)~a.e.p(u).
\end{aligned}
\label{bidirectional_optimization}
\end{equation}
Equation~\eqref{bidirectional_optimization} implies that the equilibrium condition~\eqref{conditional_mean} is satisfied for $\overline{f}$ and $\overline{g}$. Moreover, since $\overline{f}(\mathbf{x})$ and $\overline{g}(\mathbf{u})$ are normalized, their standard deviations are both $1$. Therefore, the orthonormal condition~\eqref{Orthonormal} is also satisfied. Thus $\overline{f} \in \{\widehat{\phi_i}\}_{i=1}^\infty$ and $\overline{g} \in \{\widehat{\psi_i}\}_{i=1}^\infty$. From Equation~\ref{bidirectional_optimization}, we know that the maximal correlation $(\textbf{cc}^*)^2 \in \{\lambda_i\}_{i=1}^\infty$. 

Since we obtained $\textbf{cc}^*$ by taking the supremum, the only part left is to show that $\textbf{cc}^*$ avoids the trivial eigenvalue $\lambda_1 = 1$ and finds the second largest eigenvalue $\lambda_2$. In Lemma~\ref{dependence}, we proven that the eigenfunction of the largest eigenvalue has the form $\phi_1(x) = p^{\frac{1}{2}}(x)$, and correspondingly, $\widehat{\phi_i}(x) = 1~a.e.p(x)$ is a constant function. This implies that $\mathbf{std}(\widehat{\phi_1}(x)) = \sqrt{\mathbb{E}[\widehat{\phi_1}^2(x)] - \mathbb{E}^2[\widehat{\phi_1}(x)]} = 0$. By defining the cost function~\eqref{form1}, the standard deviation of the neural network outputs cannot be zero, as $\mathbf{std}(\overline{f}(\mathbf{x})) = 1\neq 0$. Therefore $\overline{f}(\mathbf{x})$ avoids $\widehat{\phi_1}(x)$ and finds the function corresponding to the second largest eigenvalue:
\begin{equation}
\begin{gathered}
(\textbf{cc}^*)^2 = \lambda_2, \;\; \overline{f}(x) = \widehat{\phi_2}(x)\; a.e.p(x), \;\; \overline{g}(u) = \widehat{\psi_2}(u)\; a.e.p(u).
\end{gathered}
\label{one_lambda}
\end{equation}
\end{proof}

\section{Main theorem of FMCA's optimality condition}
\label{main_proof}
Here, we prove the main theorem stating that the Wiener solutions $\mathbf{m}_F = \overline{\mathbf{P}} \overline{\mathbf{g}}$ and $\mathbf{m}_G = \overline{\mathbf{P}}^\intercal \overline{\mathbf{f}}$ is equivalent to the minimization problem of the cost function. The proof strategy involves investigating the first-order derivative. Our proof will require the following assumptions. \vspace{-5pt}
\begin{itemize}[leftmargin=*]
    \item Two networks are parameterized models, $\mathbf{f}_\theta$ and $\mathbf{g}_\omega$, assumed to be universal; \vspace{-5pt}
    \item All ACFs are assumed to be positive-definite and thus invertible, which can be ensured by adding a small regularization matrix $\epsilon \mathbf{I}$ to them. For simplicity, we henceforth omit explicitly mentioning their invertibility when taking the inverse; \vspace{-5pt}
    \item For simplicity, the model $\mathbf{g}_\omega$ is assumed to be fixed. Without losing any generality, the function $\mathbf{g}_\omega$ is assumed to have a unitary ACF, i.e., $\mathbf{R}_G = \mathbf{I}$. We discuss the optimality condition when the parameters of $\mathbf{f}_\theta$ are updated to minimize the cost $r(\mathbf{f}_\theta, \mathbf{g}_\omega)$; \vspace{-7pt}
    \item The goal is to prove that when the cost is minimized, the Wiener solution is satisfied: $\mathbb{E}[\mathbf{g}_\omega(\mathbf{u})|x] = \mathbf{P}^\intercal \mathbf{R}_F^{-\frac{1}{2}}\mathbf{f}_\theta(x)$. If this can be proven, we can further fix $\mathbf{f  }_\theta$ and update $\mathbf{g}_\omega$ and show that $\mathbb{E}[\mathbf{f}_\omega(\mathbf{x})|u] = \mathbf{P} \mathbf{R}_G^{-\frac{1}{2}}\mathbf{g}_\omega(u)$ for the opposite direction. 
\end{itemize}
\begin{proof}
Given any two neural networks \(\mathbf{f}_\theta\) and \(\mathbf{g}_\omega\) that meet the specified assumptions. Applying the Schur complement to the cost function gives us
\begin{equation}
\begin{gathered}
r(\mathbf{f}_\theta, \mathbf{g}_\omega) = \log \det \mathbf{R}_F + \log\det (\mathbf{I} - \mathbf{P}^\intercal \mathbf{R}_F^{-1} \mathbf{P}) - \log \det \mathbf{R}_F = \log\det (\mathbf{I} - \mathbf{P}^\intercal \mathbf{R}_F^{-1} \mathbf{P}).
\end{gathered}
\end{equation}
The goal is to prove that, when the cost is minimized, the Wiener solution, which uses \(\mathbf{f}_\theta\) to predict \(\mathbf{g}_\omega\), is satisfied. That is, \(\mathbb{E}[\mathbf{g}_\omega(\mathbf{u})|x] = \mathbf{P}^\intercal \mathbf{R}_F^{-\frac{1}{2}}\mathbf{f}_\theta(x)\).

Define the matrix $\mathbf{M}$ as $\mathbf{M} := \mathbf{P}^\intercal \mathbf{R}_F^{-1} \mathbf{P}$, which corresponds to $\overline{\mathbf{P}}\; \overline{\mathbf{P}}^\intercal$, as detailed in Section~\ref{CORE-THEOREM}. It can be easily shown that $\mathbf{M}$ is a positive semidefinite matrix. Representing the eigenvalues of $\mathbf{M}$ with $\mathbf{\Sigma} = \text{diag}([\sigma_1, \sigma_2, \ldots, \sigma_K]^\intercal)$, we write the cost function in the following form:
\begin{equation}
\begin{gathered}
r := \sum_{i=1}^K \log(1- \sigma_i(\theta)) 
\end{gathered}
\end{equation}
We aim to demonstrate that maximizing the sum of eigenvalues is equivalent to maximizing each individual eigenvalue. Notably, assuming a universal function approximator allows us to assume, without loss of generality, that each eigenvalue is assigned with a different model with parameters $\theta_i$. This results in a new cost function, $\Tilde{r}$. Consequently, minimizing $\Tilde{r}$ satisfies: \vspace{-5pt}
\begin{equation}
\begin{gathered}
\min_{\theta_1, \cdots, \theta_K} \Tilde{r} =  \min_{\theta_1, \cdots, \theta_K} \sum_{i=1}^K \log(1- \sigma_i(\theta_k)) = \sum_{i=1}^K \min_{\theta_k} \log(1- \sigma_i(\theta_k)) \leq \min_{\theta} {r}.
\end{gathered}
\end{equation}
Since each eigenvalue is individually optimized, the optimal cost of $\Tilde{r}$ essentially acts as a lower bound for the optimal $r$. By assuming a universal approximator for optimization, we can use this lower bound to optimize each individual eigenvalue, without loss of generality. This demonstrates that using the sum of eigenvalues (i.e., the trace) instead of the log determinant does not compromise generality and simplifies the analysis. Applying the matrix trace property, we have $\frac{\partial \mathbf{Tr}(\mathbf{I} - \mathbf{M})}{\partial \theta} = -\frac{\partial \mathbf{Tr}(\mathbf{M})}{\partial \theta} = -\mathbf{Tr}(\frac{\partial \mathbf{M}}{\partial \theta})$. Thus, we proceed to investigate $\frac{\partial \mathbf{M}}{\partial \theta}$, which can be written as follows:
\begin{equation}
\begin{gathered}
\frac{\partial \mathbf{M}}{\partial \theta} = \frac{\partial \mathbf{P}^\intercal}{\partial \theta} \mathbf{R}_F^{-1} \mathbf{P} - \mathbf{P}^\intercal  \mathbf{R}_F^{-1} \frac{\partial \mathbf{R}_F}{\partial \theta} \mathbf{R}_F^{-1} \mathbf{P} +  \mathbf{P}^\intercal \mathbf{R}_F^{-1} \frac{\partial \mathbf{P}}{\partial \theta}, 
\end{gathered}
\label{gradients_of_M}
\end{equation}
Each term in~\eqref{gradients_of_M} can further be written as
\begin{equation}
\begin{gathered}
\frac{\partial \mathbf{P}}{\partial \theta} = \mathbb{E}[\frac{\partial \mathbf{f}_\theta(\mathbf{x})}{\partial \theta} \mathbf{g}_\omega^\intercal (\mathbf{u})], \;\; \frac{\partial \mathbf{R}_F}{\partial \theta} = \mathbb{E}[\frac{\partial \mathbf{f}_\theta(\mathbf{x})}{\partial \theta} \mathbf{f}_\theta^\intercal (\mathbf{x})] + \mathbb{E} [\mathbf{f}_\theta(\mathbf{x}) \frac{\partial \mathbf{f}_\theta^\intercal(\mathbf{x})}{\partial \theta} ].
\end{gathered}
\label{eq_rfff}
\end{equation}
Denote $\mathbf{W} := \mathbf{R}_F^{-1} \mathbf{P}$ as the weighting function for $\mathbf{f}_\theta$, \eqref{eq_rfff} can be further simplified as
\begin{equation}
\resizebox{1\linewidth}{!}{
$\begin{aligned}
\frac{\partial \mathbf{M}}{\partial \theta} &= \frac{\partial \mathbf{P}^\intercal }{\partial \theta} \mathbf{W} + \mathbf{W}^\intercal \frac{\partial \mathbf{R}_F}{\partial \theta} \mathbf{W} + \mathbf{W}^\intercal \frac{\partial \mathbf{P} }{\partial \theta}  \\
& = \mathbb{E} \small[\mathbf{g}_\omega(\mathbf{u}) \frac{\partial \mathbf{f}_\theta^\intercal(\mathbf{x})}{\partial \theta} \small] \mathbf{W} - \mathbf{W}^\intercal \left( \mathbb{E}\small[\frac{\partial \mathbf{f}_\theta(\mathbf{x})}{\partial \theta} \mathbf{f}_\theta^\intercal (\mathbf{x})\small] + \mathbb{E} \small[\mathbf{f}_\theta(\mathbf{x}) \frac{\partial \mathbf{f}_\theta^\intercal(\mathbf{x})}{\partial \theta} \small] \right) \mathbf{W} + \mathbf{W}^\intercal \mathbb{E}\small[\frac{\partial \mathbf{f}_\theta(\mathbf{x})}{\partial \theta} \mathbf{g}_\omega^\intercal (\mathbf{u})\small]. \\
&= \mathbb{E} \small[\mathbf{W}^\intercal \frac{\partial \mathbf{f}_\theta(\mathbf{x})}{\partial \theta} \left( \mathbf{g}_\omega^\intercal (\mathbf{u}) - \mathbf{f}_\theta^\intercal(\mathbf{x}) \mathbf{W} \right)\small] + \mathbb{E}\small[ (\mathbf{g}_\omega(\mathbf{u}) - \mathbf{W}^\intercal \mathbf{f}_\theta(\mathbf{x})) \frac{\partial \mathbf{f}_\theta^\intercal(\mathbf{x})}{\partial \theta} \mathbf{W}\small].
\end{aligned}$}
\label{partial_derivative_of_M}
\end{equation}
Further writing $\mathbf{e} :=  \mathbf{g}_\omega(\mathbf{u}) - \mathbf{W}^\intercal \mathbf{f}_\theta(\mathbf{x})$, we obtain the final form:
\begin{equation}
\begin{aligned}
\frac{\partial \mathbf{M}}{\partial \theta} = \mathbf{W}^\intercal \mathbb{E}[\frac{\partial 
\mathbf{f}_\theta(\mathbf{x})}{\partial \theta} \mathbf{e}^\intercal] + \mathbb{E}[\mathbf{e} \frac{\partial \mathbf{f}_\theta^\intercal(\mathbf{x})}{\partial \theta} ]\mathbf{W}. 
\end{aligned}
\end{equation}
In Wiener filter theory, the symbol $\mathbf{e}$ represents the prediction error. While typically $\mathbf{e}$ is a random variable, here it becomes a vector. It is essential to distinguish the function $\mathbf{g}_\omega$'s appearance in the prediction error $\mathbf{e}$, whereas in the Wiener solution, it appears as a conditional expectation. Even though it might not be feasible to find a function $\mathbf{g}_\omega$ that minimizes the error to zero, the Wiener solution remains applicable, providing the smallest achievable error. Therefore, we employ the tower property to the expectation in $\frac{\partial \mathbf{M}}{\partial \theta}$~\eqref{partial_derivative_of_M} to close the gap:
\begin{equation}
\begin{aligned}
\mathbb{E}[\frac{\partial \mathbf{f}_\theta(\mathbf{x})}{\partial \theta} \mathbf{g}_\omega^\intercal (\mathbf{u})] = \iint \frac{\partial \mathbf{f}_\theta(x)}{\partial \theta} \mathbf{g}_\omega^\intercal (u) p(x, u) dxdu  &= \int \frac{\partial \mathbf{f}_\theta(x)}{\partial \theta} \int \mathbf{g}_\omega^\intercal (u) p(u|x) du\, p(x) dx\\
&= \mathbb{E}_\mathbf{x}\left[\frac{\partial \mathbf{f}_\theta(\mathbf{x})}{\partial \theta} \mathbb{E}_\mathbf{u}[\mathbf{g}_\omega^\intercal (\mathbf{u})|\mathbf{x}]\right]. 
\end{aligned}
\label{}
\end{equation}
We now replace the function $\mathbf{g}_\omega$ in both $\mathbf{e}$ and $\frac{\partial \mathbf{M}}{\partial \theta}$ with its conditional expectation, as\begin{equation}
\begin{gathered}
\widetilde{\mathbf{e}} =   \mathbb{E}[\mathbf{g}_\omega (\mathbf{u})|\mathbf{x}] - \mathbf{W}^\intercal \mathbf{f}_\theta(\mathbf{x}), \;\; \frac{\partial \mathbf{M}}{\partial \theta} = \mathbf{W}^\intercal \mathbb{E}[\frac{\partial \mathbf{f}_\theta(\mathbf{x})}{\partial \theta} \widetilde{\mathbf{e}}^\intercal ] + \mathbb{E}[\widetilde{\mathbf{e}} \frac{\partial \mathbf{f}_\theta^\intercal(\mathbf{x})}{\partial \theta} ]\mathbf{W}. 
\end{gathered}
\end{equation}
As a result, its trace satisfies $\textbf{Tr}(\frac{\partial \mathbf{M}}{\partial \theta}) = \sum_{i=1}^K \mathbb{E}[(\mathbf{W}^\intercal \frac{\partial \mathbf{f}_\theta^\intercal(\mathbf{x})}{\partial \theta})_i \, \widetilde{\mathbf{e}}_i]$. Remember that the function $\mathbf{f}_\theta$ is universal, which means a sufficiently large number of parameters can be chosen arbitrarily. The derivative of each parameter can be viewed as a measurable function in the probability space defined by the data. In order to set the derivative to zero for any parameter in this universal model, the error $\widetilde{\mathbf{e}}$ has to be orthogonal to any arbitrary measurable function within this probability space. Thus, the only feasible option is setting the error $\widetilde{\mathbf{e}}$ as a zero vector, such as $\mathbb{E}[\mathbf{g}_\omega(\mathbf{u})|x] = \mathbf{P}^\intercal \mathbf{R}_F^{-1}\mathbf{f}_\theta(x)$. This implies that the optimality is reached if and only if the Wiener solution is satisfied, as illustrated below:
\begin{equation}
\begin{gathered}
\mathbb{E}[\mathbf{g}_\omega(\mathbf{u})|x] = \mathbf{P}^\intercal \mathbf{R }_F^{-1}\mathbf{f}_\theta(x) \Leftrightarrow \textbf{Tr}\left(\frac{\partial \mathbf{M}}{\partial \theta}\right) = \mathbf{0} \Leftrightarrow \text{maximal eigenvalues} \Leftrightarrow \text{optimal cost}.\end{gathered}
\end{equation}
We have now established that the the satisfaction of the Wiener solution is equivalent to the optimal condition of the cost, where $\mathbf{f}_\theta$ is used to predict a fixed $\mathbf{g}_\omega$. Similarly, the reverse direction, i.e., employing the Wiener solution to use $\mathbf{g}_\omega$ to predict a fixed $\mathbf{f}_\theta$, must also be satisfied. This result can then be used to prove Lemma~\ref{lemma_solution}, which proves that these functions learn the eigenfunctions of CDR.\end{proof} \vspace{-20pt}




\section{Implementation details for experiments}
\label{implementation_details}

\textbf{Statistical dependence datasets.} GAUSS uses multivariate Gaussian distributions with standardized marginals (mean $0$, standard deviation $1$), and correlation coefficient $\rho$. The cross density has a closed form $p(x'|x) = \int_{\mathcal{U}}\mathcal{N}(x'- \rho  u, 1-\rho^2)  \mathcal{N}(u- \rho  x, 1-\rho^2) du = \mathcal{N}(x' - \rho^2 x, (1 - \rho^2)\rho^2)$, with eigenfunctions having a closed-form solution using Hermite polynomials (\citealt{zhu1997gaussian}, \citealt{williams2006gaussian}). We implement conventional distributions like spiral dataset (SPIRAL), check dataset (CHECK), and sine waves (SW) (\citealt{pedregosa2011scikit}), visualized in Figure~\ref{distributiON_compare}. The sample size for these datasets is chosen to be $10^4$. \vspace{-1pt}
\begin{figure}[h]
\vspace{-7pt}
  \centering
  \begin{subfigure}{.2\textwidth}
    \includegraphics[width=\linewidth]{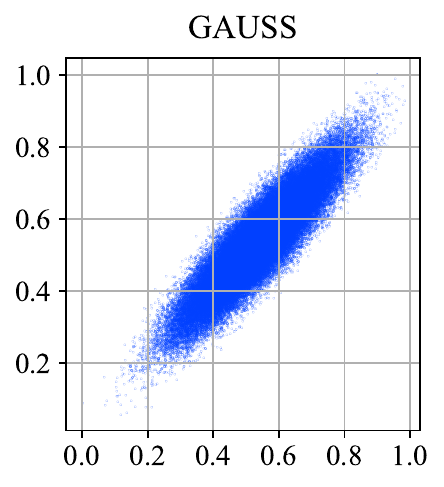}\vspace{-8pt}
    \phantomsubcaption
  \end{subfigure}
    \begin{subfigure}{.2\textwidth}
    \includegraphics[width=\linewidth]{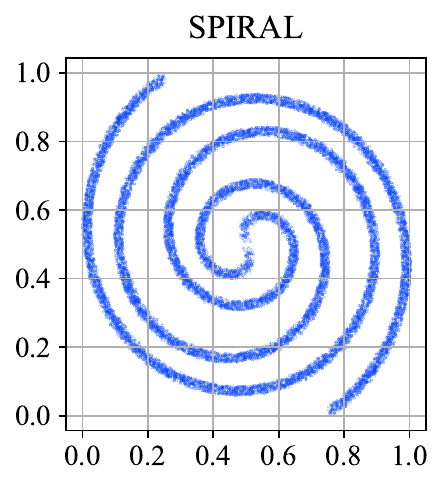}\vspace{-8pt}
    \phantomsubcaption
  \end{subfigure}
      \begin{subfigure}{.2\textwidth}
    \includegraphics[width=\linewidth]{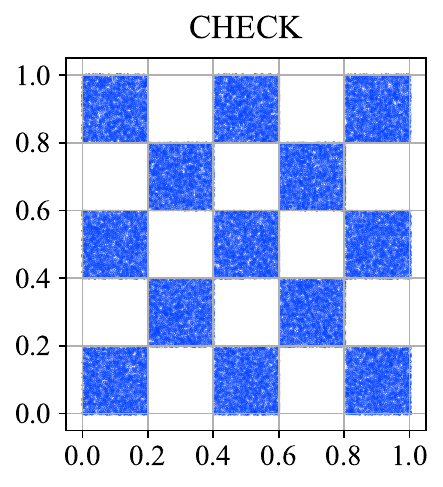}\vspace{-8pt}
    \phantomsubcaption
  \end{subfigure}
        \begin{subfigure}{.2\textwidth}
    \includegraphics[width=\linewidth]{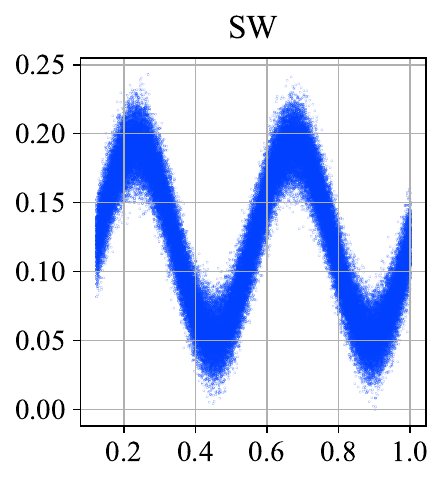}\vspace{-8pt}
    \phantomsubcaption
  \end{subfigure}
  \begin{minipage}{.8\linewidth}
\caption{Datasets used in Figure~\ref{spiral_plot} of Section~\ref{experiment_section}. While they have similar TSD, their MSD is significantly different; CHECK has only two positive eigenvalues, whereas SPIRAL has many more. }
\label{distributiON_compare}
    \end{minipage}
\end{figure}
\vspace{-1pt}

\noindent\textbf{Network structure.} The statistical dependence measurement  (Section~\ref{experiment_section}) and Markov chain aggregation (Section~\ref{aggregation_section}) use a five-layer MLP, having 2000 units in each layer, a ReLU activation function, and a Sigmoid function in the final layer to constrain the output to a range between $0$ and $1$. In FMCA-SOFT (Section~\ref{aggregation_section}), a Softmax function replaces the Sigmoid function in the last layer. All networks are optimized using an Adam optimizer with a learning rate of $10^{-4}$, $\beta_1= 0.5$ and $\beta_2 = 0.9$, and a batch size of 200. The networks used for image coding are shown in Table~\ref{image_coding_TABLE}.  \vspace{5pt}

\begin{table}[t]
\caption{\small Neural Network structure utilized in image coding.}
\centering
\footnotesize
\begin{threeparttable}
\resizebox{1\linewidth}{!}{
\begin{tabular}{lllllll}
\toprule
                                   & \textbf{Net} & \textbf{Inputs} & \textbf{Input Dim.} &  \textbf{Net Type} & \textbf{Outputs} & \textbf{Output Dim.}\\ \midrule
\multirow{2}{*}{\textbf{FMCA-PT}}  & $\mathbf{f}_\theta$     & Patches\tnote{1} &  (3, 32, 32)& ResNet-18 & \multirow{2}{*}{Bases} & \multirow{2}{*}{128}\\
                                   & $\mathbf{g}_\omega$     & Patches&  (3, 32, 32)& ResNet-18\\ \midrule
\multirow{2}{*}{\textbf{FMCA-FC}}  & $\mathbf{f}_\theta$          & Images& (3, 32, 32)&ResNet-18 & \multirow{2}{*}{Bases} & \multirow{2}{*}{128}\\ 
                                   & $\mathbf{g}_\omega$          & Uniform Noise & 500 & MLP\tnote{3}\\ \midrule
\multirow{2}{*}{\textbf{FMCA-M}}   & $\mathbf{f}_\theta$    & Images& (3, 32, 32) & ResNet-18& \multirow{2}{*}{Bases} & \multirow{2}{*}{128}\\
                                   & $\mathbf{g}_\omega$          &Same Image&(3, 32, 32) & ResNet-18\\ \midrule
\multirow{2}{*}{\textbf{FMCA-C}}   & $\mathbf{f}_\theta$          &Images& (3, 32, 32) & ResNet-18& \multirow{2}{*}{Bases} & \multirow{2}{*}{128}\\
                                   & $\mathbf{g}_\omega$          &Class&10 or 100\tnote{2} & MLP\tnote{3}\\ \toprule
\rowcolor{gray!10} \multicolumn{2}{l}{\textbf{AE-Encoder}}                        & Images & (3, 32, 32) & ResNet-18 & Latent Codes & 128\\ 
\rowcolor{gray!20} \multicolumn{2}{l}{\textbf{Decoder\tnote{4}}}              & Codes& 128 &VGG-16 & Reconstruction & (3, 32, 32)\\ 
\rowcolor{gray!30} \multicolumn{2}{l}{\textbf{Classifier}}                & Codes& 128 & MLP\tnote{3} & Labels & 10 or 100\\ \toprule
\end{tabular}}
\begin{tablenotes}
\footnotesize
\item[1] Patches fed into paired networks belong to the same image.
\item[2] Dimensions for either CIFAR10 or CIFAR100.
\item[3] FMCA-FC MLP: five-layer, 2000 units/layer; FMCA-C MLP: three-layer, 300 units/layer; Classifier MLP: five-layer, 2000 units/layer.
\item[4] Decoders for AE and FMCA reconstructions.
\end{tablenotes}
\end{threeparttable}
\label{image_coding_TABLE}
\end{table}

\noindent\textbf{Regularization Parameter for ACFs.} When computing the inverse of any ACFs for the gradient of the log determinant, similar to the pseudo-inverse, a small diagonal matrix scaled by a regularization parameter $\epsilon \mathbf{I}$ is added. This constant is vital for training stability and correlates with the regularization constant in TSD's definition (Definition~\ref{definition_eigenspectrum}). For statistical dependence measurement, Markov chain aggregation, and FMCA-FC with images, $\epsilon = 10^{-3}$ is chosen. For all other image experiments, $\epsilon=10^{-1}$ is chosen. We find that $\epsilon=10^{-1}$ produces the best generalization performance.\vspace{5pt}

\noindent\textbf{Estimating basis functions and MSD after training.} Since training updates use a batch to estimate ACFs and the gradient, this estimation is biased for the full dataset. Hence, after training, ACFs and CCFs are re-estimated to solve the eigenvalue problem, and this solution is used to normalize the neural network outputs, obtaining the basis functions and MSD. An alternative approach utilizes an adaptive filter for estimating ACFs during training for gradient updates and estimation, similar to the Adam optimizer. Both approaches produce accurate estimations of the ground truth. The pseudo code of the algorithm is shown in Algorithm~\ref{algorithm_1}. \vspace{5pt}

\begin{algorithm}[t]
\small
\caption{Pseudocode of FMCA}
\begin{algorithmic}[1]
\State Initialize neural networks $\mathbf{f}_\theta:\mathcal{X}\rightarrow \mathbb{R}^K$ and $\mathbf{g}_\omega: \mathcal{U} \rightarrow \mathbb{R}^L$; Typically choose $K=L$.
\State Given the input random process $\{x_n(t):t=1,\cdots,T_1\}_{n=1}^N$; Construct the reference random process $\{u_n(t):t=1,\cdots,T_2\}_{n=1}^N$ following Section~\ref{reference_process}. 
\While{the equilibrium not reached}
\State $k \leftarrow k+1$
\State Sample $M$ realizations $\{(x_m. u_m)\}_{m=1}^M$ from the joint random process.
\State Compute ${\mathbf{R}}_F = \frac{1}{M} \sum_{m=1}^M \mathbf{f}_\theta(x_m)  \mathbf{f}^\intercal_\theta(x_m) +\epsilon \mathbf{I}_{K}$, \; ${\mathbf{R}}_G = \frac{1}{M} \sum_{m=1}^M \mathbf{g}_\omega(x_m)  \mathbf{g}^\intercal_\omega(x_m) +\epsilon \mathbf{I}_{L}$,

${\mathbf{P}}_{FG} = \sum_{m=1}^M  \mathbf{f}_\theta(x_m) \mathbf{g}_\omega^\intercal(x_m)$.




\State Compute $\frac{\partial {r}}{\partial \theta} = \text{Tr} ( \begin{bsmallmatrix}{\mathbf{R}}_F & {\mathbf{P}}_{FG} \\
{\mathbf{P}}^\intercal_{FG} & {\mathbf{R}}_G
\end{bsmallmatrix}^{-1}  \frac{\partial \begin{bsmallmatrix}{\mathbf{R}}_F & {\mathbf{P}}_{FG} \\
{\mathbf{P}}^\intercal_{FG} & {\mathbf{R}}_G
\end{bsmallmatrix}}{\partial \theta} )  - \text{Tr}( {{\mathbf{R}}_F^{-1}} \frac{\partial {\mathbf{R}}_F}{\partial \theta} )$

$\frac{\partial {r}}{\partial \omega} = \text{Tr} ( \begin{bsmallmatrix}{\mathbf{R}}_F & {\mathbf{P}}_{FG} \\
{\mathbf{P}}^\intercal_{FG} & {\mathbf{R}}_G
\end{bsmallmatrix}^{-1}  \frac{\partial \begin{bsmallmatrix}{\mathbf{R}}_F & {\mathbf{P}}_{FG} \\
{\mathbf{P}}^\intercal_{FG} & {\mathbf{R}}_G
\end{bsmallmatrix}}{\partial \omega} )  - \text{Tr}( {{\mathbf{R}}_G^{-1}} \frac{\partial {\mathbf{R}}_G}{\partial \omega} )$.

\State Gradient descent with an optimizer $\theta \leftarrow \text{optimizer}(\theta, \frac{\partial {r}}{\partial \theta})$, $\omega \leftarrow \text{optimizer}(\omega, \frac{\partial {r}}{\partial \omega})$. 

\EndWhile
\If{MSD and TSD needed}
\State Solve the eigenproblem: $\overline{\mathbf{P}} =  {\mathbf{R}}_F^{-\frac{1}{2}} {\mathbf{R}}_{FG} {\mathbf{R}}_G^{-\frac{1}{2}}, \;\overline{\mathbf{P}} \; \overline{\mathbf{P}}^{\intercal} = \mathbf{Q}_F \mathbf{\Sigma} {\mathbf{Q}_F}^\intercal, \;
\overline{\mathbf{P}}^{\intercal} \overline{\mathbf{P}} = \mathbf{Q}_G \mathbf{\Sigma} {\mathbf{Q}_G}^\intercal$.
\State Construct MSD and TSD: $\mathbf{\Sigma} \Rightarrow \{\sigma_1, \cdots, \sigma_K\}, \; T_K = - \frac{1}{2}\sum_{i=2}^K \log({1-\sigma_i})$.
\State Normalize with ACFs: $\overline{\mathbf{f}_\theta} = {\mathbf{R}}_F^{-\frac{1}{2}} \mathbf{f}_\theta, \;\;  \overline{\mathbf{g}_\omega} = {\mathbf{R}}_G^{-\frac{1}{2}} \mathbf{g}_\omega$. 
\State Normalize with eigenvectors: $\widehat{\mathbf{f}_\theta} = \mathbf{Q}_F^T \overline{\mathbf{f}_\theta},\;\; \widehat{\mathbf{g}_\omega} = \mathbf{Q}_G^T \overline{\mathbf{g}_\omega}$; Here $\widehat{\mathbf{f}_\theta}, \widehat{\mathbf{g}_\omega}$ are the estimated bases.
\State Construct CDR: $\rho(x, x') = \widehat{\mathbf{f}_\theta}^\intercal(x) \mathbf{\Sigma} \widehat{\mathbf{f}_\theta} (x')$, $\rho(u, u') = \widehat{\mathbf{g}_\omega}^\intercal(u) \mathbf{\Sigma} \widehat{\mathbf{g}_\omega} (u')$.
\EndIf
\end{algorithmic}
\label{algorithm_1}
\end{algorithm}%

\noindent\textbf{Dimensions of network outputs.} FMCA calculates MSD using a fixed number of eigenvalues and a truncated TSD, limited by the neural network's output dimensions. Typically, the primary network and the reference network share the same output dimensions $K$ for simplicity. For statistical dependence measurement and Markov chain aggregation, $K=50$ is chosen, except for FMCA-SOFT, where the number of bases matches the aggregated states count, set to $K=6$ and $K=20$. In the case of image datasets, $K=128$ is chosen. \vspace{5pt}


\noindent\textbf{Implementations of statistical dependence baselines.} The steps of implementing KICA and HSIC are as follows:  \vspace{-5pt}
\begin{itemize}[leftmargin=*]
\item Estimate Gram matrices: For two r.p. $\mathbf{x}$ and $\mathbf{u}$, and Gaussian kernel $\mathcal{K}(x_i, x_j) = \mathcal{N}(x_i-x_j;\delta)$ with $\delta$ as the standard deviation, compute individual process Gram matrices $\mathbf{R}_{XX}$ and $\mathbf{R}_{UU}$. Each matrix consists of $\mathbf{R}_{XX}(i, j) = \mathcal{K}(x_i, x_j)$ and $\mathbf{R}_{UU}(i, j) = \mathcal{K}(u_i, u_j)$. Note that KICA and HSIC do not require CCF between $\mathbf{x}$ and $\mathbf{u}$. \vspace{-5pt}
\item KICA-KGV: Construct the normalization matrix $\mathbf{N}_{i,j} = 
\begin{cases} 
   1 - \frac{1}{n} & \text{if } i = j \\
  -\frac{1}{n} & \text{if } i \neq j
\end{cases}
$. Normalize the two Gram matrices as $\overline{\mathbf{R}_{XX}} = \mathbf{N}\mathbf{R}_{XX}\mathbf{N}$ and $\overline{\mathbf{R}_{UU}} = \mathbf{N}\mathbf{R}_{UU}\mathbf{N}$. Construct matrix $\mathbf{A} = \begin{bmatrix}\mathbf{A}_1 & 0\\ 0 & \mathbf{A}_2 \end{bmatrix}$, where $\mathbf{A}_1 = \overline{\mathbf{R}_{XX}}\,\overline{\mathbf{R}_{UU}}$ and $\mathbf{A}_2 = \overline{\mathbf{R}_{UU}}\,\overline{\mathbf{R}_{XX}}$. Construct matrix $\mathbf{B} = \begin{bmatrix}\mathbf{B}_1 & 0\\ 0 & \mathbf{B}_2 \end{bmatrix}$, where $\mathbf{B}_1 = (\overline{\mathbf{R}_{XX}}+\epsilon \mathbf{I}) \,(\overline{\mathbf{R}_{XX}}+\epsilon \mathbf{I})$ and $\mathbf{B}_2 = (\overline{\mathbf{R}_{UU}}+\epsilon \mathbf{I}) \,(\overline{\mathbf{R}_{UU}}+\epsilon \mathbf{I})$. Solve the generalized eigenvalue problem $\mathbf{A} \mathbf{v}_i = \sigma_i \mathbf{B} \mathbf{v}_i$, where $i=1, \cdots, 2N$. The generalized eigenproblem will generate $2N$ eigenvalues symmetrical over the real line. We use the $N$ positive eigenvalues to compute the TSD, with $T_{KICA} = -\frac{1}{2}\sum_{i=1}^N\log(1 - \sigma_i^2)$.
\item HSIC-NOCCO: Construct $\mathbf{C} = \mathbf{B}_1^{-\frac{1}{2}} \mathbf{A}_1 \mathbf{B_2}^{-\frac{1}{2}}$. Perform the eigenvalue problem $\mathbf{C}\mathbf{v}_i = \sigma_i \mathbf{v}_i$, where $i=1,\cdots, N$. Compute the TSD, with $T_{HSIC} = \text{Trace}(\mathbf{C})$. 
\item Hyperparameters: Kernel size $\delta=0.1$ and small constant $\epsilon=0.1$. Due to the normalize matrix $\mathbf{N}$ centering the basis functions, KICA/HSIC's first eigenvalue corresponds to FMCA's second eigenvalue, and so forth. 
\end{itemize}

\noindent To implement MINE, follow these steps: Initialize a three-layer MLP $f_\theta$, consisting of $2000$ units per layer and a Sigmoid at the final layer, which operates on space $\mathcal{X}\times \mathcal{U}$ and generates a one-dimensional output. Sample $\{\mathbf{x}, \mathbf{u}\}$ from the joint distribution, concatenate their outputs in the feature dimension, and compute $f_\theta(\mathbf{x}, \mathbf{u})$. Then, sample $\{\mathbf{x}', \mathbf{u}'\}$ from their respective marginal distributions, concatenate, and calculate $f_\theta(\mathbf{x}', \mathbf{u}')$. Minimize the variational cost $\min_{\theta} \mathbb{E}[f_\theta(\mathbf{x}, \mathbf{u})] - \log \mathbb{E}[e^{f_\theta(\mathbf{x}', \mathbf{u}')}\!+\!10^{-5}]$. Optimize the network using the Adam optimizer with a learning rate of $10^{-4}$, $\beta_1 = 0.5$, and $\beta_2 = 0.9$.
\vspace{10pt}

\noindent\textbf{Implementation of aggregation baselines.} The steps of implementing SSC and NMF are as follows: \vspace{-5pt}
\begin{itemize}[leftmargin=*]
\item Clean data: Retain taxi trip coordinates within the boundaries of $-74.05$ to $-73.85$ longitude and $40.65$ to $40.88$ latitude. \vspace{-5pt}
\item Discretize space: Perform K-means clustering on pickup/dropoff locations, obtain 1000 centers, and assign samples accordingly. \vspace{-5pt}
\item Estimate $pdf$: Empirically estimate the transition probability $p(x_{t+1}|x_t)$ and marginal probability $p(x_t)$ based on clustering centers. \vspace{-5pt}
\item SSC: Perform SVD on joint $pdf$ $p(x_t, x_{x+1})$, take top $r$ basis functions ($r=6$), and use K-means clustering for aggregated state assignments. \vspace{-5pt}
\item NMF: Decompose joint $pdf$ by minimizing the Frobenius norm $\min_{U, V} \; ||p(x_t, x_{t+1})-UV||_{\text{Fro}}^2 + ||U||_{\text{Fro}}^2$ and assign states based on the largest element in each row of $U$.\end{itemize}


\begin{thebibliography}{46}
\providecommand{\natexlab}[1]{#1}
\providecommand{\url}[1]{\texttt{#1}}
\expandafter\ifx\csname urlstyle\endcsname\relax
  \providecommand{\doi}[1]{doi: #1}\else
  \providecommand{\doi}{doi: \begingroup \urlstyle{rm}\Url}\fi

\bibitem[Andrew et~al.(2013)Andrew, Arora, Bilmes, and Livescu]{andrew2013deep}
Galen Andrew, Raman Arora, Jeff Bilmes, and Karen Livescu.
\newblock Deep canonical correlation analysis.
\newblock In \emph{Proceedings of the 30th International Conference on International Conference on Machine Learning}, pages 1247--1255, 2013.

\bibitem[Bach and Jordan(2002)]{bach2002kernel}
Francis~R Bach and Michael~I Jordan.
\newblock Kernel independent component analysis.
\newblock \emph{Journal of machine learning research}, 3\penalty0 (Jul):\penalty0 1--48, 2002.

\bibitem[Bach and Jordan(2005)]{bach2005probabilistic}
Francis~R Bach and Michael~I Jordan.
\newblock A probabilistic interpretation of canonical correlation analysis.
\newblock \emph{Technical Report 688, Department of Statistics, University of California}, 2005.

\bibitem[Baker(1977)]{baker1977numerical}
Christopher~TH Baker.
\newblock \emph{The numerical treatment of integral equations}.
\newblock Oxford University Press, 1977.

\bibitem[Barlow(1989)]{barlow1989unsupervised}
Horace~B Barlow.
\newblock Unsupervised learning.
\newblock \emph{Neural computation}, 1\penalty0 (3):\penalty0 295--311, 1989.

\bibitem[Barlow et~al.(1989)Barlow, Kaushal, and Mitchison]{barlow1989finding}
Horace~B Barlow, Tej~P Kaushal, and Graeme~J Mitchison.
\newblock Finding minimum entropy codes.
\newblock \emph{Neural Computation}, 1\penalty0 (3):\penalty0 412--423, 1989.

\bibitem[Belghazi et~al.(2018)Belghazi, Baratin, Rajeshwar, Ozair, Bengio, Courville, and Hjelm]{belghazi2018mutual}
Mohamed~Ishmael Belghazi, Aristide Baratin, Sai Rajeshwar, Sherjil Ozair, Yoshua Bengio, Aaron Courville, and Devon Hjelm.
\newblock Mutual information neural estimation.
\newblock In \emph{International conference on machine learning}, pages 531--540. PMLR, 2018.

\bibitem[Breiman and Friedman(1985)]{breiman1985estimating}
Leo Breiman and Jerome~H Friedman.
\newblock Estimating optimal transformations for multiple regression and correlation.
\newblock \emph{Journal of the American statistical Association}, 80\penalty0 (391):\penalty0 580--598, 1985.

\bibitem[Donoho and Stodden(2003)]{donoho2003does}
David Donoho and Victoria Stodden.
\newblock When does non-negative matrix factorization give a correct decomposition into parts?
\newblock \emph{Advances in neural information processing systems}, 16, 2003.

\bibitem[Duan et~al.(2019)Duan, Ke, and Wang]{duan2019state}
Yaqi Duan, Tracy Ke, and Mengdi Wang.
\newblock State aggregation learning from markov transition data.
\newblock \emph{Advances in Neural Information Processing Systems}, 32, 2019.

\bibitem[Fukumizu et~al.(2004)Fukumizu, Bach, and Jordan]{fukumizu2004dimensionality}
Kenji Fukumizu, Francis~R Bach, and Michael~I Jordan.
\newblock Dimensionality reduction for supervised learning with reproducing kernel hilbert spaces.
\newblock \emph{Journal of Machine Learning Research}, 5\penalty0 (Jan):\penalty0 73--99, 2004.

\bibitem[Gretton et~al.(2007)Gretton, Fukumizu, Teo, Song, Sch{\"o}lkopf, and Smola]{gretton2007kernel}
Arthur Gretton, Kenji Fukumizu, Choon Teo, Le~Song, Bernhard Sch{\"o}lkopf, and Alex Smola.
\newblock A kernel statistical test of independence.
\newblock \emph{Advances in neural information processing systems}, 20, 2007.

\bibitem[Hagenblad(1999)]{hagenblad1999aspects}
Anna Hagenblad.
\newblock \emph{Aspects of the Identification of Wiener Models}.
\newblock 1999.

\bibitem[Hagenblad and Ljung(2000)]{hagenblad2000maximum}
Anna Hagenblad and Lennart Ljung.
\newblock Maximum likelihood estimation of wiener models.
\newblock In \emph{Proc. IEEE Conf. Decision and Control}, volume~3, pages 2417--2418. IEEE, 2000.

\bibitem[Hayati et~al.(2023)Hayati, Fukumizu, and Parvardeh]{hayati2023kernel}
Saeed Hayati, Kenji Fukumizu, and Afshin Parvardeh.
\newblock Kernel mean embedding of probability measures and its applications to functional data analysis.
\newblock \emph{Scandinavian Journal of Statistics}, 2023.

\bibitem[Heinonen et~al.(2016)Heinonen, Mannerstr{\"o}m, Rousu, Kaski, and L{\"a}hdesm{\"a}ki]{heinonen2016non}
Markus Heinonen, Henrik Mannerstr{\"o}m, Juho Rousu, Samuel Kaski, and Harri L{\"a}hdesm{\"a}ki.
\newblock Non-stationary gaussian process regression with hamiltonian monte carlo.
\newblock In \emph{Artificial Intelligence and Statistics}, pages 732--740. PMLR, 2016.

\bibitem[Heinonen et~al.(2018)Heinonen, Yildiz, Mannerstr{\"o}m, Intosalmi, and L{\"a}hdesm{\"a}ki]{heinonen2018learning}
Markus Heinonen, Cagatay Yildiz, Henrik Mannerstr{\"o}m, Jukka Intosalmi, and Harri L{\"a}hdesm{\"a}ki.
\newblock Learning unknown ode models with gaussian processes.
\newblock In \emph{International conference on machine learning}, pages 1959--1968. PMLR, 2018.

\bibitem[Hu and Pr{\'\i}ncipe(2021)]{hu2021mimo}
Bo~Hu and Jos{\'e}~C Pr{\'\i}ncipe.
\newblock Mimo modeling by learning explicitly the projection space: The maximum correlation ratio cost function.
\newblock \emph{IEEE Transactions on Signal Processing}, 69:\penalty0 6039--6054, 2021.

\bibitem[Hu and Pr{\'\i}ncipe(2023)]{hu2023cross}
Bo~Hu and Jos{\'e}~C Pr{\'\i}ncipe.
\newblock Cross density kernel for nonstationary signal processing.
\newblock In \emph{2023 IEEE Statistical Signal Processing Workshop (SSP)}, pages 195--199. IEEE, 2023.

\bibitem[Huang et~al.(2018)Huang, Wornell, and Zheng]{huang2018gaussian}
Shao-Lun Huang, Gregory~W Wornell, and Lizhong Zheng.
\newblock Gaussian universal features, canonical correlations, and common information.
\newblock In \emph{2018 IEEE Information Theory Workshop (ITW)}, pages 1--5. IEEE, 2018.

\bibitem[Huang et~al.(2019)Huang, Makur, Wornell, and Zheng]{huang2019universal}
Shao-Lun Huang, Anuran Makur, Gregory~W Wornell, and Lizhong Zheng.
\newblock On universal features for high-dimensional learning and inference.
\newblock \emph{arXiv preprint arXiv:1911.09105}, 2019.

\bibitem[Hyv{\"a}rinen and Oja(2000)]{hyvarinen2000independent}
Aapo Hyv{\"a}rinen and Erkki Oja.
\newblock Independent component analysis: algorithms and applications.
\newblock \emph{Neural networks}, 13\penalty0 (4-5):\penalty0 411--430, 2000.

\bibitem[Karhunen(1947)]{karhunen1947lineare}
Kari Karhunen.
\newblock \emph{{\"U}ber lineare Methoden in der Wahrscheinlichkeitsrechnung: Akadem. Abhandlung}.
\newblock PhD thesis, Sana, 1947.

\bibitem[Lee and Seung(2000)]{lee2000algorithms}
Daniel Lee and H~Sebastian Seung.
\newblock Algorithms for non-negative matrix factorization.
\newblock \emph{Advances in neural information processing systems}, 13, 2000.

\bibitem[Lee and Seung(1999)]{lee1999learning}
Daniel~D Lee and H~Sebastian Seung.
\newblock Learning the parts of objects by non-negative matrix factorization.
\newblock \emph{Nature}, 401\penalty0 (6755):\penalty0 788--791, 1999.

\bibitem[Marcus and Minc(1992)]{marcus1992survey}
Marvin Marcus and Henryk Minc.
\newblock \emph{A survey of matrix theory and matrix inequalities}, volume~14.
\newblock Courier Corporation, 1992.

\bibitem[Muandet et~al.(2017)Muandet, Fukumizu, Sriperumbudur, Sch{\"o}lkopf, et~al.]{muandet2017kernel}
Krikamol Muandet, Kenji Fukumizu, Bharath Sriperumbudur, Bernhard Sch{\"o}lkopf, et~al.
\newblock Kernel mean embedding of distributions: A review and beyond.
\newblock \emph{Foundations and Trends{\textregistered} in Machine Learning}, 10\penalty0 (1-2):\penalty0 1--141, 2017.

\bibitem[Nguyen et~al.(2009)Nguyen, Wainwright, and Jordan]{nguyen2009surrogate}
XuanLong Nguyen, Martin~J Wainwright, and Michael~I Jordan.
\newblock On surrogate loss functions and f-divergences.
\newblock 2009.

\bibitem[Paninski(2003)]{Liamestimate}
Liam Paninski.
\newblock Estimation of entropy and mutual information.
\newblock \emph{Neural Computation}, 15\penalty0 (6):\penalty0 1191--1253, 2003.

\bibitem[Parzen(1961)]{parzen1961approach}
Emanuel Parzen.
\newblock An approach to time series analysis.
\newblock \emph{The Annals of Mathematical Statistics}, 32\penalty0 (4):\penalty0 951--989, 1961.

\bibitem[Parzen(1962)]{parzen1962estimation}
Emanuel Parzen.
\newblock On estimation of a probability density function and mode.
\newblock \emph{The annals of mathematical statistics}, 33\penalty0 (3):\penalty0 1065--1076, 1962.

\bibitem[Parzen(1999)]{parzen1999stochastic}
Emanuel Parzen.
\newblock \emph{Stochastic processes}.
\newblock SIAM, 1999.

\bibitem[Pedregosa et~al.(2011)Pedregosa, Varoquaux, Gramfort, Michel, Thirion, Grisel, Blondel, Prettenhofer, Weiss, Dubourg, et~al.]{pedregosa2011scikit}
Fabian Pedregosa, Ga{\"e}l Varoquaux, Alexandre Gramfort, Vincent Michel, Bertrand Thirion, Olivier Grisel, Mathieu Blondel, Peter Prettenhofer, Ron Weiss, Vincent Dubourg, et~al.
\newblock Scikit-learn: Machine learning in python.
\newblock \emph{the Journal of machine Learning research}, 12:\penalty0 2825--2830, 2011.

\bibitem[Pr{\'\i}ncipe(2010)]{Principe2010information}
Jos{\'e}~C Pr{\'\i}ncipe.
\newblock \emph{Information theoretic learning: Renyi's entropy and kernel perspectives}.
\newblock Springer Science \& Business Media, 2010.

\bibitem[Pr{\'\i}ncipe et~al.(2011)Pr{\'\i}ncipe, Liu, and Haykin]{principe2011kernel}
Jos{\'e}~C Pr{\'\i}ncipe, Weifeng Liu, and Simon Haykin.
\newblock \emph{Kernel adaptive filtering: a comprehensive introduction}.
\newblock John Wiley \& Sons, 2011.

\bibitem[R{\'e}nyi(1959)]{renyi1959measures}
Alfr{\'e}d R{\'e}nyi.
\newblock On measures of dependence.
\newblock \emph{Acta mathematica hungarica}, 10\penalty0 (3-4):\penalty0 441--451, 1959.

\bibitem[R{\"o}blitz and Weber(2013)]{roblitz2013fuzzy}
Susanna R{\"o}blitz and Marcus Weber.
\newblock Fuzzy spectral clustering by pcca+: application to markov state models and data classification.
\newblock \emph{Advances in Data Analysis and Classification}, 7:\penalty0 147--179, 2013.

\bibitem[Silverman(2018)]{silverman2018density}
Bernard~W Silverman.
\newblock \emph{Density estimation for statistics and data analysis}.
\newblock Routledge, 2018.

\bibitem[Sriperumbudur et~al.(2010)Sriperumbudur, Gretton, Fukumizu, Sch{\"o}lkopf, and Lanckriet]{sriperumbudur2010hilbert}
Bharath~K Sriperumbudur, Arthur Gretton, Kenji Fukumizu, Bernhard Sch{\"o}lkopf, and Gert~RG Lanckriet.
\newblock Hilbert space embeddings and metrics on probability measures.
\newblock \emph{The Journal of Machine Learning Research}, 11:\penalty0 1517--1561, 2010.

\bibitem[Taxi and Commission(2024)]{nyc_taxi_data}
New York~City Taxi and Limousine Commission.
\newblock Tlc trip record data, 2024.
\newblock URL \url{https://www.nyc.gov/site/tlc/about/tlc-trip-record-data.page}.
\newblock Accessed: 2024-01-20.

\bibitem[Wiener(1950)]{wiener1949extrapolation}
Norbert Wiener.
\newblock \emph{Extrapolation, Interpolation, and Smoothing of Stationary Time Series: with Engineering Applications}.
\newblock Cambridge, MA, USA: MIT Press, 1950.

\bibitem[Williams(1998)]{williams1998prediction}
Christopher~KI Williams.
\newblock Prediction with gaussian processes: From linear regression to linear prediction and beyond.
\newblock In \emph{Learning in graphical models}, pages 599--621. Springer, 1998.

\bibitem[Williams and Rasmussen(2006)]{williams2006gaussian}
Christopher~KI Williams and Carl~Edward Rasmussen.
\newblock \emph{Gaussian processes for machine learning}, volume~2.
\newblock MIT press Cambridge, MA, 2006.

\bibitem[Wills et~al.(2013)Wills, Sch{\"o}n, Ljung, and Ninness]{wills2013identification}
Adrian Wills, Thomas~B. Sch{\"o}n, Lennart Ljung, and Brett Ninness.
\newblock Identification of hammerstein--wiener models.
\newblock \emph{Automatica}, 49\penalty0 (1):\penalty0 70--81, 2013.

\bibitem[Zhang and Wang(2019)]{zhang2019spectral}
Anru Zhang and Mengdi Wang.
\newblock Spectral state compression of markov processes.
\newblock \emph{IEEE transactions on information theory}, 66\penalty0 (5):\penalty0 3202--3231, 2019.

\bibitem[Zhu et~al.(1997)Zhu, Williams, Rohwer, and Morciniec]{zhu1997gaussian}
Huaiyu Zhu, Christopher~KI Williams, Richard Rohwer, and Michal Morciniec.
\newblock Gaussian regression and optimal finite dimensional linear models.
\newblock 1997.

\end{thebibliography}
\end{document}